%% file: main.tex
\documentclass[10pt,journal,compsoc]{IEEEtran}

\ifCLASSOPTIONcompsoc
  \usepackage[nocompress]{cite}
\else
  \usepackage{cite}
\fi

\usepackage{hyperref}
\usepackage{multirow}
\usepackage{multicol}
\usepackage{amsfonts}
\usepackage{mathtools}
\newcommand{\tabincell}[2]{\begin{tabular}{@{}#1@{}}#2\end{tabular}} 
\usepackage{url}
\usepackage{comment}
\usepackage{xcolor}
\usepackage{graphicx}
\usepackage{wrapfig}
\usepackage[ruled, vlined]{algorithm2e}
\usepackage{amsthm}
\usepackage{algorithmicx}
\usepackage{algpseudocode}
\usepackage{threeparttable}
\usepackage{diagbox}
\newcommand{\modify}[1]{{ #1}}

\newtheorem{theorem}{Theorem}[section]
\newtheorem{corollary}{Remark}[section]
\newtheorem{assumption}{Hypothesis}[section]
\newtheorem{lemma}[theorem]{Lemma}
\newtheorem{definition}{Definition}[section]
\newtheorem{prop}{Proposition}[section]

\begin{document}

\title{ A Comprehensive and Modularized Statistical Framework for Gradient Norm Equality in \\Deep Neural Networks}

\author{Zhaodong Chen, Lei Deng, \IEEEmembership{Member}, \IEEEmembership{IEEE}, Bangyan Wang, Guoqi Li, \IEEEmembership{Member}, \IEEEmembership{IEEE},   \\Yuan Xie, \IEEEmembership{Fellow}, \IEEEmembership{IEEE}
\thanks{The work was partially supported by National Science Foundation (Grant No. 1725447), Tsinghua University Initiative Scientiﬁc Research Program, Tsinghua-Foshan Innovation Special Fund (TFISF), and National Natural Science Foundation of China (Grant No. 61876215). Zhaodong Chen and Lei Deng contributed equally to this work, corresponding author: Guoqi Li. Zhaodong Chen, Lei Deng, Bangyan Wang, and Yuan Xie are with the Department of Electrical and Computer Engineering, University of California, Santa Barbara, CA 93106, USA (email: \{chenzd15thu, leideng, bangyan, yuanxie\}@ucsb.edu). Guoqi Li is with the Department of Precision Instrument, Center for Brain Inspired Computing Research, Tsinghua University, Beijing 100084, China (email: liguoqi@mail.tsinghua.edu.cn).
}}


\IEEEtitleabstractindextext{%
\begin{abstract}
\input{chapters/abstract.tex}
\end{abstract}

}

\maketitle

\IEEEdisplaynontitleabstractindextext

\IEEEpeerreviewmaketitle

\IEEEraisesectionheading{\section{Introduction}\label{sec:introduction}}

\input{chapters/introduction.tex}

\section{Related Work}\label{sec:related_works}
\input{chapters/literature_review.tex}

\section{Gradient Norm Equality}\label{sec:dynamic_of_}
\input{chapters/Dynamics_of_.tex}

\section{Analysis of Spectrum-Moments}\label{sec:analysis_spectrum_moment}
\input{chapters/bn_stabilizer.tex}

\section{Serial Neural Networks}
\input{chapters/series_network.tex}

\section{Serial-parallel hybrid Networks}
\input{chapters/series_parallel_hybrid_networks.tex}

\section{Experiments} 
In this section, we first conduct several numerical experiments in Section \ref{sec: numerical_exp} to verify the correctness of our key theorems: Theorem \ref{theorem:multiplication} and \ref{theorem:Addition}. Then, in Section \ref{sec: cifar10_exp}, we perform extensive experiments to support our conclusions in previous sections. At last, in Section \ref{sec: imagenet_exp}, we further test several methods that yield interesting results on CIFAR-10 and ImageNet.

\input{tables/network_structures.tex}
\vspace{-10pt}
\subsection{Numerical Experiments}\label{sec: numerical_exp}
\input{chapters/Exp_numeric_exp.tex}

\subsection{Experiments on CIFAR-10}\label{sec: cifar10_exp}
\input{chapters/Exps/Exp_setups.tex}

\textbf{Initialization in Serial Network.}\label{exp:orth_init}
\input{chapters/Exps/Exp_orthogonal_init.tex}

\textbf{Normalization in Serial Network.}
\input{chapters/Exps/Exp_weight_standard.tex}

\textbf{Self-Normalizing Neural Network.}
\input{chapters/Exps/Exp_SeLU.tex}

\textbf{DenseNet.}\label{Exp:densenet}
\input{chapters/Exps/DenseNet.tex}

\textbf{ResNet.}
\input{chapters/Exps/Exp_ResNet.tex}

\vspace{-10pt}
\subsection{Experiments on ImageNet}
\input{chapters/Exps/imagenet.tex}\label{sec: imagenet_exp}

\section{Conclusion}
\input{chapters/conclusion.tex}

\bibliographystyle{IEEEtran}
\bibliography{ref}
\input{chapters/biography.tex}



\ifCLASSOPTIONcaptionsoff
  \newpage
\fi





\clearpage

\appendices
\input{chapters/appendix_proof.tex}
\input{chapters/l2n_overhead.tex}\label{appendix:l2n_overhead}
\input{chapters/appendix_discussions.tex}
\input{chapters/exp_setup.tex}

\end{document}

%% file: chapters/abstract.tex
The rapid development of deep neural networks (DNNs) in recent years can be attributed to the various techniques that address gradient explosion and vanishing. In order to understand the principle behind these techniques and develop new methods, plenty of metrics have been proposed to identify networks that are free of gradient explosion and vanishing. However, due to the diversity of network components and complex serial-parallel hybrid connections in modern DNNs, the evaluation of existing metrics usually requires strong assumptions, complex statistical analysis, or has limited application fields, which constraints their spread in the community. In this paper, inspired by the Gradient Norm Equality and dynamical isometry, we first propose a novel metric called Block Dynamical Isometry, which measures the change of gradient norm in individual block. Because our Block Dynamical Isometry is norm-based, its evaluation needs weaker assumptions compared with the original dynamical isometry. To mitigate the challenging derivation, we propose a highly modularized statistical framework based on free probability. Our framework includes several key theorems to handle complex serial-parallel hybrid connections and a library to cover the diversity of network components. Besides, several sufficient prerequisites are provided. Powered by our metric and framework, we analyze extensive initialization, normalization, and network structures. We find that Gradient Norm Equality is a universal philosophy behind them. Then, we improve some existing methods based on our analysis, including an activation function selection strategy for initialization techniques, a new configuration for weight normalization, and a depth-aware way to derive coefficients in SeLU. Moreover, we propose a novel normalization technique named second moment normalization, which is theoretically 30\% faster than batch normalization without accuracy loss. Last but not least, our conclusions and methods are evidenced by extensive experiments on multiple models over CIFAR10 and ImageNet.

%% file: chapters/introduction.tex
\IEEEPARstart{I}{t} has become a common sense that deep neural networks (DNNs) are more effective compared with the shallow ones \cite{delalleau2011shallow}. However, the training of very deep neural models usually suffers from gradient explosion and vanishing. To this end, plenty of schemes and network structures have been proposed. For instance: He et al. (2015) \cite{he2015delving}, Mishkin \& Matas (2015) \cite{mishkin2015all}, Xiao et al. (2018) \cite{xiao2018dynamical} and Zhang et al. (2019) \cite{zhang2019fixup} suggest that the explosion and vanishing can be mitigated with proper initialization of network parameters. Ioffe \& Szegedy (2015) \cite{ioffe2015batch}, Salimans \& Kingma (2016) \cite{salimans2016weight} and Qiao et al. (2019) \cite{qiao2019weight} propose several normalization schemes that can stabilize the neural networks during training. From the perspective of network structures, He et al. (2016) \cite{he2016deep} and Huang et al. (2017) \cite{huang2017densely} demonstrate that neural networks with shortcuts can effectively avoid gradient vanishing and explosion. These techniques are still crucial even when the network structures are constructed in data-driven ways like Neural Architecture Search (NAS): the gradient explosion and vanishing in NAS models are still handled by handcrafted methods like batch normalization (BN) and shortcut connections \cite{tan2019mnasnet}.

\modify{It's natural to ask: \emph{is there a common philosophy behind all these studies?} Such a philosophy is able to guide us to achieve novel hyper-parameter selection strategies and novel network structures.

In recent years, great efforts have been made to pursue such a philosophy. While He et al. (2015) \cite{he2015delving} and  Mishkin \& Matas (2015) \cite{mishkin2015all} preserve the information flow in the forward pass, powered by dynamical mean-field theory. Poole et al. (2016) \cite{poole2016exponential}, Xiao et al. (2018) \cite{xiao2018dynamical}, Yang et al. (2019) \cite{yang2019mean} and Schoenholz et al. (2016) \cite{schoenholz2016deep} study the stability of the statistics fixed point. They identify an order-to-chaos phase transition in deep neural networks. If the network seats steady on the border between the order and chaos phase, it will be trainable even with a depth of 10,000 \cite{xiao2018dynamical}. Pennington et al. (2017) \cite{pennington2017resurrecting}, Tarnowski et al. (2018) \cite{tarnowski2018dynamical}, Pennington et al. (2018) \cite{pennington2018emergence} and Ling \& Qiu (2018) \cite{ling2018spectrum} argue that networks achieving dynamical isometry (all the singular value of the network's input-output Jacobian matrix remain close to 1) do not suffer from gradient explosion or vanishing. Philipp et al. (2019) \cite{philipp2018gradients} directly evaluate the statistics of the gradient and propose a metric called gradient scale coefficient (GSC) that can verify whether a network would suffer gradient explosion. Arpit \& Bengio (2019) \cite{arpit2019benefits} find that networks with Gradient Norm Equality property usually have better performance. Gradient Norm Equality means that the Ferobenius Norm of the gradient is more or less equal in different layers' weights, therefore the information flow in the backward pass can be preserved and the gradient explosion and vanishing are prevented.

Although so many metrics have been proposed, most of them only focus on providing explanations for existing methods, and few of them are applied in discovering novel algorithms for cutting-edge DNN models. The major reason is that these metrics lack handy statistical tools to apply in complex network structures.
\vspace{-20pt}
\begin{figure}[ht]
\centering
\includegraphics[width=0.45\textwidth]{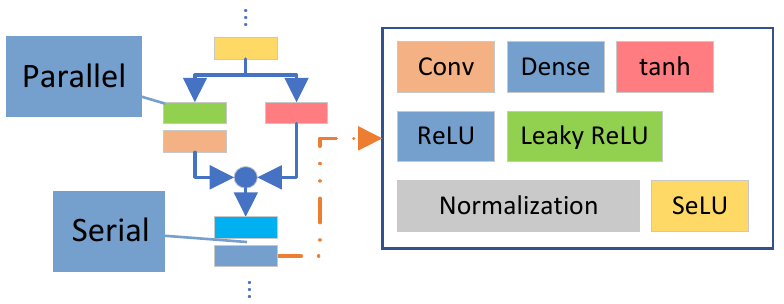}
\caption{Illustration of complex network structure.}
\vspace{-5pt}
\label{fig:framework}
\end{figure}

As illustrated in Fig. \ref{fig:framework}, the modern neural networks are usually composed of several different kinds of linear or nonlinear components like convolution, activation function, and normalization. These components are connected either in parallel or serial. The diversity of network components and different kind of connections result in two challenges: nontrivial prerequisites and complex derivation. Because of the former one, some studies rely on strong assumptions. For example, to calculate the quadratic mean norm (qm norm) of the Jacobian matrices, Philipp et al. (2019) \cite{philipp2018gradients} assume that the norm of the product of Jacobian matrices has approximate decomposability. The free probability used in Pennington et al. (2017) \cite{pennington2017resurrecting}, Tarnowski et al. (2018) \cite{tarnowski2018dynamical}, Pennington et al. (2018) \cite{pennington2018emergence} and Ling \& Qiu (2018) \cite{ling2018spectrum} requires the involved matrices to be freely independent with each other \cite{mingo2017free}, which is not commonly held and difficult to verify \cite{chen2012partial}. Because of the complex derivation, existing studies usually require strong statistics backgrounds and mathematical skills, which constraints their spread in the community. For example, even if the prerequisites of free probability are satisfied, the derivation still requires the probability density of the eigenvalues in each Jacobian matrix, which will then go through several different transforms and series expansions \cite{ling2018spectrum}. Last but not least, these challenges also limit the applicable scope of existing studies such that they only support simple serial networks with few kinds of components \cite{he2015delving,poole2016exponential,xiao2018dynamical,yang2019mean,schoenholz2016deep,arpit2019benefits}.

In this paper, we first propose a new metric, block dynamical isometry, to identify the networks without gradient explosion or vanishing. Our metric is inspired by the gradient norm equality \cite{arpit2019benefits} and dynamical isometry \cite{ling2018spectrum}. However, unlike previous studies in dynamical isometry that investigate the distribution of eigenvalues, we only focus on the $1^{st}$ and $2^{nd}$ moments. Therefore, our prerequisites are much weaker and easier to exam. We further provide several handy sufficient conditions to check them. To handle the parallel and serial connections as in Fig. \ref{fig:framework}, we extend the conclusions in Ling \& Qiu (2018) \cite{ling2018spectrum} and provide two main theorems to handle different kinds of connections. To deal with the diversity of network components, we develop a library that contains most popular components in modern DNNs. To make our theory available for most of the peers in the deep learning community, our framework is highly modularized and can be used just like sophisticated coding frameworks like TensorFlow and PyTorch. Specifically, a complex serial-parallel hybrid network can be easily analyzed by first looking up the components from the library, then checking the prerequisites, and at last linking them with the main theorems. In a word, our framework is much more comprehensive, rigorous, and easier-to-use compared with prior work.}

\input{tables/notations.tex}

To demonstrate the effectiveness of our framework, we provide several demos, which are summarized below.

\textbf{Comprehensiveness}: we provide statistical explanations for various existing studies including initialization techniques \cite{he2015delving, xiao2018dynamical, zhang2019fixup}, normalization techniques \cite{ioffe2015batch,salimans2016weight,arpit2016normalization}, self-normalizing neural network \cite{klambauer2017self} and complex network structures like ResNet \cite{he2016deep} and DenseNet \cite{huang2017densely}.

\textbf{Improvements}: We further improve several existing studies based on the insights provided by our framework. For initialization techniques, we systematically compare several activation functions, and identify that although tanh used in Xiao et al. (2018) \cite{xiao2018dynamical} is more stable, leaky ReLU with relatively higher negative slope coefficient is more effective in networks with moderate depth. Compared with tanh, this new configuration achieves up to $6.81\%$ higher accuracy on CIFAR-10, which is also $0.54\%$ higher than the BN baseline. Besides, we modify the PReLU activation function proposed by He et al. (2015) \cite{he2015delving} and give a novel one called sPReLU that automatically learns a good negative slope coefficient. It achieves $0.77\%$ higher accuracy than its BN baseline on CIFAR-10. In the Conv MobileNet v1 on ImageNet, \modify{our orthogonal initialization has only $0.78\%$ accuracy loss compared with BN}. For weight normalization \cite{salimans2016weight}, we combine it with the initialization techniques and propose a method called scaled weight standardization. In our 32-layer network on CIFAR-10, the accuracy is only $0.64\%$ lower than BN while mitigating the gradient vanishing. For self-normalizing neural network \cite{klambauer2017self}, we identify that the coefficients in the SeLU activation function should be given according to the depth of the network, and provide a new way to find these coefficients. Our new SeLU outperforms the original configuration by $0.42\%$ and $0.39\%$ with Gaussian and orthogonal weight on CIFAR-10, respectively.

\textbf{Novel Method}: Inspired by our analysis of normalization techniques, we propose a new normalization technique called second moment normalization. \modify{Its computational complexity is almost same with weight normalization \cite{salimans2016weight} and theoretically $30\%$ faster than BN; besides, with proper regularization like mixup \cite{zhang2017mixup}, it can achieve $23.74\%$ error rate on ResNet-50 on ImageNet, which is $0.07\%$ lower than BN under the same regularization ($23.81\%$).}

For the sake of clarity, we provide a describing of the default notations used throughout this paper in Table \ref{tab:notations}. Our codes in PyTorch are publicly available at \href{https://github.com/apuaaChen/GNEDNN\_release}{https://github.com/apuaaChen/GNEDNN\_release}.

%% file: tables/notations.tex
\begin{table*}[t]
\caption{Default notations.}
\centering
\resizebox{0.99\textwidth}{!}{
\begin{tabular}{c|c||c|c}
\hline
\multicolumn{4}{c}{Numbers, Arrays and Matrices}\\
\hline \hline
$a$ & a scalar & $\mathbf{a}$ & a column vector\\
\hline
$\mathbf{A}$ & a matrix & $\mathbf{n}, n\in\mathbb{R}$ & a vector or matrix\\
\hline
$\mathbf{I}$ & square identify matrix&&\\
\hline
\multicolumn{4}{c}{Operators}\\
\hline \hline
$Tr(\mathbf{A})$ & the trace of $\mathbf{A}$ & $tr(\mathbf{A})$ & the normalized trace of $\mathbf{A}$, e.g. $tr(\mathbf{I})=1$\\
\hline
$E[x]$& the expectation of r.v. $x$ & $D[x]$ & the variance of r.v. $x$\\
\hline
$\lambda_{\mathbf{A}}$ & the eigenvalues of $\mathbf{A}$&$\alpha_k(a)$& the $k^{th}$ order moment of r.v. $a$\\ 
\hline
$\mathbf{f}(\mathbf{a})$ & a mapping function taking $\mathbf{a}$ as input &$\mathbf{f_a}$ & the Jacobian matrix $\frac{\partial \mathbf{f}(\mathbf{a})}{\partial \mathbf{a}}$\\
\hline
$\phi(\mathbf{A}):=E[tr(\mathbf{A})]$& the expectation of $tr(\mathbf{A})$& $\varphi(\mathbf{A})$&$\phi(\mathbf{A^2})-\phi^2(\mathbf{A})$\\
\hline
$height(\mathbf{A})$& the height of matrix $\mathbf{A}$&$width(\mathbf{A})$& the width of matrix $\mathbf{A}$\\
\hline
$len(\mathbf{a})$& the length of vector $\mathbf{a}$&&\\
\hline
\multicolumn{4}{c}{Index}\\
\hline \hline
$[\mathbf{A}]_{i,j}$&element$(i,j)$ of $\mathbf{A}$&&\\
\hline
\end{tabular}
}
\label{tab:notations}
\vspace{-10pt}
\end{table*}

%% file: chapters/literature_review.tex
\subsection{Theorems of Well-behaved Neural Networks}

\hspace*{14pt}\textbf{Dynamical Isometry.} A neural network is dynamical isometry as long as every singular value of its input-output Jacobian matrix remains close to 1, thus the norm of every error vector and the angle between error vectors are preserved. With the powerful theorems of free probability and random matrix, Pennington et al. (2017) \cite{pennington2017resurrecting} investigate the spectrum density distribution of plaint fully-connected serial network with Gaussian/orthogonal weights and ReLU/hard-tanh activation functions; Tarnowski et al. (2018) \cite{tarnowski2018dynamical} explore the density of singular values of the input-output Jacobian matrix in ResNet and identify that dynamical isometry can be always achieved regardless of the choice of activation function. However, their studies only cover ResNet whose major branch consists of Gaussian and scaled orthogonal linear transforms and activation functions, and fail to provide a theoretical explanation of batch normalization.  Although our derivations of Theorem \ref{theorem:multiplication} and \ref{theorem:Addition} are inspired by the Result 2 \& 1 in Ling \& Qiu (2018) \cite{ling2018spectrum}, their discussions are limited to the spectrum of ResNet due to two reasons. First, their derivation requires the detailed spectrum density of involved components; second, they fail to realise that although the trace operator is cyclic-invariant, the normalized trace operator is not when rectangle matrices are involved, so that their Result 2 can only handle square Jacobian matrices. Last but not least, a universal problem in existing dynamical isometry related studies is that the derivation is based on the strong assumption that all the involved matrices are freely independent, which is uncommonly held and difficult to verify \cite{chen2012partial}.

\textbf{Order-to-Chaos Phase Transition.} Poole et al. (2016) \cite{poole2016exponential} and Schoenholz et al. (2016) \cite{schoenholz2016deep} analyze the signal propagation in simple serial neural networks and observe that there is an order-to-chaos phase transition determined by a quantity: $\chi :=\phi\left(\left(\mathbf{DW}\right)^T\mathbf{DW}\right)$ \cite{pennington2017resurrecting}, where $\mathbf{D}$ is the Jacobian matrix of activation function, $\mathbf{W}$ denotes the weight and $\phi$ represents the expectation of the normalized trace of a given matrix. The network is in the chaotic phase if $\chi>1$ and in the order phase when $\chi<1$. The chaotic phase results in gradient explosion while the order phase causes gradient vanishing. Due to the lack of convenient mathematic tools for ``$\phi$'' analysis, the current application of the order-to-chaos phase transition is also limited to vanilla serial networks.

\textbf{Gradient Scale Coefficient (GSC).} Philipp et al.(2018) \cite{philipp2018gradients} propose a metric that evaluates how fast the gradient explodes. Let $0\le l\le k \le L$ be the index of the network's layers, the GSC is defined as
\begin{equation}
    GSC(k, l) = \frac{\phi\left(\left(\Pi_{i=l}^k\mathbf{J_i}\right)^T\left(\Pi_{i=l}^k\mathbf{J_i}\right)\right)||f_k||_2^2}{||f_l||_2^2}.
\end{equation}
To efficiently calculate this metric, the authors suggest that $GSC(k,l) = \Pi_{i=l}^{k-1}GCS(i+1, i)$, which is derived under an assumption of $\phi\left(\left(\Pi_{i=l}^k\mathbf{J_i}\right)^T\left(\Pi_{i=l}^k\mathbf{J_i}\right)\right)=\Pi_{i=l}^k\phi(\mathbf{J_i}^T\mathbf{J_i})$. In our work, we provide not only a solid derivation for this assumption but also theoretical tools for networks with parallel branches, which makes our method applicable in more general situations.
\vspace{-10pt}
\subsection{Techniques that Stabilize the Network}

\hspace*{14pt}\textbf{Initialization.} It has long been observed that neural networks with proper initialization converge faster and better. Thus, handful initialization schemes have been proposed: He et al. (2015) \cite{he2015delving} introduce Kaiming initialization that maintains the second moment of activations through plaint serial neural networks with rectifier activations; Zhang et al. (2019) \cite{zhang2019fixup} expand initialization techniques to networks with shortcut connections like ResNet and achieves advanced results without BN; Xiao et al. (2018) \cite{xiao2018dynamical} provide an orthogonal initialization scheme for serial neural networks, which makes 10,000-layer networks trainable.

\textbf{Normalization.} Batch normalization (BN) \cite{ioffe2015batch} has become a standard implementation in modern neural networks \cite{he2016deep, huang2017densely}. BN leverages the statistics (mean \& variance) of mini-batches to standardize the pre-activations and allows the network to go deeper without significant gradient explosion or vanishing. Despite of BN's wide application, it has been reported that BN introduces high training latency \cite{chen2019effective,gitman2017comparison} and its effectiveness drops when the batch size is small \cite{wu2018group}. Moreover, BN is also identified as one of the major root causing quantization loss \cite{sheng2018quantization}. To alleviate these problems, Salimans \& Kingma (2016) \cite{salimans2016weight} instead normalize the weights, whereas it's less stable compared with BN \cite{gitman2017comparison}.

\textbf{Self-normalizing Neural Network.} Klambauer et al. (2017) \cite{klambauer2017self} introduce a novel activation function called ``scaled exponential linear unit" (SeLU), which can automatically force the activation towards zero mean and unit variance for better convergence

\textbf{Shortcut Connection.} In He et al. (2016) \cite{he2016deep}, the concept of shortcut in neural networks was first introduced and then further developed by Huang et al. (2017) \cite{huang2017densely}, which results in two most popular CNNs named ResNet and DenseNet. These models demonstrate that the shortcut connections make deeper models trainable.

%% file: chapters/Dynamics_of_.tex
In this section, we analyze how the norm of backward gradient evolves through the deep neural network and derive our metric for gradient norm equality.
\vspace{-10pt}
\subsection{Dynamic of Gradient Norm}

Without loss of generality, let's consider a neural network consists of sequential blocks: 
\begin{equation}
    \mathbf{f}(\mathbf{x_0}) = \mathbf{f_{L,\theta_L}}\circ \mathbf{f_{L-1,\theta_{L-1}}}\circ...\circ\mathbf{f_{1, \theta_1}}\left(\mathbf{x_0}\right)
\label{equ:series_block_network}
\end{equation}
where $\mathbf{\theta_i}$ is the vectorized parameter of the $i^{th}$ layer. We represent the loss function as $\mathcal{L}(\mathbf{f}(\mathbf{x}), \mathbf{y})$ wherein $\mathbf{y}$ denotes the label vector. At each iteration, $\mathbf{\theta_i}$ is updated by $\mathbf{\theta_i} -\mathbf{\Delta \theta_i} = \mathbf{\theta_i} - \eta\frac{\partial}{\partial \mathbf{\theta_i}}\mathcal{L}(\mathbf{f}(\mathbf{x}), \mathbf{y})$, where $\eta$ is the learning rate. With the chain rule, we have
\begin{equation}
    \begin{split}
        &\frac{\partial}{\partial \mathbf{f_{i}}}\mathcal{L}(\mathbf{f}(\mathbf{x}), \mathbf{y}) = \left(\frac{\partial \mathbf{f_{i+1}}}{\partial \mathbf{f_{i}}}\right)^T\frac{\partial}{\partial \mathbf{f_{i+1}}}\mathcal{L}(\mathbf{f}(\mathbf{x}), \mathbf{y}),\\
        &\frac{\partial}{\partial \mathbf{\theta_{i}}}\mathcal{L}(\mathbf{f}(\mathbf{x}), \mathbf{y}) = \left(\frac{\partial \mathbf{f_{i}}}{\partial \mathbf{\theta_{i}}}\right)^T \frac{\partial}{\partial \mathbf{f_{i}}}\mathcal{L}(\mathbf{f}(\mathbf{x}), \mathbf{y}).
    \end{split}
\end{equation}

For the sake of simplicity, we denote $\frac{\partial \mathbf{f_{j}}}{\partial \mathbf{f_{j-1}}} := \mathbf{J_j}\in \mathbb{R}^{m_j\times n_j}$, and $\mathbf{\Delta \theta_i}$ is given by
\begin{equation}
    \mathbf{\Delta \theta_i} = \eta \left(\frac{\partial \mathbf{f_{i}}}{\partial \mathbf{\theta_{i}}}\right)^T\left(\Pi_{j=L}^{i+1}\mathbf{J_j}\right)^T \frac{\partial}{\partial \mathbf{f}(\mathbf{x})}\mathcal{L}(\mathbf{f}(\mathbf{x}), \mathbf{y}).
\end{equation}
Further, we denote $\mathbf{K_{i+1}}:=\left(\frac{\partial \mathbf{f_{i}}}{\partial \mathbf{\theta_{i}}}\right)^T\left(\Pi_{j=L}^{i+1}\mathbf{J_j}\right)^T$ and $\mathbf{u}:=\frac{\partial}{\partial \mathbf{f}(\mathbf{x})}\mathcal{L}(\mathbf{f}(\mathbf{x}), \mathbf{y})$. We represent the scale of $\mathbf{\Delta\theta_i}$ with its L2 norm: $|| \mathbf{\Delta\theta_i}||_2^2=\eta^2\mathbf{u}^T\mathbf{K}^T_{\mathbf{i+1}}\mathbf{K}_{\mathbf{i+1}} \mathbf{u}$. As $\mathbf{K}^T_{\mathbf{i+1}}\mathbf{K_{i+1}}$ is a real symmetric matrix, it can be broken down with eigendecomposition: $\mathbf{K}^T_{\mathbf{i+1}}\mathbf{K_{i+1}} = \mathbf{Q}^T\mathbf{\Lambda Q}$, where $\mathbf{Q}$ is an orthogonal matrix. Therefore we have: 
\begin{equation}
\begin{split}
    & ||\mathbf{\Delta\theta_i}||_2^2 = \eta^2(\mathbf{Qu})^T\mathbf{\Lambda}(\mathbf{Qu}),\\
    & E\left[||\mathbf{\Delta\theta_i}||_2^2\right] = \eta^2E\left[\sum_j\lambda_j[\mathbf{Qu}]_j^2\right].
\end{split}
\end{equation}
With the symmetry, we assume $\forall i,j$, $E[[\mathbf{Qu}_i]^2]=E[[\mathbf{Qu}_j]^2]$ and $E[\lambda_i]=E[\lambda_j]$. Since $\lambda_j$ is independent of $[\mathbf{Qu}]_j$, we have
\begin{equation}
\begin{split}
    & E\left[||\mathbf{\Delta\theta_i}||_2^2\right] =  \eta^2\sum_jE[\lambda_j]E\left[[\mathbf{Qu}]_i^2\right]\\
    &~~~~~~~~~~~~~~~~~~~\approx \eta^2\phi\left(\mathbf{K}^T_{\mathbf{i+1}}\mathbf{K_{i+1}}\right)E\left[||\mathbf{u}||_2^2\right].
\end{split}
\label{equ:backward_expect_l2_norm}
\end{equation}

If $E[||\Delta \mathbf{\theta_i}||_2^2] \rightarrow 0$, the update of parameters of the $i^{th}$ layer would be too tiny to make a difference and thus the gradient vanishing occurs; If $E[||\Delta \mathbf{\theta_i}||_2^2]\rightarrow \infty$, the parameters of the $i^{th}$ layer would be drastically updated and thus the gradient explosion happens. Therefore, the network is stable when $\phi\left(\mathbf{K}^T_{\mathbf{i+1}}\mathbf{K_{i+1}}\right)$ neither grows nor diminishes exponentially with the decreasing of $i$, and we can say that the network has the property of gradient norm equality \cite{arpit2019benefits}.
\vspace{-10pt}
\subsection{Block Dynamical Isometry}

In order to simplify the derivation of
\begin{equation}
    \phi\left(\mathbf{K}^T_{\mathbf{i+1}}\mathbf{K_{i+1}}\right)\!=\!\phi\left(\left(\Pi_{j=L}^{i+1}\mathbf{J_j}\right)\frac{\partial \mathbf{f_{i}}}{\partial \mathbf{\theta_{i}}}\left(\frac{\partial \mathbf{f_{i}}}{\partial \mathbf{\theta_{i}}}\right)^T\left(\Pi_{i=L}^{i+1}\mathbf{J_j}\right)^T\right),
\end{equation}
\modify{we temporarily propose the following Hypothesis \ref{assumption:jacobian_breakdown}, which is inspired by the assumption on the approximate decomposability of the norm of the product of Jacobians in Philipp et al. (2018) \cite{philipp2018gradients}.
\begin{assumption}
Under some prerequisites, given a set of Jacobian matrices $\{\mathbf{J_{L},...,\mathbf{J_{i+1}}}, \frac{\partial \mathbf{f_{i}}}{\partial \mathbf{\theta_{i}}}\}$, we have
\begin{equation}
\begin{split}
    &\phi\left(\left(\Pi_{j=L}^{i+1}\mathbf{J_j}\right)\frac{\partial \mathbf{f_{i}}}{\partial \mathbf{\theta_{i}}}\left(\frac{\partial \mathbf{f_{i}}}{\partial \mathbf{\theta_{i}}}\right)^T\left(\Pi_{i=L}^{i+1}\mathbf{J_j}\right)^T\right)\\
    &~~~~~~~~~~~~~~~~~~~~~~~~~=\phi\left(\frac{\partial \mathbf{f_{i}}}{\partial \mathbf{\theta_{i}}}\left(\frac{\partial \mathbf{f_{i}}}{\partial \mathbf{\theta_{i}}}\right)^T\right)\Pi_{j=L}^{i+1}\phi\left(\mathbf{J_jJ_j}^T\right).
\end{split}
\label{equ:gradient_phi_decomposite}
\end{equation}
\label{assumption:jacobian_breakdown}
\end{assumption}
With the theoretical tools developed in Section \ref{sec:analysis_spectrum_moment}, this hypothesis can be easily proved and the prerequisites can be confirmed (Remark \ref{remark:hyp31}).
} In Hypothesis \ref{assumption:jacobian_breakdown}, the only term that may result in the unsteady gradient is $\Pi_{j=L}^{i+1}\phi(\mathbf{J_jJ_j}^T)$. Therefore, the gradient norm equality can be achieved by making $\forall j,\phi(\mathbf{J_jJ_j}^T)\approx1$.

However, the above condition is not sufficient for neural networks with finite width. We have $tr\left(\mathbf{J_jJ_j}^T\right)=\frac{1}{m_j}\sum_{i=1}^{m_j}\lambda_i$, where $\lambda_i$ denotes the $i^{th}$ eigenvalue of $\mathbf{J_jJ_j}^T$. Under the assumption that $\forall p,q\neq p, \lambda_p$ is independent of $\lambda_q$, the variance of $tr(\mathbf{J_jJ_j}^T)$ is given by
\begin{equation}
\begin{split}
    &D[tr(\mathbf{J_jJ_j}^T)] = \frac{1}{m_j}\sum_{i=1}^{m_j}E[\lambda_i^2]-E^2[\lambda_i]\\
    & = \phi\left(\left(\mathbf{J_jJ_j}^T\right)^2\right) - \phi^2\left(\mathbf{J_jJ_j}^T\right):=\varphi\left(\mathbf{J_jJ_j}^T\right).
\end{split}
\end{equation}
As a result, for networks that are not wide enough, in order to make sure that the $\phi(\mathbf{J_jJ_j}^T)$ of each block sits steadily around $1$, we expect $\varphi(\mathbf{J_jJ_j}^T)$ of each layer to be small. Therefore, our metric can be formally formulated as below.

\begin{definition}
\textbf{(Block Dynamical Isometry) }Consider a neural network that can be represented as a sequence of blocks as Equation \eqref{equ:series_block_network} and the $j^{th}$ block's Jacobian matrix is denoted as $\mathbf{J_j}$. If $\forall j,~~\phi(\mathbf{J_jJ_j}^T)\approx1$ and $\varphi(\mathbf{J_jJ_j}^T)\approx 0$, we say the network achieves block dynamical isometry.
\label{def:block_dynamical_isometry}
\end{definition}
\modify{As $\phi(\mathbf{J_jJ_j}^T)$ and $\varphi(\mathbf{J_jJ_j}^T)$ can be seen as the $1^{st}$ and $2^{nd}$ moment of eigenvalues of $\mathbf{J_j}$, we name them as \textbf{spectrum-moments}.}
While $\phi(\mathbf{J_jJ_j}^T)\approx 1$ addresses the problem of gradient explosion and vanishing, $\varphi(\mathbf{J_jJ_j}^T)\approx 0$ ensures that $\phi(\mathbf{J_jJ_j}^T)\approx 1$ is steadily achieved. Actually, we find that in many cases, $\phi(\mathbf{J_jJ_j}^T)\approx 1$ is enough to instruct the design or analysis of a neural network.

We name our metric as ``Block Dynamical Isometry" is because it is quite similar to the dynamical isometry discussed in Saxe et al. (2013) \cite{saxe2013exact} and Pennington et al. (2017;2018) \cite{pennington2017resurrecting,pennington2018emergence}. The original dynamical isometry expects that every singular value of the whole network's input-output Jacobian matrix remains close to 1, while ours is expected in every sequential block. Definition \ref{def:block_dynamical_isometry} allows us to use divide and conquer in the analysis of a complex neural network: a network can be first divided to several blocks connected in serial, and then conquered individually.

%% file: chapters/bn_stabilizer.tex
In this section, we will develop a statistical framework for the analysis of spectrum-moments, which will not only provide a theoretical base for Hypothesis \ref{assumption:jacobian_breakdown} but also simplify the calculation of $\phi(\mathbf{J_iJ_i}^T)$ and $\varphi(\mathbf{J_iJ_i}^T)$ of each block. 

Generally speaking, despite the great diversity of network structures, most of them can be regarded as basic network components, i.e. linear transforms and nonlinear activation functions, that are connected in serial or parallel. Therefore, our framework is expected to be able to handle $\phi\left(\left(\left(\sum_i \mathbf{J_i}\right)\left(\sum_i \mathbf{J_i}\right)^T\right)^p\right)$ and $\phi\left(\left(\left(\Pi_i \mathbf{J_i}\right)\left(\Pi_i \mathbf{J_i}\right)^T\right)^p\right)$, where $p\in\{1,2\}$. 
To serve this purpose, our framework consists of the following parts:
\begin{itemize}
    \item Two main theorems that build bridges between $\phi\!\left(\!\left(\left(\sum_i \mathbf{J_i}\right)\!\left(\sum_i \mathbf{J_i}\right)^T\right)^p\right)$, $\phi\!\left(\!\left(\left(\Pi_i \mathbf{J_i}\right)\!\left(\Pi_i \mathbf{J_i}\right)^T\right)^p\right)$ and their components $\phi\left(\left(\mathbf{J_iJ_i}^T\right)^p\right)$.
    \item A library of $\phi\left(\left(\mathbf{J_iJ_i}^T\right)^p\right)$ for common components in neural networks.
\end{itemize}
\vspace{-10pt}
\subsection{Inspiration: Propagation of Euclidean Norm in a Rotational-Invariant System}\label{sec:inspiration}

\begin{definition}
\textbf{(Rotational-Invariant Distribution)} Given a random vector $\mathbf{g_i}$, we say it has rotational-invariant distribution if it has the same distribution with $\mathbf{Ug_i}$, for any unitary matrix $\mathbf{U}$ independent of $\mathbf{g_i}$.
\label{def:rotational_invariant_distribution}
\end{definition}

Let the gradient of the $i^{th}$ layer be $\frac{\partial}{\partial \mathbf{f_{i}}}\mathcal{L}(\mathbf{f}(\mathbf{x}), \mathbf{y}) = \mathbf{g_i}$, where $\mathbf{g_i}$ has a rotational-invariant distribution. Under this assumption, intuitively, its elements share the same second moment. The gradient of its previous layer is $\frac{\partial}{\partial \mathbf{f_{i-1}}}\mathcal{L}(\mathbf{f}(\mathbf{x}), \mathbf{y}) = \mathbf{g_{i-1}} = \mathbf{J_i}\frac{\partial}{\partial \mathbf{f_{i}}}\mathcal{L}(\mathbf{f}(\mathbf{x}), \mathbf{y}) = \mathbf{J_ig_i}$, where $\mathbf{J_i}$ is a Jacobian matrix. As the values in the Jacobian matrix are trainable, we can also assume that it is a random matrix. With the singular value decomposition, we have $\mathbf{J_i} = \mathbf{U\Sigma V}^H$, where $\mathbf{U}$ and $\mathbf{V}$ are unitary matrices, and $\mathbf{\Sigma}$ is a diagonal matrix whose diagonal entries are the singular values ($\sigma_1, \sigma_2, ...$) of $\mathbf{J_i}$. When we calculate $\mathbf{U\Sigma V}^H\mathbf{g_i}$, $\mathbf{V}^H$ first rotates the origin distribution to the new orthogonal basis, and then $\mathbf{\Sigma}$ stretches each basis by the corresponding singular value. At last. $\mathbf{U}$ rotates the distribution to the output orthogonal basis.

On the one hand, the distribution of $\mathbf{g_i}$ is invariant under the rotation of $\mathbf{V}^H$. On the other hand, since $\mathbf{U}$ is a unitary matrix, it doesn't change the $L_2$ norm of $\mathbf{\Sigma V}^H\mathbf{g_i}$. Therefore, we have
\begin{equation}
\begin{split}
    & E\left[||\mathbf{g_{i-1}}||_2^2\right] = E\left[||[\sigma_1[\mathbf{g_i}]_1, \sigma_2[\mathbf{g_i}]_2, ..., \sigma_m [\mathbf{g_i}]_m]^T||_2^2\right]\\
    &=\sum_{j=1}^mE[\sigma_j^2]E[[\mathbf{g_i}]_j^2] = \phi\left(\mathbf{J_iJ_i}^T\right) E\left[||\mathbf{g_{i}}||_2^2\right].
\end{split}
\label{equ:trans_between_layer}
\end{equation}
The above derivation is valid when $\mathbf{g_i}$ is invariant under rotation. Therefore, if we want to calculate the $L_2$ norm of the gradient of all the layers with Equation \eqref{equ:trans_between_layer}, we have to make sure that any rotation of its intermediate value will not change $\phi\left((\Pi_{i} \mathbf{J_i})(\Pi_{i} \mathbf{J_i})^T\right)$.
\vspace{-10pt}
\subsection{Main Theorems}\label{sec:main_theorems}

Inspired by the previous inspiration as well as Tarnowski et al. (2018) \cite{tarnowski2018dynamical}'s study, we formulate the main theorems of this paper as below.

\begin{definition}
\textbf{($\mathbf{k^{th}}$ Moment Unitarily Invariant)} Let $\{\mathbf{A_i}\} := \{\mathbf{A_1}, \mathbf{A_2}...,\mathbf{A_L}\}$ be a series independent random matrices. Let $\{\mathbf{U_i}\}:=\{\mathbf{U_1}, \mathbf{U_3}...,\mathbf{U_L}\}$ be a series independent haar unitary matrices independent of  $\{\mathbf{A_1}, \mathbf{A_2}...,\mathbf{A_L}\}$.
We say that $(\Pi_{i} \mathbf{A_i})(\Pi_{i} \mathbf{A_i})^T$ is the $k^{th}$ moment unitarily invariant if $\forall 0<p\le k$, we have
\begin{equation}
    \phi\left(\left((\Pi_{i} \mathbf{A_i})(\Pi_{i} \mathbf{A_i})^T\right)^p\right) = \phi\left(\left((\Pi_{i} \mathbf{U_iA_i})(\Pi_{i} \mathbf{U_iA_i})^T\right)^p\right).
\end{equation}
And we say that $(\sum_{i} \mathbf{A_i})(\sum_{i} \mathbf{A_i})^T$ is $k^{th}$ moment unitarily invariant if $\forall 0<p\le k$, we have
\begin{equation}
    \phi\!\left(\!\left(\!(\sum_{i}\! \mathbf{A_i})\!(\sum_{i} \mathbf{A_i})^T\!\right)^p\!\right)\! =\! \phi\!\left(\!\left(\!(\!\sum_{i}\! \mathbf{U_i}\mathbf{A_i})(\sum_{i} \mathbf{U_i}\!\mathbf{A_i}\!)^T\!\right)^p\!\right).
\end{equation}
\label{def:moment_unitary_invariant}
\end{definition}

\begin{definition}
\textbf{(Central Matrix)} A matrix $\mathbf{A}$ is called a central matrix if  $\forall i, j$, we have $E[[\mathbf{A}]_{i,j}]=0$.
\label{def:central_matrix}
\end{definition}

\begin{definition}
\textbf{(R-diagonal Matrices)} (Definition 17 in Cakmak (2012) \cite{cakmak2012non}) A random matrix $\mathbf{X}$ is R-diagonal if it can be decomposed as $\mathbf{X}=\mathbf{UY}$, such that $\mathbf{U}$ is Haar unitary and free of $\mathbf{Y}=\sqrt{\mathbf{XX}^H}$. 
\label{def:R-diagonal}
\end{definition}

\begin{theorem}
\textbf{(Multiplication).} Given $\mathbf{J} := \Pi_{i=L}^1\mathbf{J_i}$, where $\{\mathbf{J_i}\in\mathbb{R}^{m_{i}\times m_{i-1}}\}$ is a series of independent random matrices. If $(\Pi_{i=L}^1\mathbf{J_i})(\Pi_{i=L}^1\mathbf{J_i})^T$ is at least the $1^{st}$ moment unitarily invariant (Definition \ref{def:moment_unitary_invariant}), we have
\begin{equation}
    \phi\left((\Pi_{i=L}^1\mathbf{J_i})(\Pi_{i=L}^1\mathbf{J_i})^T\right) = \Pi_{i=L}^1\phi\left(\mathbf{J_iJ_i}^T\right).
\label{equ:mul_expectation}
\end{equation}
If $(\Pi_{i=L}^1\mathbf{J_i})(\Pi_{i=L}^1\mathbf{J_i})^T$ is at least the $2^{nd}$ moment unitarily invariant (Definition \ref{def:moment_unitary_invariant}), we have
\begin{equation}
\begin{split}
    & \varphi\left((\Pi_{i=L}^1\mathbf{J_i})(\Pi_{i=L}^1\mathbf{J_i})^T\right)\\
    & =\phi^2\left((\Pi_{i=L}^1\mathbf{J_i})(\Pi_{i=L}^1\mathbf{J_i})^T\right) \sum_{i}\frac{m_{L}}{m_{i}}\frac{\varphi\left(\mathbf{J_iJ_i}^T\right)}{\phi^2\left(\mathbf{J_iJ_i}^T\right)}.
\end{split}
\label{equ:mul_variance}
\end{equation}
(Proof: Appendix \ref{proof:multiplication})
\label{theorem:multiplication}
\end{theorem}

\input{tables/components.tex}

\begin{theorem}
\textbf{(Addition).} Given $\mathbf{J} := \sum_{i}\mathbf{J_i}$, where $\{\mathbf{J_i}\}$ is a series of independent random matrices. If at most one matrix in $\{\mathbf{J_i}\}$ is not a central matrix (Definition \ref{def:central_matrix}), we have
\begin{equation}
    \phi\left(\mathbf{JJ}^T\right)=\sum_i \phi\left(\mathbf{J_iJ_i}^T\right).
\label{equ:add_expectation}
\end{equation}
If $(\sum_{i}\mathbf{J_i})(\sum_{i}\mathbf{J_i})^T$ is at least the $2^{nd}$ moment unitarily invariant (Definition \ref{def:moment_unitary_invariant}), and $\forall i, \mathbf{U_iJ_i}$ is R-diagonal (Definition \ref{def:R-diagonal}), we have
\begin{equation}
    \varphi\left(\mathbf{JJ}^T\right) = \phi^2\left(\mathbf{JJ}^T\right) + \sum_i \varphi\left(\mathbf{J_iJ_i}^T\right) - \phi^2\left(\mathbf{J_iJ_i}^T\right).
\label{equ:add_variance}
\end{equation}
(Proof: Appendix \ref{proof:addition})
\label{theorem:Addition}
\end{theorem}

All in all, Definition \ref{def:moment_unitary_invariant} defines the rotational-invariant system described in Section \ref{sec:inspiration}, and Theorem \ref{theorem:multiplication} and \ref{theorem:Addition} handle the serial and parallel connections in neural networks, respectively.

\subsection{Discussion of Prerequisites}

Although sufficient conditions of Theorem \ref{theorem:multiplication} and \ref{theorem:Addition} are provided, it is still difficult to judge whether a series of Jacobian matrices satisfies them. In this section, we further provide a few sufficient conditions of Definition \ref{def:moment_unitary_invariant} that are much easier to use.

\begin{definition}
\textbf{(Expectant Orthogonal Matrix)} A random matrix $\mathbf{J}$ is called an expectant orthogonal matrix if it satisfies: \textcircled{1} $\forall i, p\neq i$, $E[[\mathbf{J}^T\mathbf{J}]_{p,i}]=0$; \textcircled{2} $\forall i,j$, $E[[\mathbf{J}^T\mathbf{J}]_{i,i}]=E[[\mathbf{J}^T\mathbf{J}]_{j,j}]$.
\label{def:expectattion_diagonal_matrix}
\end{definition}

\begin{prop}
$(\Pi_{i=L}^1 \mathbf{J_i})(\Pi_{i=L}^1 \mathbf{J_i})^T$ is at least the $1^{st}$ moment unitary invariant if: \textcircled{1} $\forall i,j\neq i$, $\mathbf{J_i}$ is independent of $\mathbf{J_j}$; \textcircled{2} $\forall i \in [2, L]$, $\mathbf{J_i}$ is an expectant orthogonal matrix. (Proof: Appendix \ref{proof:1st_moment_unitary_invariant_exp_diag})
\label{prop:1st_moment_unitary_invariant_exp_diag}
\end{prop}
\modify{
\begin{corollary}
With Proposition \ref{prop:1st_moment_unitary_invariant_exp_diag}, as long as $\forall j, \mathbf{J_j}$ is an expectant orthogonal matrix, $(\Pi_{j=L}^{i+1}\mathbf{J_j})\frac{\partial \mathbf{f_{i}}}{\partial \mathbf{\theta_{i}}}$ is the $1^{st}$ moment unitarily invariant. According to Theorem \ref{theorem:multiplication}, as long as $(\Pi_{j=L}^{i+1}\mathbf{J_j})\frac{\partial \mathbf{f_{i}}}{\partial \mathbf{\theta_{i}}}$ is the $1^{st}$ moment unitarily invariant, the decomposition in Equation \eqref{equ:gradient_phi_decomposite} holds and Hypothesis \ref{assumption:jacobian_breakdown} is confirmed.
\label{remark:hyp31}
\end{corollary}
}
\begin{prop}
\textbf{(Properties of Expectant Orthogonal Matrices and Central Matrices)}
\begin{itemize}
\item If $\{\mathbf{J_i}\}$ be a series independent expectant orthogonal matrices, $\Pi_i\mathbf{J_i}$ is also an expectant orthogonal matrix.
    \item If $\mathbf{J_i}$ be a central matrix, for any random matrix $\mathbf{A}$ independent of $\mathbf{J_i}$, $\mathbf{J_iA} and \mathbf{AJ_i}$ are also central matrices.
\end{itemize}
(Proof: Appendix \ref{proof:properties_of_eom_cm})
\label{prop:properties_of_eom_cm}
\end{prop}

Proposition \ref{prop:1st_moment_unitary_invariant_exp_diag} and \ref{prop:properties_of_eom_cm} are two sufficient tools that allow us to judge whether a network structure satisfies the prerequisites by evaluating its components. For the $2^{nd}$ moment unitary invariant, we will discuss in specific cases when required. Notably, as the conditions we provided are sufficient but not necessary, Theorem \ref{theorem:multiplication} and \ref{theorem:Addition} may still hold for networks that do not satisfy these conditions.

\subsection{Components Library}\label{sec:parts_library}
Here we provide a library of some most commonly used components in neural networks. We theoretically analyze the expectation and variance of their input-output Jacobian matrix $\mathbf{J}$'s eigenvalues as well as whether $\mathbf{J}$ satisfies Definition \ref{def:expectattion_diagonal_matrix} and \ref{def:central_matrix}. \modify{The detailed proofs can be found in Appendix \ref{proof:part_library}}.

%% file: tables/components.tex
\begin{table*}[t]
\caption{Common components in neural networks (Proof: Appendix \ref{proof:part_library}).}
\centering

\resizebox{0.98\textwidth}{!}{
\begin{threeparttable}[b]
\begin{tabular}{c|c|c|c|c}
\hline
Part & $\phi(\mathbf{JJ}^T)$ & $\varphi(\mathbf{JJ}^T)$&\textbf{Def. \ref{def:expectattion_diagonal_matrix}}&\textbf{Def. \ref{def:central_matrix}}\\
\hline \hline
\multicolumn{5}{c}{Activation Functions \tnote{1}}\\
\hline
ReLU($P(x>0)=p$)&$p$&$p-p^2$&${\surd}$&${\times}$\\
\hline
leaky ReLU($P(x>0)=p$), $\gamma$: negative slop coefficient &$p+\gamma^2(1-p)$&\tabincell{c}{$\gamma^4(1-p)+p-$\\$(p+\gamma^2(1-p))^2$}&${\surd}$&${\times}$\\
\hline
tanh &1&0&${\surd}$&${\times}$\\
\hline
\multicolumn{5}{c}{Linear Transformations}\\
\hline
Dense($\mathbf{u}:=\mathbf{K}\mathbf{y}$), $\mathbf{K}\in \mathbb{R}^{m\times n}\sim i.i.d. N(0, \sigma^2)$&$n\sigma^2$&$mn\sigma^4$&${\surd}$&${\surd}$\\
\hline
CONV($\mathbf{u}:=\mathbf{K}\star\mathbf{y}$), $\mathbf{K}\in \mathbb{R}^{c_{out}\times c_{in}\times k_h\times k_w}\sim i.i.d. N(0, \sigma^2)$&$c_{in}\widetilde{k_hk_w}\sigma^2, \tnote{1}$&&${\surd}$&${\surd}$\\
\hline
Orthogonal($\mathbf{u}:=\mathbf{K}\mathbf{y}$, $\mathbf{K}\mathbf{K}^T=\beta^2\mathbf{I}$) &$\beta^2$&0&${\surd}$&${\surd}$\\
\hline
\multicolumn{5}{c}{Normalization}\\
\hline
Data Normalization($\mathbf{u}:=norm(\mathbf{y})$), $D[\mathbf{y}\in\mathbb{R}^{m\times 1}] = \sigma_B^2$&$\frac{1}{\sigma_B^2}$&$\frac{2}{m\sigma_B^4}$&${\surd}$&${\times}$\\
\hline
\end{tabular}
\begin{tablenotes}
    \item [1] The $\widetilde{k_hk_w}$ denotes the effective kernel size, which can be simply calculated from Algorithm \ref{Alg:conv_adjust}.
\end{tablenotes}
\end{threeparttable}
}
\label{tab:parts_library}
\vspace{-15pt}
\end{table*}

%% file: chapters/series_network.tex
A serial neural network is defined as the neural network whose components are connected in serial, such as LeNet \cite{lecun1998gradient} and VGG \cite{simonyan2014very}. We will show that powered by our framework, \modify{the conclusions of several previous studies including initialization \cite{he2015delving,xiao2018dynamical}, normalization \cite{salimans2016weight,arpit2016normalization,ioffe2015batch}, self-normalizing neural network \cite{klambauer2017self} and DenseNet\cite{huang2017densely} can be easily reached or even surpassed with several lines of derivation.}

As revealed in Table \ref{tab:parts_library}, all the listed parts satisfy Definition \ref{def:expectattion_diagonal_matrix}, thus we have the following propositions:
\begin{prop}
For any neural network, if it is composed of parts given in Table \ref{tab:parts_library} and its Jacobian matrix can be calculated by $\mathbf{J} = (\Pi_i\mathbf{J_i})$, then $(\Pi_i\mathbf{J_i})(\Pi_i\mathbf{J_i})^T$ is at least the $1^{st}$ moment unitary invariant.
\label{prop:prerequisite_series_networks}
\end{prop}
\begin{proof}
According to Table \ref{tab:parts_library}, all the components satisfy Definition \ref{def:expectattion_diagonal_matrix}, so they are all expectant orthogonal matrices. Under the assumption that the Jacobian matrices of different components are independent, according to Proposition \ref{prop:1st_moment_unitary_invariant_exp_diag}, we have $(\Pi_i\mathbf{J_i})(\Pi_i\mathbf{J_i})^T$ is at least the $1^{st}$ moment unitary invariant.
\end{proof}

Proposition \ref{prop:prerequisite_series_networks} reflects that Equation \eqref{equ:mul_expectation} is applicable in this section.

\subsection{Initialization Techniques}\label{sec:initialization_serial}
\input{chapters/series_network/initialization.tex}

\vspace{-10pt}
\subsection{Normalization Techniques}
\input{chapters/series_network/normalization.tex}
\subsection{Self-Normalizing Neural Network}\label{sec:self_normalizing_nn}
\input{chapters/series_network/selu.tex}
\vspace{-10pt}
\subsection{Shallow Network Trick}
\input{chapters/shallow_network.tex}
\vspace{-10pt}
\subsection{DenseNet}
\input{chapters/series_network/densenet.tex}

%% file: chapters/series_network/initialization.tex
It has long been aware that a good initialization of network parameters can significantly improve the convergence and make deeper networks trainable \cite{mishkin2015all,he2015delving,xiao2018dynamical}. In this subsection, we will discuss some of the most popular initialization techniques. We consider a simple network block with a single linear transform (the weight kernel is $\mathbf{K}\in \mathbb{R}^{m\times n}$) and an activation function. The Jacobian matrix of the whole block is denoted as $\mathbf{J_i}$.

Since the activation functions are commonly applied right after linear transforms, we assume that the mean of input pre-activations is zero, thus $p=P(x>0)=1/2$. Moreover, Equation \eqref{equ:mul_variance} can be applied if the kernel follows i.i.d. Gaussian distribution.
\begin{prop}
A neural network is the $\infty^{th}$ moment unitarily invariant if it is composed of cyclic central Gaussian transform with i.i.d. entries and any network components. (Proof: Appendix \ref{proof:gaussian_activ_2nd_invariant})
\label{prop:gaussian_activ_2nd_invariant}
\end{prop}

\textbf{Kaiming Normal (KM) \cite{he2015delving}.} We denote the Jacobian matrix of the $i^{th}$ layer as $\mathbf{J_i}$. With Equation \eqref{equ:mul_expectation}-\eqref{equ:mul_variance}, we have 
\begin{equation}
\begin{split}
    & \phi\left(\mathbf{J_iJ_i}^T\right)=n\sigma^2\times\frac{1}{2}=\frac{1}{2}\sigma^2n,\\
    & \varphi\left(\mathbf{J_iJ_i}^T\right) =\phi^2(\mathbf{J_iJ_i}^T)\left(\frac{m}{m}\frac{\frac{1}{4}}{(\frac{1}{2})^2} + \frac{m}{m}\frac{mn\sigma^4}{(n\sigma^2)^2}\right).
\end{split}
\end{equation}
In order to achieve the block dynamical isometry, we force $\phi(\mathbf{J_iJ_i}^T)=1$, and we have $\sigma=\frac{\sqrt{2}}{\sqrt{n}}$, $\varphi(\mathbf{J_iJ_i}^T)=1+\frac{m}{n}$. For fully-connected layers, $n$ denotes the width of the weight matrix; for convolutional layers, $n=c_{in}\widetilde{k_hk_w}$, and $c_{in}=1$ if in point-wise convolution \cite{howard2017mobilenets}. Although using the effective kernel size $\widetilde{k_hk_w}$ provides more accurate estimation, we empirically find that replacing $k_hk_w$ with $\widetilde{k_hk_w}$ only has trifling impact on accuracy. The reason is that most of the feature maps are large enough, and the cutting-off effect caused by padding (see Appendix \ref{proof:part_library}) is less significant compared with other factors like parameter update. The optimal $\sigma$ for other activation functions like leaky ReLU and tanh can be obtained in the same way, and we summarize the conclusions in Table \ref{tab:activation_fns_gaussian}.

\begin{table}[ht]
\vspace{-10pt}
\caption{Optimal $\sigma$ for ReLU, leaky ReLU and tanh with Gaussian kernel.}
\centering
\resizebox{0.48\textwidth}{!}{
\begin{tabular}{c|c|c|c}
\hline
  &ReLU &\tabincell{c}{leaky ReLU. $\gamma$: negative\\ slope coefficient} & tanh
\\ \hline \hline
$\phi(\mathbf{J_iJ_i}^T)$&$\frac{1}{2}\sigma^2n$&$\frac{1}{2}\sigma^2n(1+\gamma^2)$&$\approx \sigma^2n$\\
\hline
Optimal $\sigma$&$\frac{\sqrt{2}}{\sqrt{n}}$&$\sqrt{\frac{2}{n(1+\gamma^2)}}$&$\frac{1}{\sqrt{n}}$\\
\hline
\tabincell{c}{$\varphi(\mathbf{J_iJ_i}^T)$\\under\\optimal $\sigma$}&$1+\frac{m}{n}$&$\left(\frac{1-\gamma^2}{1+\gamma^2}\right)^2 + \frac{m}{n}$&$\frac{m}{n}$\\
\hline
\end{tabular}}
\label{tab:activation_fns_gaussian}
\vspace{-5pt}
\end{table}

\textbf{sPReLU. } Instead using a fixed negative slop coefficient like leaky ReLU, PReLU \cite{he2015delving} replaces it with a trainable parameter $\alpha$. Although He et al. (2015) \cite{he2015delving} initialize weights with $\frac{\sqrt{2}}{\sqrt{n}}$, we find it might be problematic when the network is deep: for an $L$-layer network, we have $\Pi_{i=L}^1\phi(\mathbf{J_iJ_i}^T)=(1+\alpha^2)^L$. He et al. (2015) \cite{he2015delving} found that the learned $\alpha$ in some layers are significantly greater than 0, therefore the original setup may not be stable in relatively deep networks.

Since $\alpha$ is kept updating during training, it is difficult to address the above issue by initialization. So we modify PReLU by simply rescaling the output activations with $\frac{1}{\sqrt{1+\alpha^2}}$ as follows, which is named as ``sPReLU":
\begin{equation}
        sPReLU(x) = \frac{1}{\sqrt{1+\alpha^2}}\left\{
             \begin{array}{lr}
             x~~~~~~~~~~if~~x>0  \\
             \alpha x~~~~~~~~if~~x \le0
             \end{array}
\right..
\end{equation}
With the rescaling, we have $\Pi_{i=L}^1\phi(\mathbf{J_iJ_i}^T)=1$. However, leaving $\alpha$ without any constraint may lead to an unstable training process, thus we clip it within $[0. 0.5]$.

\textbf{Orthogonal Initialization.}
\modify{Because our target is to have $\phi(\mathbf{J_iJ_i}^T)=1$ and $\varphi(\mathbf{J_iJ_i}^T)\approx 0$, one intuitive idea is to initialize $\mathbf{J_i}$ to have orthogonal rows. }Proposition \ref{prop:orth_activ_2nd_invariant} demonstrates that Equation \eqref{equ:mul_variance} is applicable for blocks consisting of orthogonal kernels and any activation functions.

\begin{prop}
A neural network block composed of an orthogonal transform layer and any activation function is at least the $2^{nd}$ moment unitarily invariant. (Proof: Appendix \ref{proof:orth_activ_2nd_invariant})
\label{prop:orth_activ_2nd_invariant}
\end{prop}

As illustrated in Table \ref{tab:parts_library}, from the perspective of $\phi$, $\beta^2$ is equivalent with $n\sigma^2$, therefore we can easily obtain the optimal $\beta$ for different activation functions.
\begin{table}[ht]
\vspace{-10pt}
\caption{Optimal $\sigma$ for ReLU, leaky ReLU and tanh with orthogonal kernel.}
\centering
\resizebox{0.48\textwidth}{!}{
\begin{tabular}{c|c|c|c}
\hline
  &ReLU &\tabincell{c}{leaky ReLU. $\gamma$: negative\\ slope coefficient} & tanh
\\ \hline \hline
Optimal $\beta$&$\sqrt{2}$&$\sqrt{\frac{2}{1+\gamma^2}}$&$\approx1$\\
\hline
\tabincell{c}{$\varphi(\mathbf{J_iJ_i}^T)$\\under\\optimal $\beta$}&$1$&$\left(\frac{1-\gamma^2}{1+\gamma^2}\right)^2$&$\approx0$\\
\hline
\end{tabular}}
\label{tab:activation_fns_orth}
\vspace{-5pt}
\end{table}

\textbf{Comparison.} Table \ref{tab:activation_fns_gaussian}\&\ref{tab:activation_fns_orth} show that with proper initialization, all the activation functions can achieve $\phi(\mathbf{J_iJ_i}^T)=1$, whereas their $\varphi(\mathbf{J_iJ_i}^T)$ are quite different. For example, ReLU has the highest $\varphi$, while tanh has the lowest with more stability. However, since rectifiers like ReLU has non-saturating property \cite{krizhevsky2012imagenet} and produce sparse representations \cite{glorot2011deep}, they are usually more effective than tanh. \modify{Besides, unlike rectifiers that preserve the forward and backward flows simultaneously\cite{arpit2019benefits}, we find that the $2^{th}$ moment of the forward information is diminished with tanh.}

Leaky ReLU provides us an opportunity to trade off between stability and nonlinearity. Although its nonlinearity is the most effective when $\gamma$ is around a certain value (i.e. $1/5.5$ \cite{xu2015empirical}), a relatively greater $\gamma$ can effectively reduce $\varphi$. However, the optimal $\gamma$ has to be explored experimentally.

sPReLU has a similar effect with leaky ReLU, as argued in He et al. (2015) \cite{he2015delving}, but it learns a greater $\alpha$ to keep more information in the first few layers, which provides more stability. In the later stage when the nonlinearity is required, the learned $\alpha$ is relatively small to preserve the nonlinearity.

The comparison of $\varphi(\mathbf{J_iJ_i}^T)$ under optimal initialization in Table \ref{tab:activation_fns_gaussian}\&\ref{tab:activation_fns_orth} also indicates that the orthogonal initialization provides much lower $\varphi$ compared with the Gaussian initialization, since the orthogonal kernel's $\varphi$ is $0$.

\modify{
\textbf{Relationship to Existing Studies.} In some network structures, our theorems can even be used to analyze the information flow in the forward pass owing to Proposition \ref{prop:general_linear_transforms} and \ref{prop:evolution_2nd_norm_net }.

\begin{definition}
\textbf{(General Linear Transform)} Let $\mathbf{f}(\mathbf{x})$ be a transform whose Jacobian matrix is $\mathbf{J}$. $\mathbf{f}$ is called general linear transform when it satisfies:
\begin{equation}
    E\left[\frac{||\mathbf{f}(\mathbf{x})||_2^2}{len(\mathbf{f}(\mathbf{x}))}\right] = \phi\left(\mathbf{JJ}^T\right)E\left[\frac{||\mathbf{x}||_2^2}{len(\mathbf{x})}\right].
\end{equation}
\label{def:general_linear_transformation}
\end{definition}

\begin{prop}
Data normalization with 0-mean inputs, linear transforms and rectifier activation functions are general linear transforms (Definition \ref{def:general_linear_transformation}). (Proof: Appendix \ref{proof:general_linear_transforms})
\label{prop:general_linear_transforms}
\end{prop}

\begin{prop}
    For a serial neural network $\mathbf{f}(\mathbf{x})$ composed of general linear transforms and its input-output Jacobian matrix is $\mathbf{J}$, we have
    \begin{equation}
    E\left[\frac{||\mathbf{f}(\mathbf{x})||_2^2}{len(\mathbf{f}(\mathbf{x}))}\right] = \phi\left(\mathbf{JJ}^T\right)E\left[\frac{||\mathbf{x}||_2^2}{len(\mathbf{x})}\right].
    \end{equation}
    (Proof: Appendix \ref{proof:evolution_2nd_norm_net})
\label{prop:evolution_2nd_norm_net }
\end{prop}

According to Proposition \ref{prop:general_linear_transforms}, the rectifier activation functions and linear transforms in He et al. (2015) \cite{he2015delving} are all general linear transforms. Therefore, Proposition \ref{prop:evolution_2nd_norm_net } shows that $\phi(\mathbf{JJ}^T)$ also describes the evolution of the $2{nd}$ moment/variance of activations in the forward pass, which is equivalent to He et al. (2015) \cite{he2015delving} but more convenient.
}

%% file: chapters/series_network/normalization.tex
Even if the parameters in a neural network are properly initialized with schemes proposed in the previous subsection, there is no guarantee that their statistic properties remain unchanged during training, especially when the parameters are updated with a large learning rate. To address this issue, normalization techniques are introduced to maintain the parameters' statistic properties during training.

\textbf{Weight Normalization (WN) \cite{salimans2016weight}.} Let $\mathbf{K}$ denote the weight matrix, WN can be represented as $\hat{\mathbf{K}} = \frac{g}{||\mathbf{K}||}\mathbf{K}$, where $g$ is a constant scaling factor and $\hat{\mathbf{K}}$ is what we use for training and inference. To further improve the performance, a mean-only batch normalization is usually applied \cite{salimans2016weight}. Under this setup, the standard deviation of a normalized kernel is $\sigma_{\hat{\mathbf{K}}} = g$ and $\phi(\mathbf{JJ}^T)=ng^2$.
Salimans \& Kingma (2016) \cite{salimans2016weight} take $g=e^s/\sqrt{n}$, which may not be the optimal setup of activation functions, for there is no guarantee that $\phi(\mathbf{JJ}^T)=e^{2s}\approx1$. Therefore, it has been observed that WN is less stable in deep networks \cite{gitman2017comparison}.

\textbf{Scaled Weight Standardization (sWS).} Inspired by WN, we propose a new weight-related normalization technique, which is defined as: $\hat{\mathbf{K}} = \frac{g}{\sigma_{\mathbf{K}}}(\mathbf{K}-\mu_{\mathbf{K}})$, where $\mu_{\mathbf{K}}$ and $\sigma_{\mathbf{K}}$ denote the kernel's mean and variance, respectively. Therefore, we have $\mu_{\hat{\mathbf{K}}}=0$ and $\sigma_{\hat{\mathbf{K}}}=ng^2$, and the mean-only batch normalization is no longer required. As the most intuitive idea is to normalize the weights to ``Kaiming Normal (KM)" during the training, the optimal $g$ for different activation functions are listed in Table \ref{tab:activation_fns_weight_norm}.
\vspace{-5pt}
\begin{table}[ht]
\caption{Optimal $g$ for ReLU, leaky ReLU and tanh with sWS.}
\vspace{-10pt}
\centering
\resizebox{0.48\textwidth}{!}{
\begin{tabular}{c|c|c|c}
\hline
  &ReLU &\tabincell{c}{leaky ReLU. $\gamma$: negative\\ slope coefficient} & tanh
\\ \hline \hline
Optimal $g$&$\frac{\sqrt{2}}{\sqrt{n}}$&$\sqrt{\frac{2}{n(1+\gamma^2)}}$&$\frac{1}{\sqrt{n}}$\\
\hline
\end{tabular}}
\label{tab:activation_fns_weight_norm}
\vspace{-5pt}
\end{table}

Similar conclusion for ReLU has been reached by previous studies. For instance, Arpit et al. (2016) \cite{arpit2016normalization} propose a scheme called Normprop which normalizes the hidden layers with theoretical estimation of the following distribution:
\begin{equation}
    o_i = \frac{1}{\sqrt{\frac{1}{2}(1-\frac{1}{\pi})}}\left[ReLU\left(\frac{\gamma_i\mathbf{W}_i^T\mathbf{x}}{||\mathbf{W_i}||_F}+\beta_i\right)-\sqrt{\frac{1}{2\pi}}\right],
\end{equation}
where $o_i$ denotes the $i^{th}$ output, and $\gamma_i$ and $\beta_i$ are trainable parameters initialized with $1/1.21$ and $0$, respectively. $\mathbf{W}_i$ is the weight corresponding to the $i^{th}$ input $\mathbf{x}_i$. Because of $||\mathbf{W}_i||_F=\sqrt{n\sigma_W^2}$ and $\frac{1}{1.21\sqrt{\frac{1}{2}(1-\frac{1}{\pi})}}\approx1.415\approx\sqrt{2}$, the scaling of the weight is exactly the same with our derivation. 

\textbf{Data Normalization (DN) \cite{ioffe2015batch, ulyanov2016instance, wu2018group}.} This has become a regular component of deep neural networks, for it enables us to train deeper networks, use large learning rates and apply arbitrary initialization schemes \cite{ioffe2015batch}. In DN, the pre-activations are normalized to $N(0, 1)$ by
\begin{equation}
    \hat{\mathbf{x}} = \frac{\mathbf{x} - E[\mathbf{x}]}{\sqrt{D[\mathbf{x}] + \epsilon}},~~\mathbf{y} = \mathbf{\gamma}\hat{\mathbf{x}} + \mathbf{\beta}.
\label{equ:data_normalization}
\end{equation}
DN can be explained by slightly extending Proposition \ref{prop:evolution_2nd_norm_net }.
\begin{prop}
    We consider a serial network block composed of general linear transforms (Definition \ref{def:general_linear_transformation}). The $2^{nd}$ moment of the block's input activation is  $\alpha_2^{(0)}$ and the block's Jacobian matrix is $\mathbf{J}$. If the Jacobian matrix of its last component $\mathbf{J_l}$ satisfies $\phi(\mathbf{J_lJ_l}^T)=\frac{\beta}{\alpha_2^{(l-1)}}$ wherein $\beta$ is a constant value and $\alpha_2^{(l-1)}$ is the $2^{nd}$ moment of its input data, then we have $\phi(\mathbf{JJ}^T)=\frac{\beta}{\alpha_2^{(0)}}$.
\label{coro:bn_explain}
\end{prop}

\begin{proof}
Since the network is composed of general linear transforms, with Proposition \ref{prop:evolution_2nd_norm_net }, we have
\begin{equation}
    \alpha_2^{(l-1)}=\Pi_{i=l-1}^1\phi\left(\mathbf{J_iJ_i}^T\right)\alpha_2^{(0)}
\label{equ:le_pre_info}.
\end{equation}
Therefore, we further have
\begin{equation}
    \phi\left(\mathbf{JJ}^T\right) =\Pi_{i=l-1}^1\phi\left(\mathbf{J_iJ_i}^T\right)\frac{\beta}{\Pi_{i=l-1}^1\phi\left(\mathbf{J_iJ_i}^T\right)\alpha_2^{(0)}}\!=\!\frac{\beta}{\alpha_2^{(0)}}.
\end{equation}
\end{proof}

According to Ioffe \& Szegedy (2015) \cite{ioffe2015batch}, DN is performed right after the linear transforms, thus its inputs have zero-mean, and we further have $\sigma_B^2=\alpha_{2,B}$. For instance, the input of the block shown in Fig. \ref{fig:simpleDN} is the output of a BN layer, therefore its $2^{nd}$ moment $\alpha_2^{(0)}$ is 1. With Proposition \ref{coro:bn_explain}, we have $\phi(\mathbf{JJ}^T)=1$, thus DN can effectively address gradient explosion or vanishing.
\begin{figure}[ht!]
\centering
\includegraphics[width=0.48\textwidth]{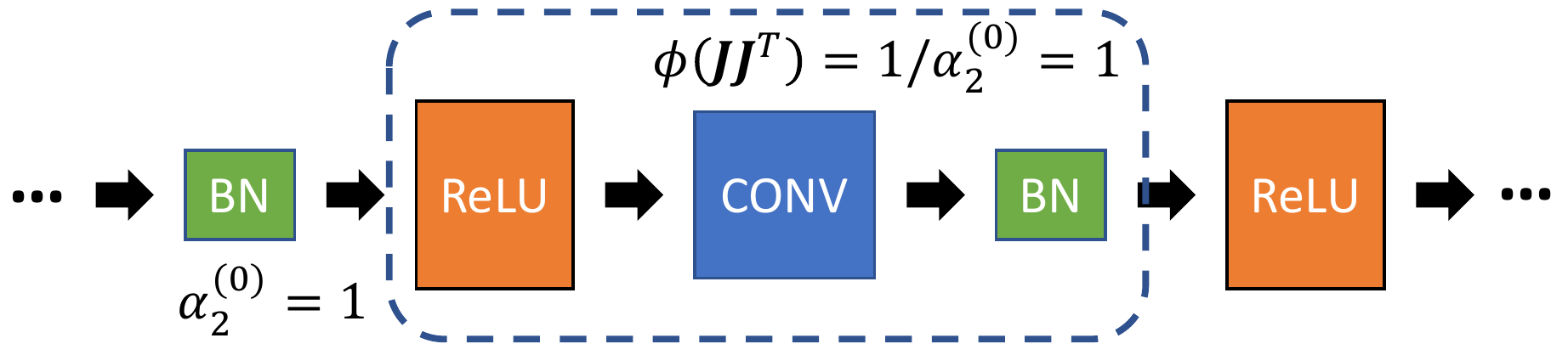}
\caption{Example block for Proposition \ref{coro:bn_explain}.}
\label{fig:simpleDN}
\vspace{-5pt}
\end{figure}

Proposition \ref{coro:bn_explain} can be interpreted from another prospective. As illustrated in Equation \eqref{equ:le_pre_info}, in a network composed of general linear transforms, the pre-activations' $2^{nd}$ moment $\alpha_2^{(l-1)}$ contains the information about the status of all layers it has passed through, and a network component with $\phi\propto \frac{1}{\alpha_2^{(l-1)}}$ can effectively offset the affect of these layers. This explains why DN techniques like BN is more stable with less awareness of the initialization and being sustainable to high learning rate.

\textbf{Comparison.} The common topic for all the normalization techniques is standardizing the $1^{st}$ and $2^{nd}$ moments of the pre-activations, and the only difference is what the moments are estimated upon. Specifically, DN gets its $1^{st}$ and $2^{nd}$ moments from the pre-activations, while WN estimates the $2^{nd}$ moment from the weight kernel. However, different sources of estimation will result in different execution efficiency, stability and convenience.

For execution efficiency, as the weight kernels usually contain fewer data compared with pre-activations, estimating moments from the weight kernels usually has lower computational overhead. For stability, while WN depends on the ``Gaussian Assumption", which is not necessarily held during training, Proposition \ref{coro:bn_explain} is valid for any linear transforms. Moreover, each WN sweeps the snow from its own doorstep, whereas DN improves the condition of all the layers before it, thus even if one or two DN layers malfunction, the following ones would compensate for them. For convenience, as WN is born out of weight initialization, its hyper-parameters require carefully selection, which makes it less suitable for complex network structures. Oppositely, DN can automatically improve the network's condition without handcrafted hyper-parameters.

\textbf{Second Moment Normalization (SMN).} Inspired by the above comparison, we propose a novel normalization method called SMN, wherein the $1^{st}$ moment is obtained from the weight kernel while the $2^{nd}$ moment is estimated from the pre-activations. In SMN, the pre-activations are normalized by
\begin{equation}
    \hat{\mathbf{x}} = \frac{\mathbf{x}}{\sqrt{E\left[[\mathbf{x}]_i^2\right]}}, \mathbf{y}=\gamma\hat{\mathbf{x}} + \beta.
\end{equation}
Since our derivation is based on the assumption that the weight kernels have zero expectation, which may be violated during training, we further add weight centralization onto each weight:
\begin{equation}
    \hat{\mathbf{K}} = \mathbf{K} - E[[\mathbf{K}]_i].
\end{equation}

For stability and convenience, similar to DN, we have
\begin{equation}
    \phi\left(\hat{\mathbf{x}}_{\mathbf{x}}\hat{\mathbf{x}}_{\mathbf{x}}^T\right) \approx \frac{1}{\alpha_2^2},~~\varphi\left(\hat{\mathbf{x}}_{\mathbf{x}}\hat{\mathbf{x}}_{\mathbf{x}}^T\right) \approx 0,
\label{equ:eigs_l2norm}
\end{equation}
where $\hat{\mathbf{x}}_{\mathbf{x}}$ satisfies Definition \ref{def:expectattion_diagonal_matrix} but defies Definition \ref{def:central_matrix}. The proof is in Appendix \ref{proof:eigs_l2norm}. Because of $\phi(\hat{\mathbf{x}}_{\mathbf{x}}\hat{\mathbf{x}}_{\mathbf{x}}^T) \approx \frac{1}{\alpha_2^2}$, the $2^{nd}$ moment normalization can achieve similar effect with DN when applied right after linear transforms. Therefore, SMN is as stable and convenient as DN. For execution efficiency, we have
\begin{equation}
\begin{split}
    & SMN:~~\hat{\mathbf{x}} = \frac{\mathbf{x}}{\sqrt{E\left[[\mathbf{x}]_i^2\right]}},~~\hat{\mathbf{K}} = \mathbf{K} - E[[\mathbf{K}]_i],\\
    & WN:~~\hat{\mathbf{x}} = \mathbf{x} - E\left[[\mathbf{x}]_i\right],~~\hat{\mathbf{K}} = \frac{\mathbf{K}}{E[[\mathbf{K}]_i^2]]}.\\
\end{split}
\label{equ:compare_wn_smn}
\end{equation}
SMN can be viewed as a reversed version of WN, and there is only one additional element-wise square operator compared with WN, so it has much less computational overhead than DN. \modify{Following the analysis in Chen et al. (2019) \cite{chen2019effective}, in Appendix \ref{appendix:l2n_overhead}, we find that SMN reduces the number of operations from 13 to 10, which brings about $30\%$ speedup.}
\input{algorithms/smn_alg.tex}

We provide the detailed algorithm for SMN in the convolutional layer in Algorithm \ref{Alg:smn_alg}. Inspired by Ioffe \& Szegedy (2015) \cite{ioffe2015batch}, we centralize the mean of the weight kernels of each output channel rather than shifting the mean of all weights to zero. Similarly, the $2^{nd}$ moment is also standardized in a channel-wise manner. Also, the trainable parameters $\gamma$ and $\beta$ in BN are introduced to represent the identity transform \cite{ioffe2015batch}. Besides the $2^{nd}$ moment, according to prior work \cite{wu2018l1}, we can also use L1-norm to further reduce the complexity:
\begin{equation}
    \hat{\mathbf{x}} = \frac{\mathbf{x}}{E\left[|[\mathbf{x}]_i|\right]},~~\hat{\mathbf{K}} = \mathbf{K} - E[[\mathbf{K}]_i].
\end{equation}

Although our SMN can statistically replace BN, it somehow has weaker regularization ability, because estimating the $1^{st}$ moment from pre-activations introduces Gaussian noise that can regularize the training process \cite{luo2018towards}. Fortunately, this can be compensated by addition regularization like mixup \cite{zhang2017mixup}.

%% file: algorithms/smn_alg.tex
\vspace{-10pt}
\begin{algorithm}[ht]
\DontPrintSemicolon
 \KwData{Input pre-activation $\mathbf{x}\in[batch\_size, c_{in}, H_i, W_i]$; Convolving kernel: $\mathbf{K}\in [c_{out}, c_{in}, h, w]$; Scaling factor $\mathbf{\gamma}\in [c_{out}]$; Bias $\mathbf{\beta}\in [c_{out}]$}
 \KwResult{Normalized pre-activation $y\in[batch\_size, c_{out}, H_o, W_o]$;}
 \Begin{
 $\mu_K = mean(\mathbf{K}[c_{out}, :])$\\
 $\hat{\mathbf{K}}=\mathbf{K} - \mu_K$  //weight centralization\\
 $\mathbf{x}=\mathbf{K}\ast \mathbf{x}$\\
 $\alpha_2=mean(square(\mathbf{x})[c_{out}, :])$\\
 $\mathbf{y}=\mathbf{\beta} + \frac{\mathbf{\gamma}}{\sqrt{\alpha_2}}\mathbf{x}$
}
 \Return{$\mathbf{y}$}
 \caption{Second Moment Normalization}
 \label{Alg:smn_alg}
\end{algorithm}
\vspace{-10pt}

%% file: chapters/series_network/selu.tex
Klambauer et al. (2017) \cite{klambauer2017self} propose a self-normalizing property empowered by SeLU activation function given by 
\begin{equation}
    SeLU(x) = \lambda\left\{
             \begin{array}{lr}
             x~~~~~~~~~~~~~~if~~x>0  \\
             \alpha e^x-\alpha~~if~~x \le0
             \end{array}
\right. 
\label{equ:SELU}
\end{equation}
where $\alpha\approx 1.6733$, $\lambda \approx 1.0507$. However, the setup in \cite{klambauer2017self} only works for weights whose entries follow $N(0, \frac{1}{n})$. Here we generally let the linear transform have $\phi(\mathbf{JJ}^T)=\gamma_0$.

\begin{prop}
Let $\mathbf{J}$ be the Jacobian matrix of SeLU. When the pre-activations obey $N(0, \sigma^2)$, we have the following conclusions:
\begin{equation}
    \begin{split}
        & \phi\left(\mathbf{JJ}^T\right)=\lambda^2\alpha^2e^{2\sigma^2}cdf(-2\sigma^2, N(0, \sigma^2))+ \frac{\lambda^2}{2},\\
        & E[SeLU^2(\mathbf{x})]=\frac{1}{2}\lambda^2\sigma^2 + \frac{1}{2}\lambda^2\alpha^2 + \\
        & \lambda^2\alpha^2\left(e^{2\sigma^2}cdf(-2\sigma^2, N(0, \sigma^2))\!-\! 2e^{\frac{\sigma^2}{2}}cdf(-\sigma^2, N(0, \sigma^2))\right),\\
        & E[SeLU(\mathbf{x})] = \lambda\alpha e^{\frac{\sigma^2}{2}}cdf(-\sigma^2, N(0, \sigma^2))- \frac{\lambda\alpha}{2} + \sqrt{\frac{\sigma^2}{2\pi}}\lambda.
    \end{split}
\end{equation}
(Proof: Appendix \ref{proof:selu_expect_JJ})
\label{prop:selu_expect_JJ}
\end{prop}

Let SeLU be applied layer-wisely and the $2^{nd}$ moment of output activations have a fixed point of $1$. With Proposition \ref{prop:evolution_2nd_norm_net }, the variance of pre-activations equals $\gamma_0$. Then, with Proposition \ref{prop:selu_expect_JJ}, the optimal $\alpha$ and $\lambda$ can be solved from
\begin{equation}
    \begin{split}
        & \left(\lambda^2\alpha^2e^{2\gamma_0}cdf(-2\gamma_0, N(0, \gamma_0))+ \frac{\lambda^2}{2}\right)\gamma_0 = 1 + \epsilon,\\
        & \lambda^2\alpha^2\left(e^{2\gamma_0}cdf(-2\gamma_0, N(0, \gamma_0)) - 2e^{\frac{\gamma_0}{2}}cdf(-\gamma_0, N(0, \gamma_0))\right) \\
        &\! +\! \frac{1}{2}\lambda^2\alpha^2 + \frac{1}{2}\lambda^2\gamma_0\! =\! 1.
    \end{split}
\label{equ:solution_of_selu}
\end{equation}
The former equation constrains $\phi(\mathbf{J_iJ_i}^T)\approx 1$ and the latter one ensures the fixed point of the $2^{nd}$ moment. $\epsilon$ is a small constant near $0$, which prevents SeLU from degenerating back to ReLU. When $\epsilon=0, \gamma_0=1$, the only solution of Equation \eqref{equ:solution_of_selu} is $\lambda=\sqrt{2}, \alpha=0$, which is equivalent to the KM initialization with ReLU. One explanation is that if $\epsilon=0$, we would have $\alpha_2(x_{out})/\alpha_2(x_{in})=\phi(\mathbf{JJ}^T)$, which is only held when the network satisfies Proposition \ref{prop:evolution_2nd_norm_net }. Notably, the original SeLU in \cite{klambauer2017self} can be solved from Equation \eqref{equ:solution_of_selu} by letting $\gamma_0=1$, $\epsilon\approx0.0716$.

Although $\phi(\mathbf{J_iJ_i}^T)\approx1$ can be achieved from multiple initialization schemes, SeLU's strength comes from its attractive fixed point \cite{klambauer2017self}, which is effective even when the assumptions and initial statistic properties are violated. However, this attractive property takes over 80-page proofs in \cite{klambauer2017self}, so it is challenging to extend to more general situation. In this work, we provide an empirical understanding by analogizing it with data normalization.

In Proposition \ref{coro:bn_explain}, we demonstrate that a network component with $\phi(\mathbf{J_lJ_l}^T)=\frac{\beta}{\alpha_2^{(l-1)}}$ can stabilize the general linear network block based on the information contained in $\alpha_2^{(l-1)}$, here we discuss a more general situation in which $\phi(\mathbf{J_lJ_l}^T)=h_l(\alpha_2^{(l-1)})$ where $h_l$ is a real function. We further assume that 
the network component satisfies Definition \ref{def:general_linear_transformation}. 
When $h_l(\alpha_2^{(l-1)})$ satisfies
\begin{equation}
\begin{split}
    & 1<h_l(\alpha_2^{(l-1)})<\frac{\beta}{\alpha_2^{(l-1)}},~~~~~if~~\alpha_2^{(l-1)}<\beta;\\
    & 1>h_l(\alpha_2^{(l-1)})>\frac{\beta}{\alpha_2^{(l-1)}},~~~~~if~~\alpha_2^{(l-1)}>\beta.
\end{split}
\label{equ:partial_norm_def}
\end{equation}
Since $\Pi_{i=l}^1\phi(\mathbf{J_iJ_i}^T) = h_l(\alpha_2^{(l-1)})\alpha_2^{(l-1)}/\alpha_2^{(0)}$, we have
\begin{equation}
    \left|\Pi_{i=l-1}^1\phi\left(\mathbf{J_iJ_i}^T\right)-\frac{\beta}{\alpha_2^{(0)}}\right|>\left|\Pi_{i=l}^1\phi\left(\mathbf{J_iJ_i}^T\right)-\frac{\beta}{\alpha_2^{(0)}}\right|,\\
\label{equ:partial_norm}
\end{equation}
which illustrates that $\Pi_{i}^1\phi(\mathbf{J_iJ_i}^T)$ converges to the fixed point of $\frac{\beta}{\alpha_2^{(0)}}$. As the convergence may take several layers, we call it as ``partial normalized". Similarly, when $\forall \alpha_2^{(l-1)}$, $h_l(\alpha_2^{(l-1)})$ satisfies
\begin{equation}
\begin{split}
    & h_l(\alpha_2^{(l-1)})>\frac{\beta}{\alpha_2^{(l-1)}},~~~~~~~~~~~if~~\alpha_2^{(l-1)}<\beta;\\
    & 0<h_l(\alpha_2^{(l-1)})<\frac{\beta}{\alpha_2^{(l-1)}},~~~~~if~~\alpha_2^{(l-1)}>\beta,
\end{split}
\label{equ:over_norm_def}
\end{equation}
we have 
\begin{equation}
    \left(\Pi_{i=l}^1\phi\left(\mathbf{J_iJ_i}^T\right)-\frac{\beta}{\alpha_2^{(0)}}\right)\left(\Pi_{i=l-1}^1\phi\left(\mathbf{J_iJ_i}^T\right)-\frac{\beta}{\alpha_2^{(0)}}\right)<0.
\end{equation}
$\Pi_{i}^1\phi(\mathbf{J_iJ_i}^T)$ swings around the fixed point of $\frac{\beta}{\alpha_2^{(0)}}$ but there is no guarantee for convergence, so we name its as ``over normalized".

\begin{figure}[ht]
\vspace{-10pt}
\centering
\includegraphics[width=0.48\textwidth]{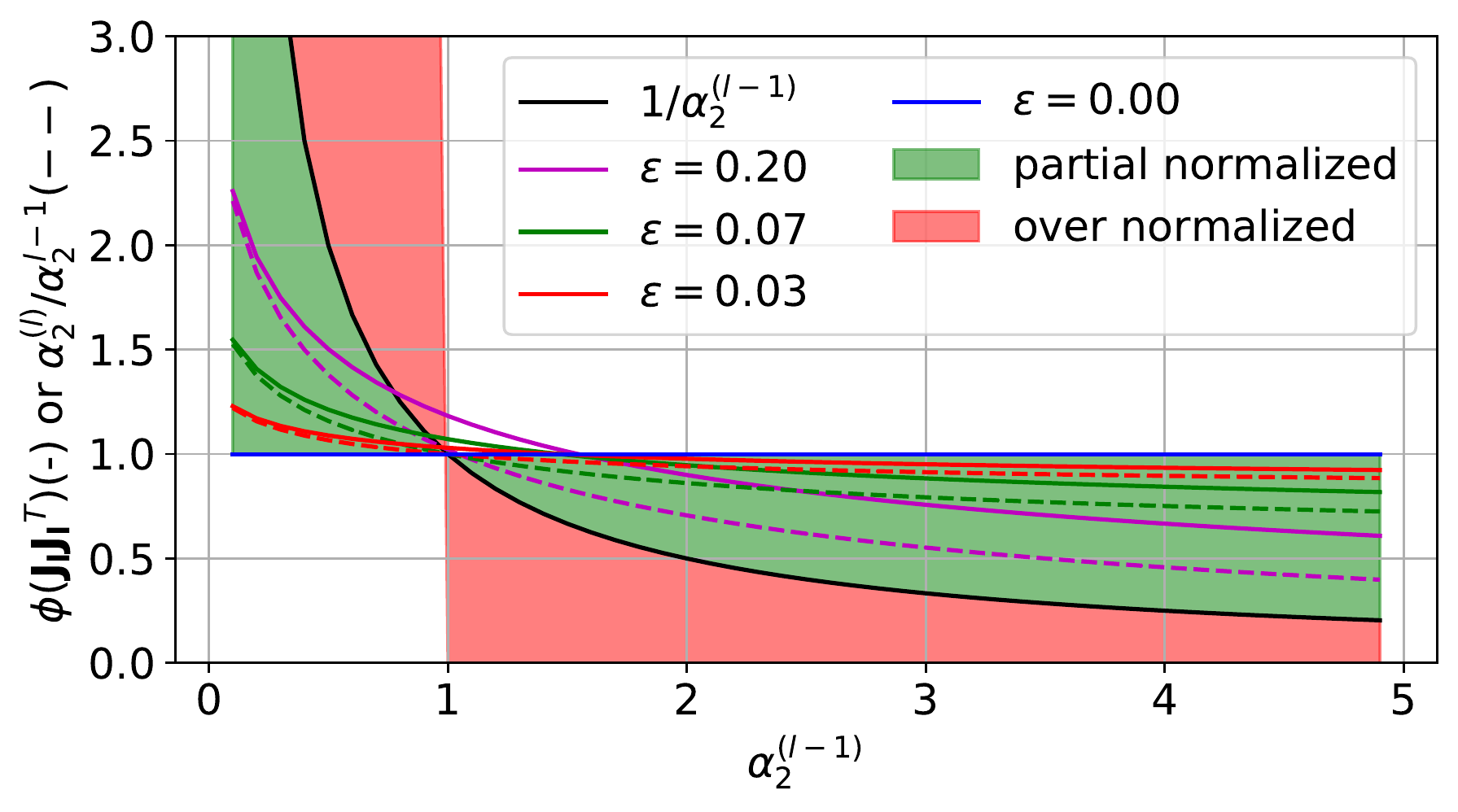}
\caption{$\phi(\mathbf{J_lJ_l}^T)$ (the solid line) and $\alpha_2^{(l)}/\alpha_2^{(l-1)}$ (the dashed line) of SeLU under different $\epsilon$. We have $\gamma_0=1$ for all the configurations.}
\label{fig:num_selu}
\vspace{-5pt}
\end{figure}

For SeLU, we have $\alpha_2^{(0)}\!=\!\beta\!=\!1$, and we plot $\phi(\mathbf{J_lJ_l}^T)\sim \alpha_2^{(l-1)}$ and $\alpha_2^{(l)}/\alpha_2^{(l-1)}\sim \alpha_2^{(l-1)}$ of different configurations in Fig. \ref{fig:num_selu}. It shows that 1) when $\epsilon$ is relatively small, we have $\phi(\mathbf{J_lJ_l}^T)\approx\alpha_2^{(l)}/\alpha_2^{(l-1)}$, and SeLU can be seen as a general linear transform; 2) when $\epsilon>0$, $\phi(\mathbf{J_lJ_l}^T)$ is in the ``partial normalized" region, which suggests that it will take a few layers to converge to a fixed point. Moreover, the $\phi(\mathbf{J_lJ_l}^T)$ of the configurations with greater $\epsilon$ is closer to $\frac{1}{\alpha_2^{(l-1)}}$, leading to to faster convergence; whereas a too large $\epsilon$ will result in gradient explosion, because of $\Pi_{i=L}^1\phi(\mathbf{J_iJ_i}^T)=(1+\epsilon)^L$. For a neural network with finite depth, we have
\begin{equation}
    (1+\epsilon)^L = 1 + L\epsilon + \sum_{i=2}^LC_L^i\epsilon^i.
\end{equation}
As a result, taking $\epsilon < \frac{1}{L}$ can effectively constrain the gradient norm while maintaining good normalization efficiency.

%% file: chapters/shallow_network.tex
Let's consider a neural network with sequential blocks: 
\begin{equation}
    \mathbf{f}(\mathbf{x_0}) = \mathbf{f_{L,\theta_L}}\circ \mathbf{f_{L-1,\theta_{L-1}}}\circ...\circ\mathbf{f_{1, \theta_1}}\left(\mathbf{x_0}\right),
\end{equation}
and the Jacobian matrix of the $i^{th}$ block is $\mathbf{J_i}$. We assume that $\mathbf{J} = \Pi_{i=L}^1\mathbf{J_i}$ is at least the $1^{st}$ moment unitarily invariant (Definition \ref{def:moment_unitary_invariant}). With Theorem \ref{theorem:multiplication}, we have $\phi(\mathbf{JJ}^T) = \Pi_i\phi(\mathbf{J_iJ_i}^T)$. In order to prevent the gradient explosion or vanishing, we expect $\forall i, \phi(\mathbf{J_iJ_i}^T) \approx 1$, which can be achieved with all the techniques discussed above. However, it might be influenced by many factors including the update of parameters under a large learning rate, invalid assumptions or systematic bias (like the cutting-off effect of padding), thus the actual $\phi(\mathbf{J_iJ_i}^T)$ can be represented as $1 + \gamma_i$ and $\phi(\mathbf{JJ}^T)$ can be $\Pi_{i=1}^L(1 + \gamma_i)$. Even if $\gamma_i$ for single layer is small enough, when $L$ is large, the influence of single $\gamma_i$ might accumulate and result in gradient explosion or vanishing. As a result, techniques like initialization, weight standardization and SeLU are less stable under large learning rates in deep networks. Fortunately, this can be addressed by the following proposition.
\begin{prop}
\textbf{(Shallow Network Trick).} Assume that for each of the $L$ sequential blocks in a neural network, we have $\phi(\mathbf{J_iJ_i}^T)=\omega+\tau\phi(\widetilde{\mathbf{J_i}}\widetilde{\mathbf{J_i}}^T)$ where $\mathbf{J_i}$ is its Jacobian matrix. Given $\lambda\in \mathbb{N}^+<L$, if $C_L^{\lambda}(1-\omega)^{\lambda}$ and $C_L^{\lambda}\tau^{\lambda}$ are small enough, the network would be as stable as a $\lambda$-layer network when both networks have $\forall~i,~\phi(\mathbf{J_iJ_i}^T)\approx1$.
\label{prop:shallow_network_trick}
\end{prop}
\begin{proof}
Because of $\phi(\mathbf{J_iJ_i}^T)=\omega + \tau\phi(\mathbf{J_iJ_i}^T)$, the optimal $\phi(\mathbf{J_iJ_i})$ is $\frac{1-\omega}{\tau}$.
We consider both the absolute and relative errors of $\phi(\mathbf{J_iJ_i})$ by representing it as $\frac{1-\omega}{\tau}(1+\gamma_i)$ and $\frac{1-\omega}{\tau} + \delta_i$, respectively. For both kinds of error, we have
\begin{equation}
\begin{split}
    & \phi\left(\mathbf{JJ}^T\right)\! =\! \Pi_{i=1}^L(1 + (1-\omega)\gamma_i)\!=\!1 + \sum_{i=1}^LC_L^i(1-\omega)^i\Pi_{j}\gamma_j,\\
    & \phi\left(\mathbf{JJ}^T\right) = \Pi_{i=1}^L(1 + \tau\delta_i)= 1 + \sum_{i=1}^LC_L^i\tau^i\Pi_{j}\delta_j.
\end{split}
\label{equ:shallow_network_trick}
\end{equation}
When $\omega\rightarrow1^-$, we have $\tau\rightarrow 0^+$, $\lim_{i\rightarrow\infty}(1-\omega)^i=\lim_{i\rightarrow\infty}\tau^i=0$, and the accumulation of error would diminish as $i$ is large.

Here we borrow the concept of effective depth proposed in Philipp et al. (2018)\cite{philipp2018gradients}: assuming that $C_L^i(1-\omega)^i\Pi_{j}\gamma_j$ and $C_L^i\tau^i\Pi_{j}\delta_j$ are neglectable when $i>\lambda, \lambda<L$, all the errors are only influential within $\lambda$ layers, thus it would be as stable as a $\lambda$-layer shallow network.
\end{proof}

%% file: chapters/series_network/densenet.tex
We denote the activations as $\mathbf{x_i}\in \mathbb{R}^{c_is_{fm}\times 1}$ where $c_i$ is the number of channels and $s_{fm}$ is the size of feature maps, and denote $\delta_i = c_i - c_{i-1}$. In DenseNet\cite{huang2017densely}, the output of each layer within a dense block is concatenated with the input on the channel dimension to create dense shortcut connections, which is illustrated as follows:
\begin{equation}
    \mathbf{x}_i = \left[\mathbf{x}_{i-1}, \mathbf{H}_i\left(\mathbf{x}_{i-1}\right)\right],~~\frac{\partial \mathbf{x}_i}{\partial \mathbf{x}_{i-1}}={\left[\begin{array}{c}
     \mathbf{I} \\
     \mathbf{H}_i
\end{array}\right]} :=\mathbf{J_i},
\end{equation}
where $\mathbf{H}_i \in \mathbb{R}^{\delta_is_{fm}\times c_{i-1}s_{fm}}$, $\mathbf{I} \in \mathbb{R}^{c_{i-1}s_{fm}\times c_{i-1}s_{fm}}$. Since
\begin{equation}
    \mathbf{J_i}^T\mathbf{J_i} = \left[\mathbf{I}~~\mathbf{H}_i^T\right]{\left[\begin{array}{c}
     \mathbf{I} \\
     \mathbf{H}_i
\end{array}\right]} = \left[\mathbf{I} + \mathbf{H}_i^T\mathbf{H}_i \right]
\label{equ:densenet_JTJ}
\end{equation}
and $\mathbf{H}_l$ is composed of the parts defined in Table \ref{tab:parts_library}, with Proposition \ref{prop:properties_of_eom_cm}, the non-diagonal entries of $\mathbf{I} + \left(\mathbf{H}_l\right)^T\mathbf{H}_l$ have a zero expectation while the diagonal entries share an identical expectation, and $\mathbf{J_l}$ satisfies Proposition \ref{prop:1st_moment_unitary_invariant_exp_diag}. Therefore, $\phi(\mathbf{J_iJ_i}^T)$ can be calculated by
\begin{equation}
    \phi\left(\mathbf{J_iJ_i}^T\right) = \frac{c_{i-1}}{c_i} + \frac{\delta_i}{c_{i}}\phi\left(\mathbf{H_iH_i}^T\right).
\label{equ:phi_of_densenet}
\end{equation}

As a result, in order to achieve block dynamical isometry, we expect $\phi(\mathbf{H_iH_i}^T)\approx 1$, which can be achieved with the methods discussion in previous subsections. We will evaluate some configurations in Section \ref{Exp:densenet}. Equation \eqref{equ:phi_of_densenet} also reveals that DenseNet is an instance of the shallow network trick (Proposition \ref{prop:shallow_network_trick}), thus it is more stable compared with vanilla serial neural networks under the same depth.

%% file: chapters/series_parallel_hybrid_networks.tex
We define the serial-parallel hybrid network as a network consisting of a sequence of blocks connected in serial, while each block may be composed of several parallel branches. Famous serial-parallel hybrid networks include Inception\cite{szegedy2015going}, ResNet\cite{he2016deep}, and NASNet\cite{zoph2018learning}. As proposed in Proposition \ref{prop:prerequisite_series_networks}, as long as the input-output Jacobian matrices of all the blocks satisfy Definition \ref{def:expectattion_diagonal_matrix}, the network is at least the $1^{st}$ moment unitary invariant.

\begin{prop}
Let $\{\mathbf{J_i}\}$ denote a group of independent input-output Jacobian matrices of the parallel branches of a block. $\sum_i\mathbf{J_i}$ is an expectant orthogonal matrix, if it satisfies: 1) $\forall i$, $\mathbf{J_i}$ is an expectant orthogonal matrix; 2) at most one matrix in $\{\mathbf{J_i}\}$ is not central matrix. (Proof: Appendix \ref{proof:prerequisite_series_parallel_hybrid_networks})
\label{prop:prerequisite_series_parallel_hybrid_networks}
\end{prop}

According to Proposition \ref{prop:properties_of_eom_cm}, as long as each branch is composed of the parts in Table \ref{tab:parts_library} and at most one branch does not contain a zero-mean linear transform, with Proposition \ref{prop:prerequisite_series_parallel_hybrid_networks}, the series-parallel hybrid network is at least the $1^{st}$ moment unitarily invariant.

ResNet \cite{he2016deep} is one of the most popular network structures that can avoid gradient explosion and vanishing, it is also the simplest serial-parallel hybrid network. The Jacobian matrix of each residual block is $\mathbf{J_i} = \mathbf{I} + \widetilde{\mathbf{J_i}}$, with Equation \eqref{equ:add_expectation}, we have
\begin{equation}
    \phi\left(\mathbf{J_i^{(l)}J_i^{(l)}}^T\right) = 1 + \phi\left(\widetilde{\mathbf{J_i}^{(l)}}\widetilde{\mathbf{J_i}^{(l)}}^T\right).
\label{equ:shortcut}
\end{equation}
From the above equation, ResNet can be viewed as an extreme example of the shallow network trick (Proposition \ref{prop:shallow_network_trick}) wherein $(1-\omega)\rightarrow 0$. As a result, its extremely low effective depth provides higher stability. 

\textbf{Data Normalization in ResNet.} The $2^{nd}$ moment of the activations of ResNet with BN does not stay at a fixed point but keeps increasing through the layers \cite{zhang2019fixup}. Let's consider a ResNet whose $l^{th}$ block is represented as $\mathbf{x_{l+1}} = BN(\mathbf{f}(\mathbf{x_l})) + \mathbf{x_l}$. Since the $2^{nd}$ moment of $BN$'s output is 1, under the assumption that the outputs of the major branch and the shortcut branch are independent, we have $\alpha_2^{(l+1)} = 1 + \alpha_2^{(l)}$. At the down-sampling layers, since the shortcut connection is also handled by BN, $\alpha_2^{(l+1)}$ would be reset to $2$. We denote the Jacobian matrix of the major branch as $\widetilde{\mathbf{J}^{(l)}}$, with Proposition \ref{prop:evolution_2nd_norm_net }, it's easy to obtain $\phi(\widetilde{\mathbf{J}^{(l)}}\widetilde{\mathbf{J}^{(l)}}^T) = \frac{1}{\alpha_2^{(l-1)}}$, and then we have
\begin{equation}
\begin{split}
    &\phi\left(\mathbf{J_i^{(l+1)}J_i^{(l+1)}}^T\right)=1 + \frac{1}{\alpha_2^{l}} = \frac{\alpha_2^{(l+1)}}{\alpha_2^{(l)}},\\
    &\Pi_{i=L}^l\phi\left(\mathbf{J_i^{(i)}J_i^{(i)}}^T\right) = \frac{\alpha_2^{(L)}}{\alpha_2^{(l-1)}}.
\end{split}
\label{equ:scale_relu_phi}
\end{equation}
As the $2^{nd}$ moment of the activations in ResNet linearly rather than exponentially increases, and such an increasing is periodically stopped by down-sampling. Thus with Equation \eqref{equ:scale_relu_phi}, gradient explosion or vanishing will not happen in ResNet when the depth is finite.

\textbf{Fixup Initialization \cite{zhang2019fixup}.} Without loss of generality, we consider a ResNet consisting of $L$ residual blocks, wherein each block has $m$ convolutional layers activated by ReLU. The feature maps are down-sampled for $d$ times throughout the network. We assume that the convolutional layers are properly initialized such that $\phi(\mathbf{J_i^{(c)}J_i^{(c)}}^T)=\alpha$, and the convolutional layers in the down-sampling shortcuts are initialized to have $\phi(\mathbf{J_i^{(c)}J_i^{(c)}}^T)=\alpha_d$. For a single block whose number of input channels equals to number of output channels, we have $\phi(\mathbf{J_iJ_i}^T) = 1+(\frac{\alpha}{2})^m$; for the down-sampling block, we have $\phi(\mathbf{J_iJ_i}^T) = \alpha_d + (\frac{\alpha}{2})^m$. When $L$ is finite, we have the following proposition:
\begin{prop}
\textbf{(``Plus One" Trick).} Assume that for each of the $L$ sequential blocks of a series-parallel hybrid neural network, we have $\phi(\mathbf{J_iJ_i}^T)=1+\phi(\widetilde{\mathbf{J_i}}\widetilde{\mathbf{J_i}}^T)$ where $\mathbf{J_i}$ is its Jacobian matrix. The network has gradient norm equality as long as
\begin{equation}
    \phi\left(\widetilde{\mathbf{J_i}}\widetilde{\mathbf{J_i}}^T\right) = O(\frac{1}{L^p}), p>1.
\label{equ:plus_1_trick}
\end{equation}
(Proof: Appendix \ref{proof:plus_1_trick})
\label{prop:plus_1_trick}
\end{prop}

As a result, it is optimal to have $\alpha_d=1,\alpha=2L^{-\frac{p}{m}}, p>1$. For Gaussian weights, we can initialize the weights with $N(0, L^{-p/m}\frac{2}{n})$. As KM initializes the weights to $N(0, \frac{2}{n})$, the Fixup initialization is just equivalent to scaling the weights initialized with KM by $L^{-p/2m}$; for orthogonal weights, we have $\beta=L^{-p/2m}\sqrt{2}$. For the down-sampling convolutions, it should be initialized to have Gaussian weights with $N(0, \frac{1}{n})$ or orthogonal weights with $\beta=1$.

Zhang et al. (2019) \cite{zhang2019fixup} observe that although ResNet with the Fixup initialization can achieve gradient norm equality, it does not regularize the training like BN does. To solve this problem, additional scalar multiplier and bias are added before each convolution, linear, and element-wise activation layer. The multipliers and biases are trainable under a learning rate of $1/10$ to improve stability. Moreover, further regularization like mixup \cite{zhang2017mixup} is used. Although we reach the same conclusion claimed in \cite{zhang2019fixup}, our derivation is much simpler owing to the highly modularized framework.

%% file: tables/network_structures.tex
\begin{table*}[t]
\caption{Network structures.}
\centering
\resizebox{0.8\textwidth}{!}{
\begin{tabular}{c|c|c|c|c}
\hline
  &Out Size &Serial Network & ResNet & DenseNet
\\ \hline \hline
conv1&$32\times 32$&\multicolumn{2}{|c|}{$3\times 3,16,s~1$}&$3\times 3,24,s~1$\\
\hline
block1&$32\times 32$&${\left[\begin{array}{c} 
    3\times 3, 16, s~1 \\ 
    3\times 3, 16, s~1 \\
\end{array}\right]}\times 5$&${\left(\begin{array}{c} 
    3\times 3, 16, s~1 \\ 
    3\times 3, 16, s~1 \\
\end{array}\right)}\times 9$&${\left\{\begin{array}{c} 
    1\times 1, 48, s~1 \\ 
    3\times 3, 12, s~1 \\
\end{array}\right\}}\times 8$\\
\hline
ds1&$16\times16$&${\left[\begin{array}{c} 
    3\times 3, 32, s~2 \\ 
    3\times 3, 32, s~1 \\
\end{array}\right]}\times 1$&${\left(\begin{array}{c} 
    3\times 3, 32, s~2 \\ 
    3\times 3, 32, s~1 \\
\end{array}\right)}\times 1$&${\left.\begin{array}{c} 
    1\times 1, 60, s~2 \\ 
\end{array}\right.}$\\
\hline
block2&$16\times 16$&${\left[\begin{array}{c} 
    3\times 3, 32, s~1 \\ 
    3\times 3, 32, s~1 \\
\end{array}\right]}\times 4$&${\left(\begin{array}{c} 
    3\times 3, 32, s~1 \\ 
    3\times 3, 32, s~1 \\
\end{array}\right)}\times 8$& ${\left\{\begin{array}{c} 
    1\times 1, 48, s~1 \\ 
    3\times 3, 12, s~1 \\
\end{array}\right\}}\times 8$\\
\hline
ds2&$8\times8$&${\left[\begin{array}{c} 
    3\times 3, 64, s~2 \\ 
    3\times 3, 64, s~1 \\
\end{array}\right]}\times 1$&${\left(\begin{array}{c} 
    3\times 3, 64, s~2 \\ 
    3\times 3, 64, s~1 \\
\end{array}\right)}\times 1$&${\left.\begin{array}{c} 
    1\times 1, 78, s~2 \\ 
\end{array}\right.}$\\
\hline
block3&$8\times 8$&${\left[\begin{array}{c} 
    3\times 3, 64, s~1 \\ 
    3\times 3, 64, s~1 \\
\end{array}\right]}\times 4$&${\left(\begin{array}{c} 
    3\times 3, 64, s~1 \\ 
    3\times 3, 64, s~1 \\
\end{array}\right)}\times 8$& ${\left\{\begin{array}{c} 
    1\times 1, 48, s~1 \\ 
    3\times 3, 12, s~1 \\
\end{array}\right\}}\times 8$\\
\hline
&$1\times 1$&\multicolumn{3}{|c}{average pooling, 10-d fc, softmax}\\
\hline
\end{tabular}}
\label{tab:network_models}
\vspace{-15pt}
\end{table*}

%% file: chapters/Exp_numeric_exp.tex
\input{chapters/num_exps/num_4_1.tex}

%% file: chapters/num_exps/num_4_1.tex
For Theorem \ref{theorem:multiplication}, we consider a network formulated as $\mathbf{x_{out}}=\mathbf{f_L}\circ \mathbf{f_{L-1}}\circ...\mathbf{f_1}(\mathbf{x_{in}})$. Each $\mathbf{f_i}$ consists of an $m_i\times m_{i-1}$ weight matrix whose entries follow i.i.d. $N(0, \sigma_i^2)$ and a ReLU activation function. The entries of $\mathbf{x_{in}}\in \mathbb{R}^{m_0\times 1}$ follow i.i.d. $N(\mu, \sigma^2)$. According to Proposition \ref{prop:gaussian_activ_2nd_invariant}, such a network certainly satisfies the prerequisites. 

The network can be determined by a joint state $\left[\{m_i\}, \{\sigma_i\}, \mu, \sigma, L\right]$. Let $U(a, b)$ denote the uniform distribution within $[a,b]$, and $U_{i}(a, b)$ represent the discrete uniform distribution on integers from $a$ to $b$. To cover different configurations, we repeat the experiment for 100 times and the joint state is uniformly drawn from $[\{U_{i}(1000, 5000)\},\! \{U(0.1, 5)\},\! U(\!-5,\! 5),\! U(0.1, 5), U_{i}(2, 20)]$.

For Theorem \ref{theorem:Addition}, we consider a network formulated as $\mathbf{x_{out}}=\sum_{i=1}^n\mathbf{f_i}(\mathbf{x_{in}})$. Each $\mathbf{f_i}$ consists of an $m\times m$ weight matrix whose entries follow i.i.d. $N(0, \sigma_i^2)$ and a ReLU activation function. The entries of $\mathbf{x_{in}}\in \mathbb{R}^{m\times 1}$ follow i.i.d. $N(\mu, \sigma^2)$. As Central i.i.d. Gaussian matrices are asymptotically R-diagonal (Equation 4.45 in Cakmak (2012) \cite{cakmak2012non}) and with Theorem 32 in Cakmak (2012) \cite{cakmak2012non}, all the blocks of the given network are R-diagonal. As the network is determined by a joint state $\left[ m, \{\sigma_i\}, \mu, \sigma, n\right]$, we also repeat the experiment for 100 times and at each time the joint state is uniformly drawn from $\left[U_{i}(1000, 5000), \{U(0.1, 5)\}, U(-5, 5), U(0.1, 5), U_{i}(2, 20)\right]$.

We denote the input-output Jacobian matrix of the whole network as $\mathbf{J}$, and we evaluate our theorems by measuring how well $\left(\phi(\mathbf{JJ}^T)/\phi(\mathbf{JJ}^T)_{t}, \varphi(\mathbf{JJ}^T)/\varphi(\mathbf{JJ}^T)_t\right)$ concentrates around $(1, 1)$, where $\phi(\mathbf{JJ}^T), \varphi(\mathbf{JJ}^T)$ are directly calculated from the defined Jacobian matrices while $\phi(\mathbf{JJ}^T)_t, \varphi(\mathbf{JJ}^T)_t$ are theoretical values derived by Theorem \ref{theorem:multiplication} and \ref{theorem:Addition}.

\begin{figure}[!htbp]
\vspace{-10pt}
\centering
\includegraphics[width=0.48\textwidth]{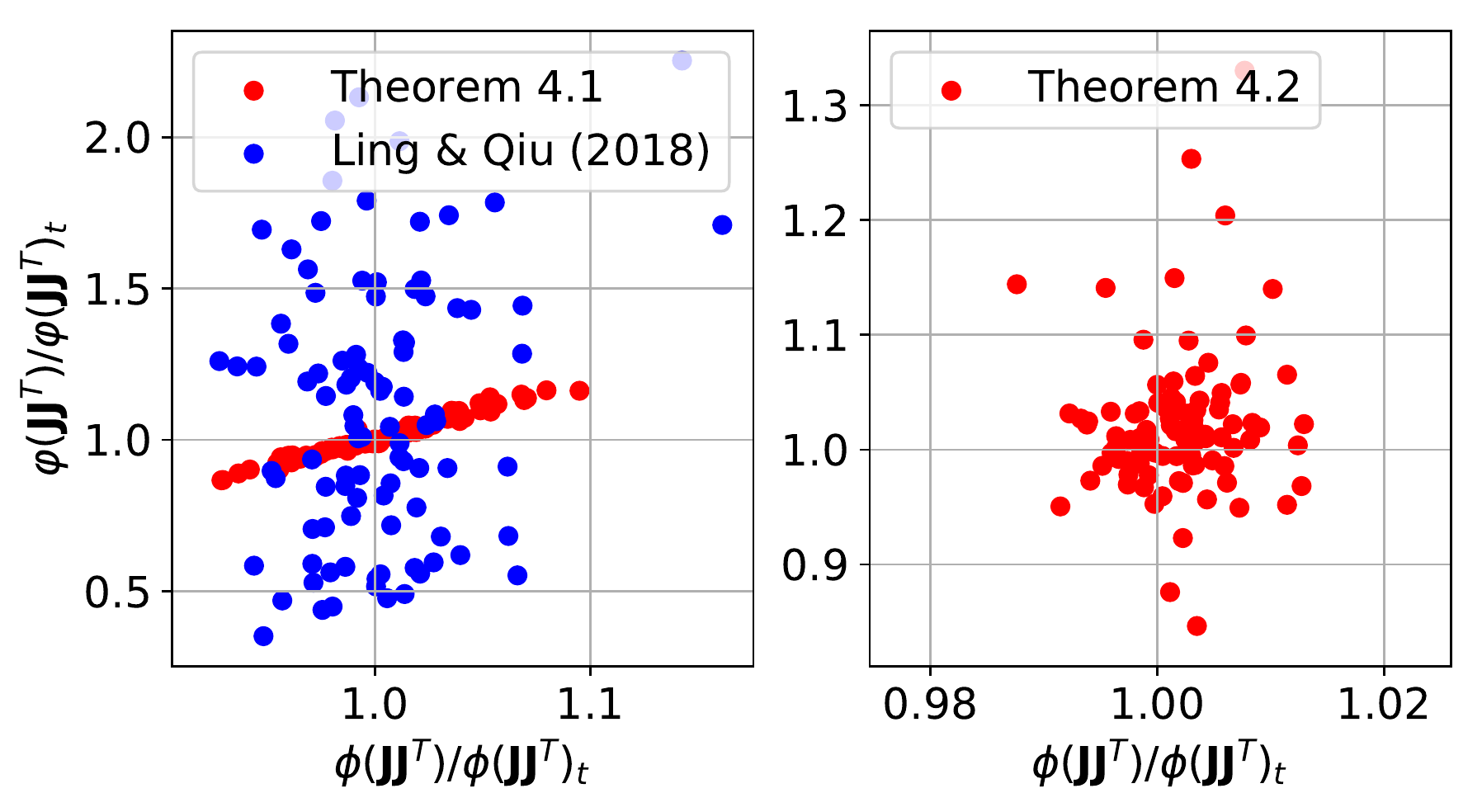}
\caption{Verification of Theorem \ref{theorem:multiplication} and \ref{theorem:Addition}. Each point denotes the result of one experiment.}
\label{fig:num4_1}
\vspace{-5pt}
\end{figure}

The results are shown in Fig. \ref{fig:num4_1}. We can see that despite of the numerical error, the experiment results well concentrate around $(1, 1)$. Besides, while the Result 2 in Ling \& Qiu (2018) \cite{ling2018spectrum} is quite similar to our Theorem \ref{theorem:multiplication}, their result can only handle the situations when the input and output feature map sizes are equal. Note that the estimation error of $\varphi(\mathbf{JJ}^T)$ with the theory in \cite{ling2018spectrum} is much greater than ours.

%% file: chapters/Exps/Exp_setups.tex
We first validate the conclusions yielded by our theorems on CIFAR-10 classification.
The basic models we use are shown in Table \ref{tab:network_models}, where ``[]" denotes a vanilla network block, ``()" denotes a network block with shortcut connection, and ``\{\}" denotes a dense block whose major branch's output is concatenated with its input in the channel dimension. The shortcut connections in down-sampling layers are handled by average pooling and zero padding following Zhang et al. (2019)\cite{zhang2019fixup}. All the models are trained with a batch size of 128. We use SGD as the optimizer with momentum=0.9 and weight decay=0.0005. Besides, we clip the gradient within $[-2, 2]$ for all the experiments to increases the stability. 

For all the experiments of serial networks except for DenseNet, the ``serial network" in Table \ref{tab:network_models} is applied, which is equivalent to a ResNet-32 without shortcut connections. The models are trained for 130 epochs. The initial learning rate is set to 0.01, and decayed to 0.001 at epoch 80. For experiments on DenseNet, the models are trained for 130 epochs. The initial learning rate is set to 0.1, and decayed by $10\times$ at epoch 50, 80. For experiments on ResNet, we follow the configuration in Zhang et al. (2019)\cite{zhang2019fixup}, i.e. all the models are trained for 200 epochs with initial learning rate of 0.1 that is decayed by 10 at epoch 100, 150. 

To support our conclusions,  we evaluate all the configurations from two perspectives: module performance (test accuracy) and gradient norm distribution. Each configuration is trained from scratch 4 times to reduce the random variation and the test accuracy is averaged among the last 10 epochs. The gradient norm of each weight is represented by the $L_2$ norm of the weights' gradient, $||\mathbf{\Delta \theta_i}||^2_2/\eta^2$, which is collected from the first 3 epochs (1173 iterations). For clarity, we color the range from 15 percentile to 85 percentile, and represent the median value with a solid line.

%% file: chapters/Exps/Exp_orthogonal_init.tex
To support our conclusions in Section \ref{sec:initialization_serial},  we evaluate the initialization techniques in a 32-layer serial network on CIFAR-10. The test accuracy of all configurations is summarized in Table \ref{tab:Accuracy_eval_init_serial}, and the gradient distribution is illustrated in Fig. \ref{fig:init_grad}. We evaluates two kinds of orthogonal initialization strategies: the orthogonal initialization \footnote{\href{https://pytorch.org/docs/stable/nn.init.html\#torch.nn.init.orthogonal\_}{pytorch.org/docs/stable/nn.init.html\#torch.nn.init.orthogonal\_}} in Saxe et al. (2013) \cite{saxe2013exact} and the delta orthogonal initialization \footnote{We use the implementation for orthogonal initialization provided in \href{https://github.com/JiJingYu/delta\_orthogonal\_init\_pytorch}{https://github.com/JiJingYu/delta\_orthogonal\_init\_pytorch}} in Xiao et al. (2018) \cite{xiao2018dynamical}.

\begin{figure*}[ht!]
\centering
\includegraphics[width=0.98\textwidth]{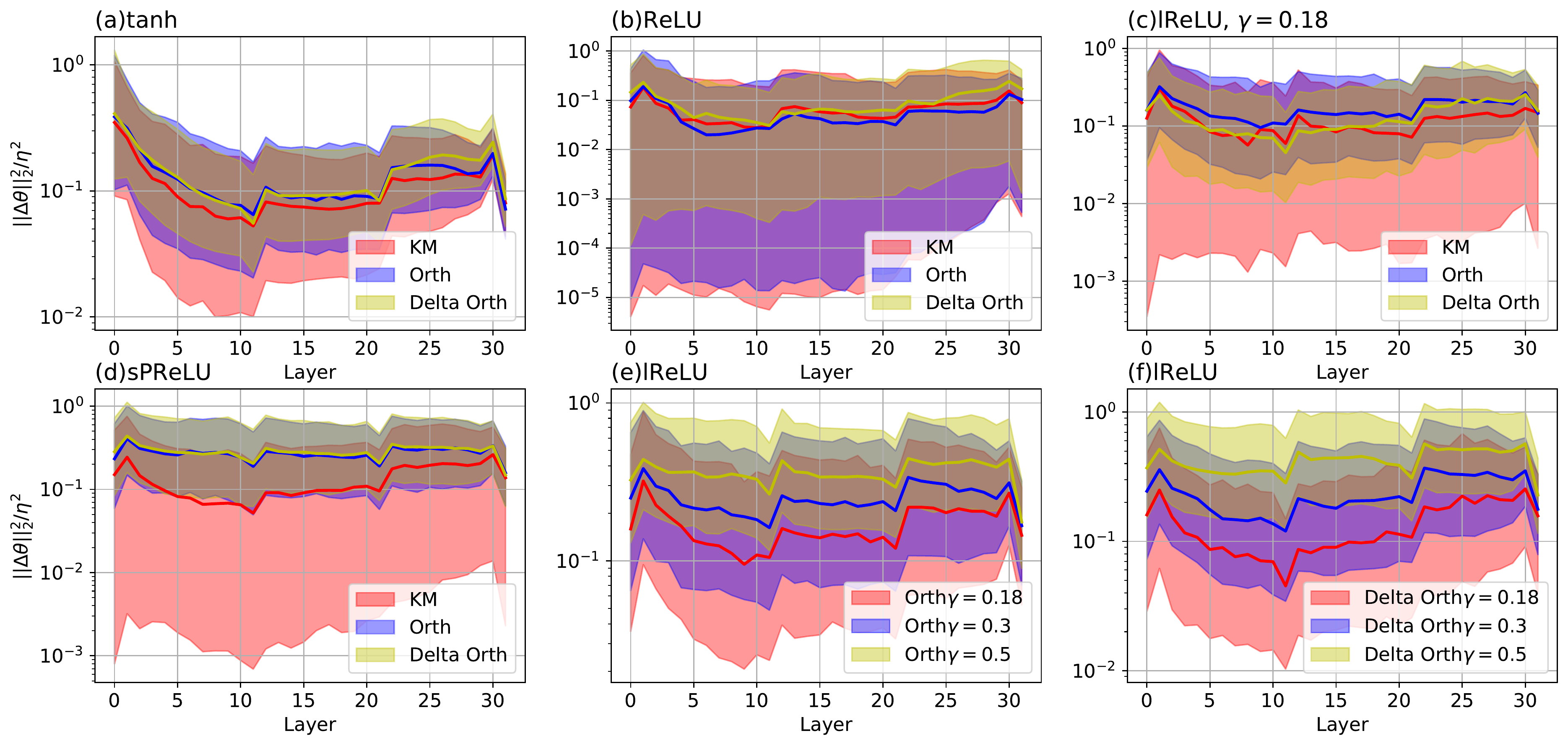}
\caption{Gradient norm distribution throughout the network under different configurations. The colored regions represent the range from 15 percentile to 85 percentile, while the solid line is the median. ``lReLU" denotes ``leaky ReLU".}
\label{fig:init_grad}
\vspace{-15pt}
\end{figure*}

\begin{table}[ht]
\caption{Test accuracy of initialization techniques on CIFAR-10 with different activation functions and configurations (Cl=$95\%$).}
\centering
\resizebox{0.45\textwidth}{!}{
\begin{tabular}{c|c|c}
\hline
{\bf Activation Function}  &{\bf Approach}	&{\bf Test Acc.}
\\ \hline \hline
\multirow{3}{*}{\textbf{tanh}}&BN&$\mathbf{85.77\%\pm 0.77\%}$\\
\cline{2-3}
&KM& $83.33\%\pm1.02\%$\\
\cline{2-3}
&Orth&$83.13\%\pm 0.54\%$\\
\cline{2-3}
&Delta Orth\cite{xiao2018dynamical}&$83.31\%\pm 0.38\%$\\
\hline \hline
\multirow{3}{*}{\textbf{ReLU}}&BN&$\mathbf{88.70\%\pm0.31\%}$\\
\cline{2-3}
&KM& $85.13\%\pm1.35\%$\\
\cline{2-3}
&Orth&$85.53\%\pm0.64\%$\\
\cline{2-3}
&Delta Orth&$86.10\%\pm1.33\%$\\
\hline \hline
\multirow{3}{*}{\textbf{lReLU},$\gamma=0.18$}&BN&$\mathbf{89.19\%\pm0.41\%}$\\
\cline{2-3}
&KM&$87.96\%\pm1.09\%$\\
\cline{2-3}
&Orth&$88.51\%\pm0.37\%$\\
\cline{2-3}
&Delta Orth&$87.97\%\pm1.34\%$\\
\hline
\multirow{2}{*}{\textbf{lReLU},$\gamma=0.3$}&BN&$89.58\%\pm0.51\%$\\
\cline{2-3}
&Orth&$89.24\%\pm0.44\%$\\
\cline{2-3}
&Delta Orth&$\mathbf{90.12\%\pm0.64\%}$\\
\hline
\multirow{2}{*}{\textbf{lReLU},$\gamma=0.5$}&BN&$88.60\%\pm0.34\%$\\
\cline{2-3}
&Orth&$88.91\%\pm0.27\%$\\
\cline{2-3}
&Delta Orth&$\mathbf{89.53\%\pm0.32\%}$\\
\hline \hline
\multirow{3}{*}{\textbf{PReLU}\cite{he2015delving}}&BN&$\mathbf{88.96\%\pm0.35\%}$\\
\cline{2-3}
&KM&$88.11\%\pm0.99\%$\\
\cline{2-3}
&Orth&$87.39\%\pm3.06\%$\\
\cline{2-3}
&Delta Orth&$82.00\%\pm7.39\%$\\
\hline \hline
\multirow{3}{*}{\textbf{sPReLU} (ours)}&BN&$88.96\%\pm0.35\%$\\
\cline{2-3}
&KM&$88.87\%\pm0.32\%$\\
\cline{2-3}
&Orth&$89.16\%\pm0.32\%$\\
\cline{2-3}
&Delta Orth&$\mathbf{89.73\%\pm0.34\%}$\\
\hline
\end{tabular}
}
\label{tab:Accuracy_eval_init_serial}
\vspace{-20pt}
\end{table}

To begin with, as illustrated in Fig. \ref{fig:init_grad}, the gradient distributions of all configurations with tanh, ReLU, leaky ReLU and sPReLU are more or less neutral, and Table \ref{tab:Accuracy_eval_init_serial} shows that all these configurations can converge, which demonstrates the effectiveness of the initialization schemes under relatively deep network and moderate learning rate. Second, the gradient norm distribution of tanh is more concentrated and neutral compared with rectifiers, whereas its test accuracy is much lower. Both these phenomenon accord with our predictions in Section \ref{sec:initialization_serial}: tanh is more stable compared with rectifier neurons, whereas rectifiers are more effective. Besides, the gradient explosion occasionally happens with PReLU. Moreover, with $\gamma=0.18$, leaky ReLU outperforms ReLU by $+2.83\%$, $+2.98\%$ and $+1.87\%$ on Gaussian, orthogonal and delta orthogonal weights, respectively, which can be partially attributed to the additional stability provided by leaky ReLU. The reason is that the gradient norm is more concentrated with leaky ReLU, as illustrated in Fig. \ref{fig:init_grad}(b)-(c). Fig. \ref{fig:init_grad}(e)-(f) compare the gradient norm of leaky ReLU with different negative slope coefficient $\gamma$, and models with a larger $\gamma$ have flatter distribution, whereas a too high $\gamma$ even results in weak nonlinearity. This trade-off between stability and nonlinearity is also illustrated in Table \ref{tab:Accuracy_eval_init_serial}: while the test accuracy of $\gamma=0.18$ is $+0.59\%$ higher than that of $\gamma=0.5$ when the network is stabilized with BN, the latter one is $+0.4\%$ or $+1.56\%$ higher with orthogonal or delta orthogonal weights, respectively. The highest test accuracy is achieved when $\gamma=0.3$. With delta orthogonal weights, it is even $+0.54\%$ higher than the BN baseline. For sPReLU, Table \ref{tab:Accuracy_eval_init_serial} shows that compared with PReLU, sPReLU achieves $+0.76\%$ accuracy gain on Gaussian weights, $+1.77\%$ on orthogonal initialization and $+7.73\%$ on delta orthogonal weights with a much narrower confidence interval. Besides, sPReLU achieves comparable results with leaky ReLU under $\gamma=0.3$ without the need of hand-crafted hyper-parameter. Last but not least, for all the activation functions except tanh and PReLU, orthogonal and delta orthogonal weights achieve better results compared with Gaussian weights in Table \ref{tab:Accuracy_eval_init_serial} and demonstrate more concentrated gradient norm in Fig. \ref{fig:init_grad}. For tanh, one possible explanation is that tanh diminishes the flow of information in the forward pass, and the noise introduced by the Gaussian distribution might partially alleviate this problem. All in all, our discussions in Section \ref{sec:initialization_serial} predict most of phenomenons in our experiences, which demonstrates the effectiveness of our theorem.

\begin{figure*}[hbt]
  \centering
    \includegraphics[width=0.98\textwidth]{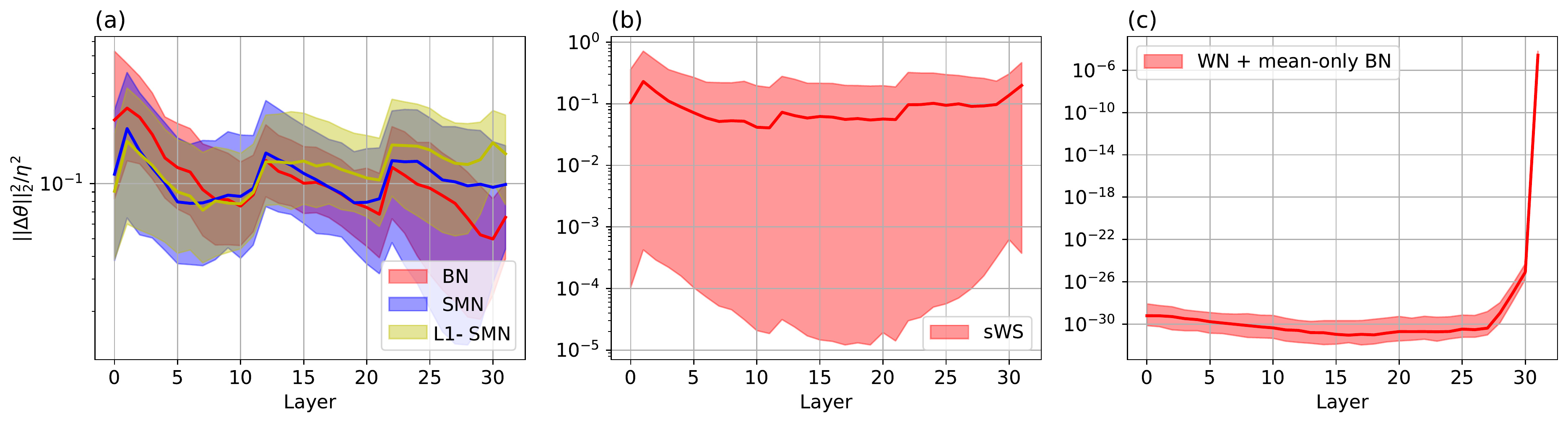}
  \caption{Gradient norm distribution throughout the network under different normalization techniques.}
  \label{fig:norm_grad}
 \vspace{-10pt}
\end{figure*}

%% file: chapters/Exps/Exp_weight_standard.tex
In this part, we evaluate the performance of different normalization techniques. The test accuracy of all configurations is summarized in Table \ref{tab:Accuracy_eval_norm_serial}, and the gradient distribution is shown in Fig. \ref{fig:norm_grad}.

\begin{table}[ht]
\vspace{-10pt}
\caption{Test accuracy of normalization techniques on CIFAR-10 in serial networks (Cl=$95\%$).}
\centering
\resizebox{0.48\textwidth}{!}{
\begin{tabular}{c|c}
\hline
{\bf Approach}	&{\bf Test Acc.}
\\ \hline \hline
Batch Normalization (BN)&$\mathbf{88.70\%\pm0.31\%}$\\
\hline
\tabincell{c}{Second Moment Normalization (SMN) (ours)}&$\mathbf{88.50\%\pm0.26\%}$\\
\hline
\tabincell{c}{L1-norm SMN (L1-SMN) (ours)} & $88.34\%\pm0.61\%$ \\
\hline \hline
Scaled Weight Standardization (sWS) (ours)&$\mathbf{88.06\%\pm0.49\%}$\\
\hline
Weight Norm + mean-only BN\cite{salimans2016weight}&$10\%$ (not converge)\\
\hline
\end{tabular}
}
\label{tab:Accuracy_eval_norm_serial}
\vspace{-5pt}
\end{table}

According to Table \ref{tab:Accuracy_eval_norm_serial}, our SMN and its L1-norm version achieve comparable test accuracy compared with BN, and its gradient norm distribution in the 32-layer serial network is also akin to that of BN, both of which demonstrate the effectiveness of our novel normalization technique. While the original weight normalization does not converge due to the improper initialization, our scaled weight standardization demonstrates a neutral gradient norm distribution, and its test accuracy is only $0.64\%$ lower than the BN baseline. The only difference between SMN and sWS is that SMN estimates the $2^{nd}$ moment from pre-activations while sWS estimates from weight kernels, and the former one's distribution is obviously narrower than the latter one. This evidences our earlier conclusion that the $2^{nd}$ moment should be obtained from pre-activations for stability.

%% file: chapters/Exps/Exp_SeLU.tex
In this part, we evaluate SeLU under difference setups of $\gamma_0$ and $\epsilon$ with Gaussian or orthogonal weights. The test accuracy of all configurations is summarized in Table \ref{tab:Accuracy_eval_selu_serial}, and the gradient distribution is illustrated in Fig. \ref{fig:selu_grad}.

\begin{table}[ht]
\vspace{-10pt}
\caption{Test accuracy of SeLU under different configurations (Cl=$95\%$). All the methods except for \cite{klambauer2017self} are ours.}
\centering
\resizebox{0.45\textwidth}{!}{
\begin{tabular}{c|c|c|c}
\hline
{\bf Weight Initialization}  &{\bf $\gamma_0$} &{\bf $\epsilon$}	&{\bf Test Acc.}
\\ \hline \hline
\multirow{5}{*}{\textbf{KM}}&\multirow{4}{*}{\textbf{1}}&\cite{klambauer2017self}&$89.00\%\pm 0.51\%$\\
\cline{3-4}
&& $0.00$ & $85.13\%\pm1.35\%$\\
\cline{3-4}
&& $0.03$ & $\mathbf{89.42\%\pm0.29\%}$\\
\cline{3-4}
& &$0.07$&$89.25\%\pm 0.58\%$\\
\cline{2-4}
&2&0.03&$89.42\%\pm 0.55\%$\\
\hline \hline
\multirow{5}{*}{\textbf{Orth}}&\multirow{4}{*}{\textbf{1}}&\cite{klambauer2017self}&$89.10\%\pm 0.33\%$\\
\cline{3-4}
&& $0.00$ & $85.53\%\pm0.64\%$\\
\cline{3-4}
&& $0.03$ & $\mathbf{89.49\%\pm0.32\%}$\\
\cline{3-4}
& &$0.07$&$89.10\%\pm 0.39\%$\\
\cline{2-4}
&2&0.03&$89.34\%\pm 0.39\%$\\
\hline \hline
\multicolumn{3}{c}{\textbf{BN with ReLU}}&$88.70\%\pm0.31\%$\\
\hline
\end{tabular}
}
\label{tab:Accuracy_eval_selu_serial}
\end{table}
\begin{figure}[ht]
\centering
\includegraphics[width=0.48\textwidth]{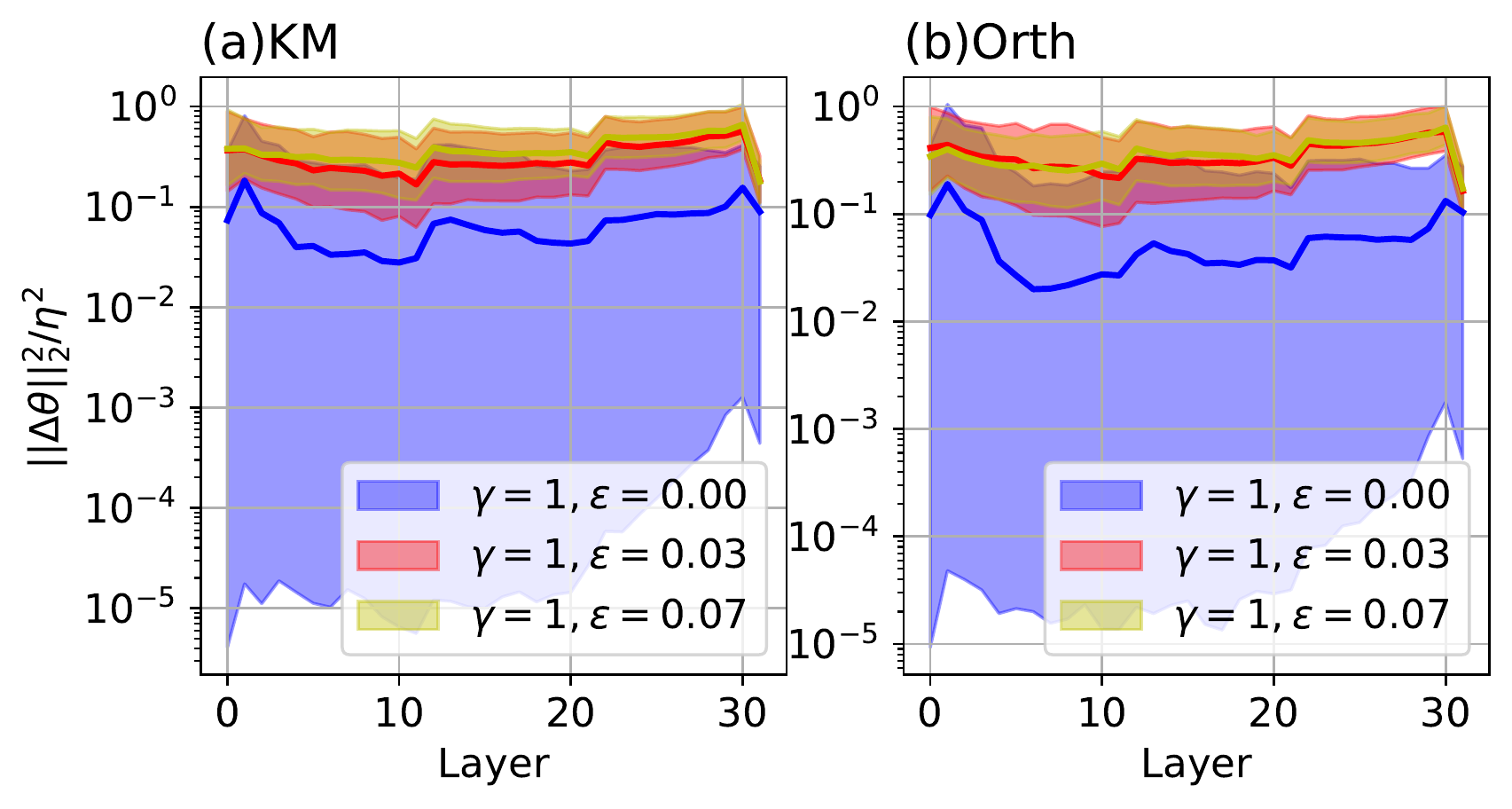}
\caption{Gradient norm distribution throughout the network with SeLU under different configurations.}
\label{fig:selu_grad}
\vspace{-15pt}
\end{figure}

For the orthogonal initialization, we find that the delta orthogonal initialization \cite{xiao2018dynamical} is less stable compared with the orthogonal initialization \cite{saxe2013exact}, which might be caused by that the sparse kernel in the delta orthogonal does not work well under the central limit theorem.
As shown in Table \ref{tab:Accuracy_eval_selu_serial}, when $\epsilon=0.07$, the test accuracy of our model is similar to the result in Klambauer et al. (2017)\cite{klambauer2017self}, which demonstrates that their work is a special case of ours. As our analysis suggests that $\epsilon$ should be slightly smaller than $\frac{1}{L}$, for the 32-layer network, we choose $\epsilon=0.03$, and the test accuracy is $+0.42\%$ and $+0.39\%$ higher than the original configuration with the Gaussian and orthogonal initialization, which indicates that the original choices of $\alpha$ and $\lambda$ are not optimal. As illustrated in Fig. \ref{fig:selu_grad}, a higher $\epsilon$ results in more neutral gradient norm distribution, which also accords to our prediction. Besides, our method can still achieve comparable result when $\gamma_0=2$. With SeLU, the orthogonal initialization does not have significant advantage over the Gaussian initialization, this reflects the normalization effectiveness of SeLU.

%% file: chapters/Exps/DenseNet.tex
Here we evaluate the performance of some initialization and normalization techniques on DenseNet.
\begin{table}[ht]
\vspace{-10pt}
\caption{Test accuracy on DenseNet (Cl=$95\%$).}
\centering
\resizebox{0.48\textwidth}{!}{
\begin{tabular}{c|c}
\hline
{\bf Approach}	&{\bf Test Acc.}
\\ \hline \hline
Kaiming Init + ReLU\cite{he2015delving}&$89.37\%\pm0.43\%$\\
\hline
\tabincell{c}{Orthogonal Init + leaky ReLU,\\ $\gamma=0.3$ (ours)}& $89.56\%\pm0.30\%$\\
\hline
\tabincell{c}{Orthogonal Init + SeLU, \\$\gamma_0=2, \epsilon=0.03$ (ours)}& $\mathbf{90.51\%\pm0.35\%}$\\
\hline \hline
Batch Normalization (BN) &$\mathbf{92.10\%\pm0.54\%}$\\
\hline
Scaled Weight Standardization (sWS) (ours)&$91.35\%\pm0.46\%$\\
\hline
Second Moment Normalization (SMN) (ours)&$\mathbf{92.06\%\pm0.25\%}$\\
\hline
\end{tabular}
}
\label{tab:Accuracy_eval_dense}
\vspace{-5pt}
\end{table}

\begin{figure}[ht]
\centering
\includegraphics[width=0.48\textwidth]{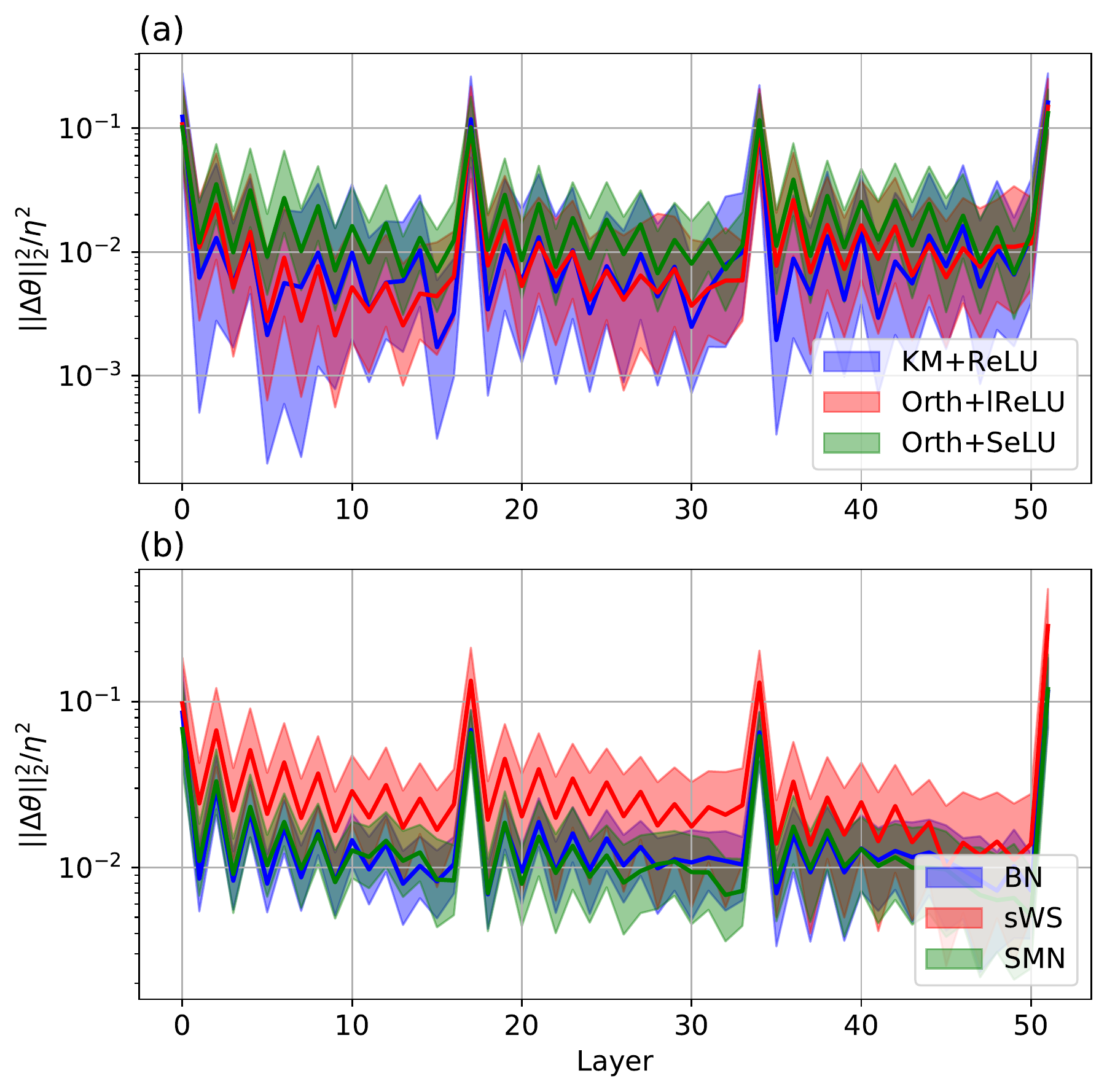}
\caption{Gradient norm distribution throughout DenseNet under different configurations.}
\label{fig:densenet_grad}
\vspace{-20pt}
\end{figure}

We take KM initialization as the baseline for initialization techniques and BN as baseline for normalization techniques. As listed in Table \ref{tab:Accuracy_eval_dense}, leaky ReLU yields $+0.19\%$ higher accuracy than the initialization baseline. For normalization techniques, while the accuracy of sWS is $0.75\%$ lower than BN, SMN we proposed is only $-0.04\%$ lower on accuracy, which further demonstrates its effectiveness. SeLU with $\epsilon=0.03$ surpasses other initialization techniques, whereas its accuracy is relatively lower than that with normalization techniques.

As illustrated in Fig. \ref{fig:densenet_grad}, even in the 52-layer network with a learning rate of 0.1, the gradient norm is still more concentrated than serial networks without dense connections, which verifies our conclusion that the dense connections can effectively stabilize the network. In Fig. \ref{fig:densenet_grad}(a), SeLU's gradient is more neutral compared with others; in Fig. \ref{fig:densenet_grad}(b), while SMN has a similar gradient distribution with BN, that of sWS is relatively higher. These phenomenons accord with the accuracy results in Table \ref{tab:Accuracy_eval_dense}.

%% file: chapters/Exps/Exp_ResNet.tex
Here we evaluate the performance of Fixup initialization and SMN on ResNet-56. The accuracy is summarized in Table \ref{tab:Accuracy_eval_resnet} and the gradient norm distribution is illustrated in Fig. \ref{fig:resnet_grad}. Fixup initialization with bias, scale, and mixup regularization achieves higher accuracy compared with BN, which illustrates its effectiveness. Moreover, although in Zhang et al. (2019) \cite{zhang2019fixup} $p$ is set to 2, we empirically show that $p=1.5$ can yield a slightly higher accuracy. The test accuracy of SMN is $0.43\%$ lower than BN, which can be reduced to $0.17\%$ with mixup regularization. However, as Fig. \ref{fig:resnet_grad} shows, SMN shares the similar gradient distribution with BN. These results imply that since the mean is estimated from weight kernels, compared with BN, SMN has a weaker regularization effect during training, and the data augmentation like mixup can partially compensate for it.

\begin{table}[ht]
\vspace{-5pt}
\caption{Test accuracy on ResNet-56 under different configurations (Cl=$95\%$).}
\centering
\resizebox{0.48\textwidth}{!}{
\begin{tabular}{c|c|c}
\hline
{\bf Method}  &{\bf Remarks} &{\bf Test Acc.}
\\ \hline \hline
\multirow{4}{*}{\textbf{Fixup}}&$p=2$\cite{zhang2019fixup}& $90.38\%\pm0.81\%$\\
\cline{2-3}
& $p=2$, b\&s\cite{zhang2019fixup}& $92.38\%\pm 0.15\%$\\
\cline{2-3}
& \tabincell{c}{$p=2$, b\& s, mixup\cite{zhang2019fixup}}& $93.93\%\pm 0.53\%$\\
\cline{2-3}
& \tabincell{c}{$p=1.5$, b\& s, mixup\\(ours)} & $\mathbf{94.28\%\pm 0.40\%}$\\
\hline \hline
\multirow{2}{*}{\textbf{BN}} & & $93.71\% \pm 0.27\%$\\
\cline{2-3}
& mixup & $\mathbf{94.10\% \pm 0.48\%}$\\
\hline
\multirow{2}{*}{\textbf{SMN} (ours)} & &$93.28\%\pm 0.82\%$\\
\cline{2-3}
& mixup & $\mathbf{93.93\%\pm 0.56\%}$\\
\hline
\end{tabular}
}
\label{tab:Accuracy_eval_resnet}
\end{table}
\vspace{-15pt}
\begin{figure}[ht]
\centering
\includegraphics[width=0.46\textwidth]{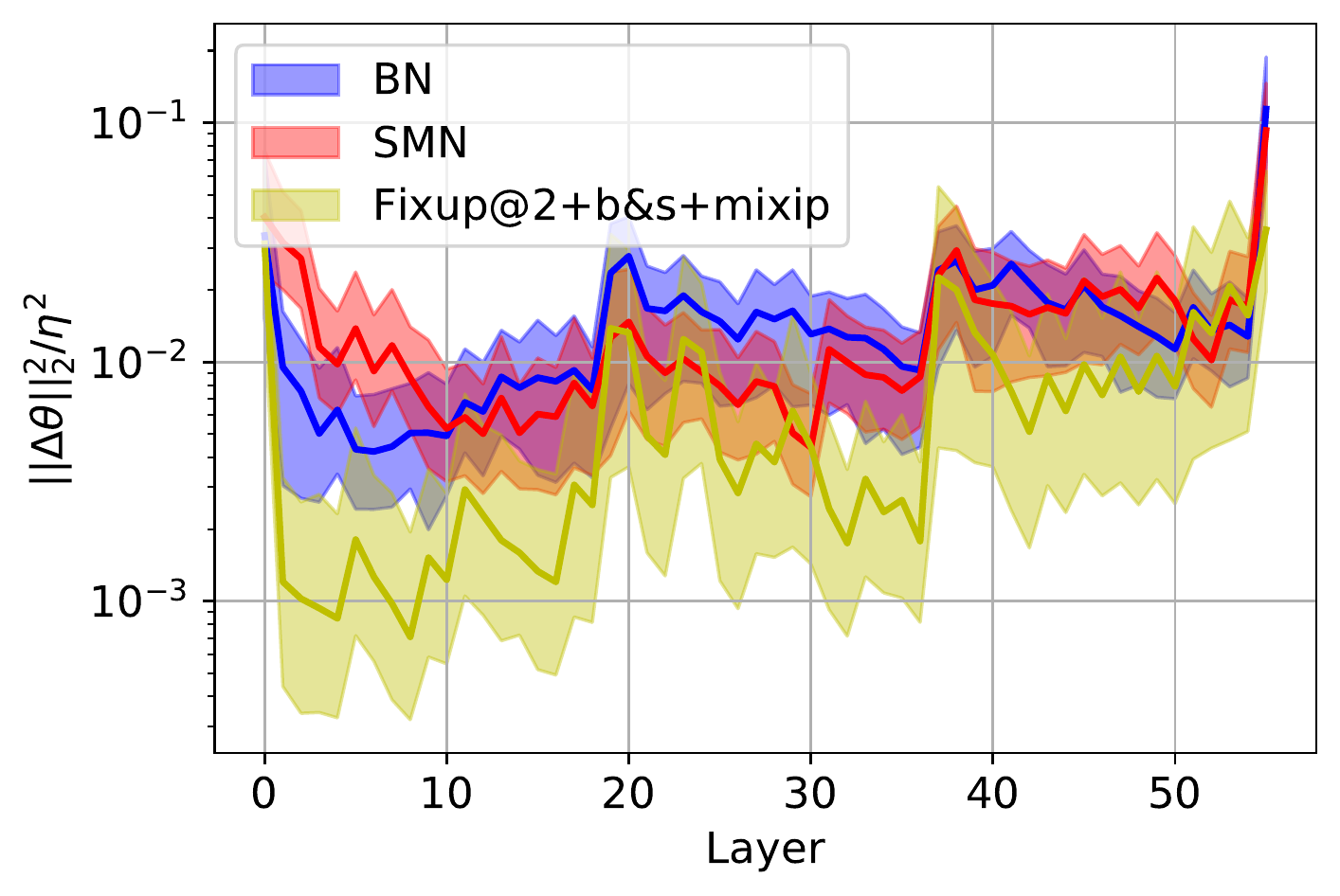}
\vspace{-10pt}
\caption{Gradient norm distribution throughout ResNet-56 under different configurations.}
\label{fig:resnet_grad}
\vspace{-10pt}
\end{figure}

%% file: chapters/Exps/imagenet.tex
Unlike previous theoretical studies \cite{haber2017stable, burkholz2018exact,tarnowski2018dynamical,pennington2017resurrecting,xiao2018dynamical} that only evaluate their conclusions on small datasets like MNIST and CIFAR-10, we further validate some of the important conclusions on ImageNet to demonstrate that they are still valid on large-scale networks.

We choose Conv MobileNet V1\cite{howard2017mobilenets} and ResNet 50 \cite{he2016deep} for serial and parallel networks, respectively. Conv MobileNet V1 is one of the latest serial networks, which has relatively good accuracy (71.7\% reported in Howard et al. (2017) \cite{howard2017mobilenets}) on ImageNet and is not over-parameterized like VGG \cite{simonyan2014very}. The ``Conv" means we use traditional convolutions instead of depthwise separable convolution, which is majorly due to two reasons. First, we find the latter one takes hundreds of epochs to converge. Second, as in depthwise convolution we have $c_{in}=1$, it is too small for most of mentioned techniques on serial networks. The Conv MobileNet V1 consists of 15 convolutional layers and a fully-connected layer at the end, wherein all the blocks are connected in serial and there are no shortcut connections between them. Originally, it is stabilized with BN. Since most methods for serial networks are not stable under the high learning rate, we follow the training scheme in Simonyan\& Zisserman (2014) \cite{simonyan2014very}, i.e. the model is trained for 90 epochs, the batch size is set to 512, and the initial learning rate is set to 0.02 with a decay by 10$\times$ at epoch 60, 75. For ResNet-50, we follow the classic training scheme in He et al. (2016)\cite{he2016deep}, i.e. the model is trained for 90 epochs, the batch size is set to 256, and the initial learning rate is set to 0.1 with a decay by 10$\times$ at epoch 30, 60. All the results are averaged over the last 10 epochs.

\textbf{On Conv MobileNet V1.} Previous experiments on CIFAR-10 illustrate that leaky ReLU with $\gamma\approx 1/L$ and SeLU surpass the accuracy of BN. Here we further evaluate their performance on ImageNet. The detailed configurations are: 1) Leaky ReLU, $\gamma=0.3$ with the orthogonal initialization; 2) SeLU, $\epsilon=0.06$ or $0.03$, $\gamma_0=1$ with the Gaussian initialization. We choose $\epsilon=0.06$ for it is slightly smaller than $\frac{1}{L}=0.067$. The results are given in Table \ref{tab:Accuracy_serial_imagenet}.

\begin{table}[ht]
\vspace{-10pt}
\caption{Test error of methods on Conv MobileNet V1 (Cl=$95\%$).}
\centering
\resizebox{0.48\textwidth}{!}{
\begin{tabular}{c|c|c}
\hline
{\bf Method}	&{\bf Top-1 Error} &{\bf Top-5 Error}
\\\hline \hline
SeLU\cite{klambauer2017self}&\multicolumn{2}{|c}{Explode in the first epoch}\\
\hline
SeLU $\epsilon=0.06$ (ours)&$\mathbf{30.37\%\pm0.10\%}$ &$\mathbf{11.67\%\pm0.03\%}$\\
\hline
SeLU $\epsilon=0.03$ (ours)&$30.87\%\pm0.07 \%$ &$11.84\%\pm0.04\%$\\
\hline \hline
ReLU, Gaussian\cite{he2015delving}&$31.16\%\pm0.08\%$&$11.87\%\pm0.06\%$\\
\hline
lReLU, Gaussian (ours)&$29.39\%\pm0.08 \%$ &$\mathbf{10.82\%\pm0.07\%}$\\
\hline
lReLU, Orth (ours) &$\mathbf{29.36\%\pm0.13 \%}$ &$10.82\%\pm0.08\%$\\
\hline
lReLU, Delta Orth (ours) &$29.47\%\pm0.09 \%$ &$10.92\%\pm0.08\%$\\
\hline \hline
BN&$28.58\%\pm0.07\%$ &$10.16\%\pm0.05\%$\\
\hline
\end{tabular}
}
\label{tab:Accuracy_serial_imagenet}
\end{table}

The original configuration of SeLU \cite{klambauer2017self} suffers from the gradient explosion due to the too large $\epsilon$. Via $\epsilon=0.06$, we reach $30.37\%$ top-1 error with Gaussian weights. However, for smaller $\epsilon$, i.e. $0.03$, the top-1 error is 0.5\% higher, for its normalization effectiveness is lower. For leaky ReLU with $\gamma=0.3$ and the Gaussian initialization, the top-1 error is only $0.89\%$ higher than the BN baseline and $1.77\%$ lower than the ReLU + Gaussian baseline.

\textbf{On ResNet-50.} For ResNet-50, we test the performance of the Fixup initialization and our SMN For the former one, we test both Zhang et al. (2019) \cite{zhang2019fixup}'s original configuration and ours with $p=1.5$. The scalar multiplier and bias are added and the interpolation coefficient in mixup is set to 0.7, just following \cite{zhang2019fixup}. For the latter one, we directly replace BN in original ResNet-50 with SMN without any further modification. For the BN baseline, the interpolation coefficient in mixup is set to 0.2, which is reported to be the best \cite{zhang2019fixup}. The results are summarized in Table \ref{tab:Accuracy_res_imagenet}.

\begin{table}[ht]
\vspace{-10pt}
\caption{Test error of methods on ResNet-50 (Cl=$95\%$).}
\centering
\resizebox{0.48\textwidth}{!}{
\begin{tabular}{c|c|c}
\hline
{\bf Method}	&{\bf Top-1 Error} &{\bf Top-5 Error}
\\ \hline \hline
BN&$24.35\%\pm0.15\%$&$7.49\%\pm0.09\%$\\
\hline
BN+mixup&$23.81\%\pm0.13\%$&$\mathbf{6.86\%\pm0.08\%}$\\
\hline
SMN (ours)&$24.90\%\pm0.19\%$&$7.65\%\pm0.14\%$\\
\hline
SMN+mixup (ours)&$\mathbf{23.74\%\pm0.23\%}$&$6.94\%\pm0.10\%$\\
\hline
L1-MN+mixup (ours)&$24.04\%\pm0.19\%$&$7.10\%\pm0.14\%$\\
\hline
Fixup\cite{zhang2019fixup}&$24.77\%\pm0.15\%$&$7.72\%\pm0.13\%$\\
\hline
Fixup (ours)&$24.72\%\pm0.12\%$&$7.70\%\pm0.10\%$\\
\hline
WN\cite{salimans2016weight}&$33\%$\cite{gitman2017comparison}&Not reported\\
\hline
\end{tabular}
}
\label{tab:Accuracy_res_imagenet}
\end{table}

Without the mixup regularization, the top-1 error of our SMN is 0.55\% higher than BN. However, we also observe that its top-1 training error is 0.68\% lower, which implies that the test accuracy loss is mainly due to the lack of regularity. Inspired by Zhang et al. (2019) \cite{zhang2019fixup}, we further utilize the mixup\cite{zhang2017mixup} to augment the input data with the interpolation coefficient of 0.2, which is the same with the baseline configuration. Then, the top-1 error becomes $23.74\%$, which is $0.07\%$ lower than BN. Also, we evaluate the L1-norm configuration with the mixup regularization, and the top-1 error is still comparable with BN. Notably, our SMN has a similar computational complexity with WN. \modify{In Appendix \ref{appendix:l2n_overhead}, we follow the analysis in Chen et al. (2019) \cite{chen2019effective} and find that SMN can reduce the number of operations from 13 to 10, which would bring about $30\%$ speedup. Moreover, under the same mixup configuration, SMN achieves the similar performance with BN, this demonstrates it could be a powerful substitute for BN.} For the Fixup initialization, our configuration can reach the same test error of the configuration in Zhang et al. (2019) \cite{zhang2019fixup}.

%% file: chapters/conclusion.tex
In this paper, we propose a novel metric, block dynamical isometry, that can characterize DNNs using the gradient norm equality property. A comprehensive and highly modularized statistical framework based on advanced tools in free probability is provided to simplify the evaluation of our metric. Compared with existing theoretical studies, our framework can be applied to networks with various components and complex connections, which is much easier to use and only requires weaker prerequisites that are easy to verify. Powered by our novel metric and framework, unlike previous studies that only focus on a particular network structure or stabilizing methodology, we analyze extensive techniques including initialization, normalization, self-normalizing neural network and shortcut connections. Our analyses not only show that the gradient norm equality is a universal philosophy behind these methods but also provides inspirations for the improvement of existing techniques and the development of new methods. As demos, we introduce an activation function selection strategy for initialization, a novel configuration for weight normalization, a depth-aware way to derive coefficient in SeLU, and the second moment normalization. These methods achieve advanced results on both CIFAR-10 and ImageNet with rich network structures. Besides what we have presented in this paper, there is still potential in our framework that is not fully exploited. For instance, our analysis in Section \ref{sec:self_normalizing_nn} shows ``SeLU" may not be the only choice for self-normalizing neural networks. Moreover, although we focus on CNNs in this paper, the methodology also has the potential to improve other models like recurrent neural networks and spiking neural networks. Last but not least, our framework can also be utilized in other norm-based metrics like GSC \cite{philipp2018gradients}.

%% file: chapters/biography.tex
\vspace{-30pt}
\begin{IEEEbiography}[{\includegraphics[width=1in, height=1.25in, clip, keepaspectratio]{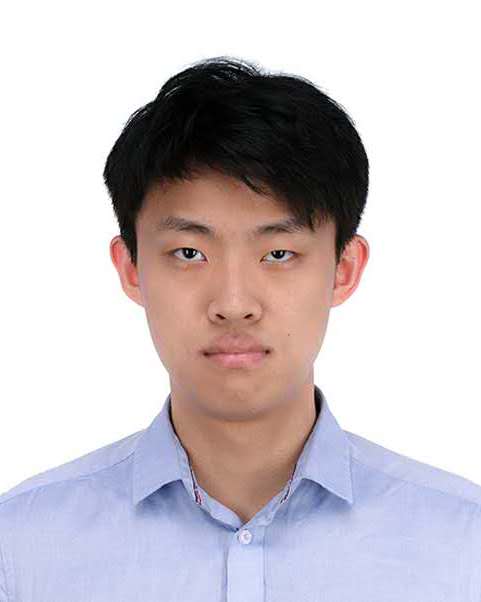}}] {Zhaodong Chen} received his B.E. degree from Tsinghua University, China in 2019. He is currently pursuing the Ph.D. degree with the Department of Computer Engineering, University of California, Santa Barbara. He is currently doing research in Scalable Energy-Efficient Architecture Lab. His research interest lays in deep learning and computer architecture.
\end{IEEEbiography}
\vspace{-60pt}
\begin{IEEEbiography}[{\includegraphics[width=1in, height=1.25in, clip, keepaspectratio]{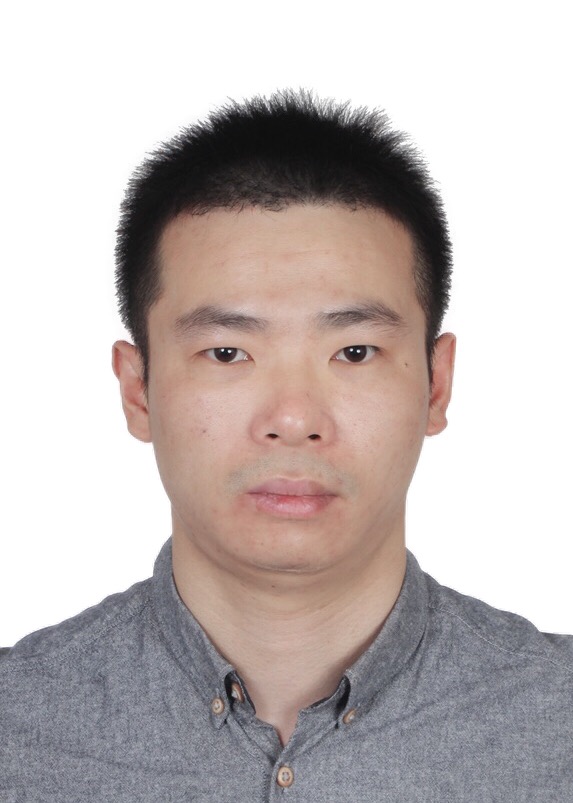}}] {Lei Deng} received the B.E. degree from University of Science and Technology of China, Hefei, China in 2012, and the Ph.D. degree from Tsinghua University, Beijing, China in 2017. He is currently a Postdoctoral Fellow at the Department of Electrical and Computer Engineering, University of California, Santa Barbara, CA, USA. His research interests span the area of brain-inspired computing, machine learning, neuromorphic chip, computer architecture, tensor analysis, and complex networks. Dr. Deng has authored or co-authored over 40 refereed publications. He was a PC member for \emph{ISNN} 2019. He currently serves as a Guest Associate Editor for \emph{Frontiers in Neuroscience} and \emph{Frontiers in Computational Neuroscience}, and a reviewer for a number of journals and conferences. He was a recipient of MIT Technology Review Innovators Under 35 China 2019.
\end{IEEEbiography}
\vspace{-40pt}
\begin{IEEEbiography}[{\includegraphics[width=1in, height=1.25in, clip, keepaspectratio]{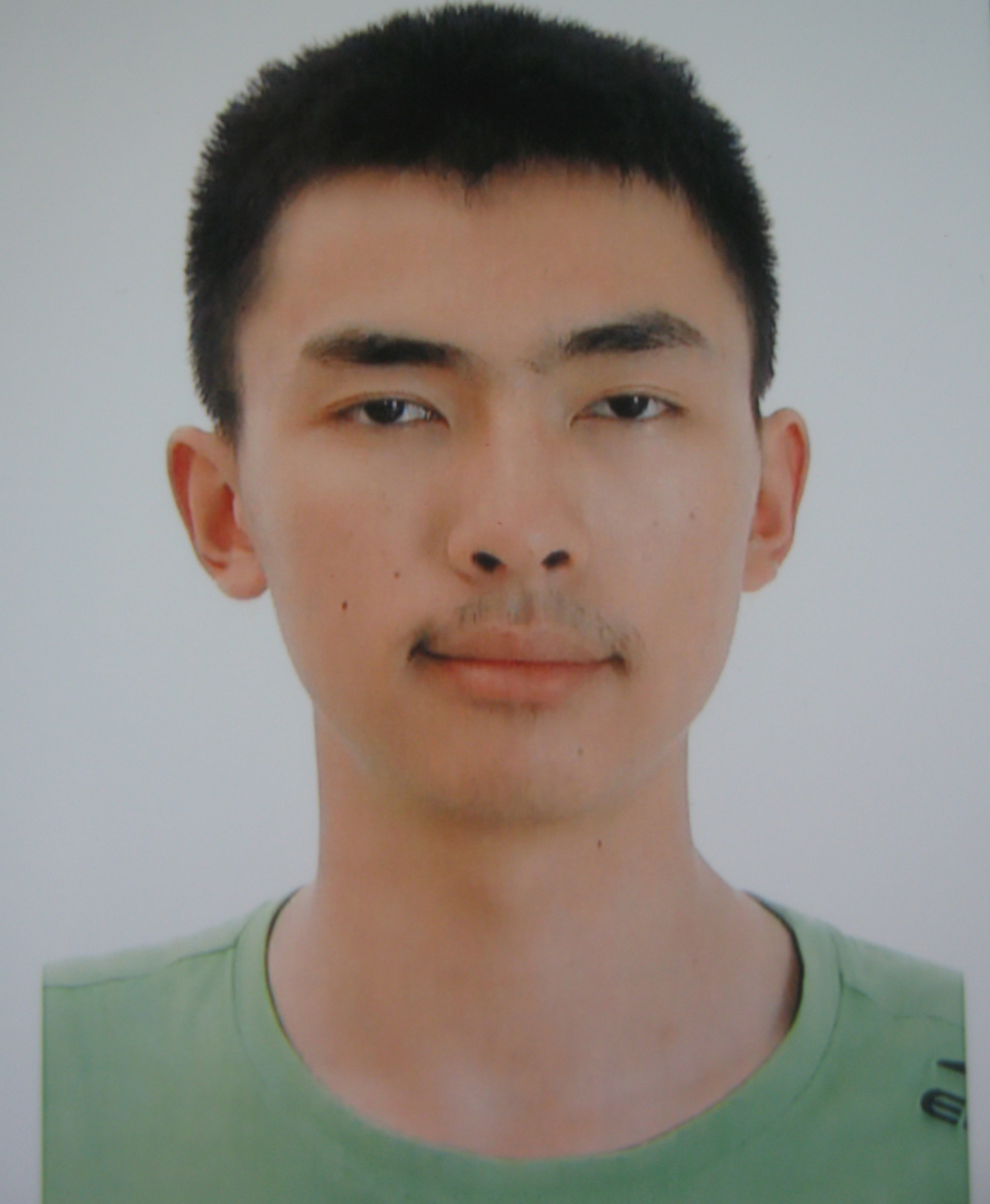}}] {Bangyan Wang} received his B.E. degree from Tsinghua University, China in 2017. He is currently a Ph.D. student at the Department of Electrical and Computer Engineering, University of California, Santa Barbara. His current research interests include domain-specific accelerator design and tensor analysis.
\end{IEEEbiography}
\vspace{-40pt}
\begin{IEEEbiography}[{\includegraphics[width=1in,height=1.25in,clip,keepaspectratio]{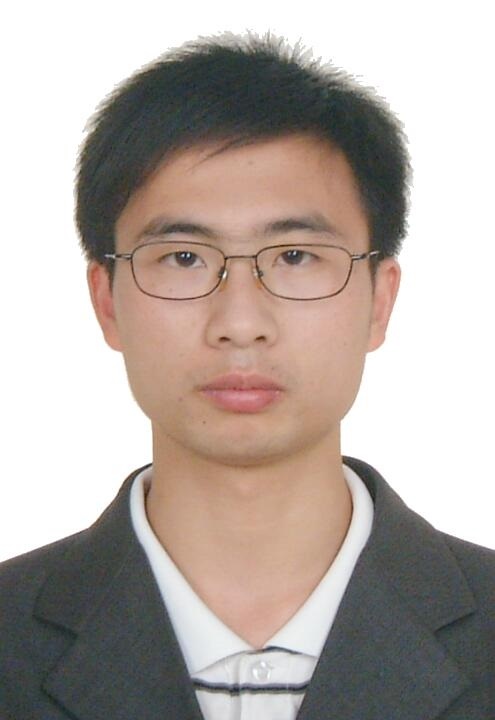}}]{Guoqi Li}  received the B.E. degree from the Xi’an University of Technology, Xi’an, China, in 2004, the M.E. degree from Xi’an Jiaotong University, Xi’an, China, in 2007, and the Ph.D. degree from Nanyang Technological University, Singapore, in 2011. He was a Scientist with Data Storage Institute and the Institute of High Performance Computing, Agency for Science, Technology and Research (ASTAR), Singapore, from 2011 to 2014. He is currently an Associate Professor with the Center for Brain Inspired Computing Research (CBICR), Tsinghua University, Beijing, China. His current research interests include machine learning, brain-inspired computing, neuromorphic chip, complex systems and system identification. 

Dr. Li has authored or co-authored over 100 journal and conference papers. He has been actively involved in professional services such as serving as the International Technical Program Committee Member, the Publication Chair, the Tutorial/Workshop Chair, and the Track Chair for international conferences. He is currently an Editorial-Board Member for \emph{Control and Decision} and \emph{Frontiers in Neuroscience}, and a Guest Associate Editor for \emph{Frontiers in Neuroscience}. He is a reviewer for \emph{Mathematical Reviews} published by the American Mathematical Society and serves as a reviewer for a number of other prestigious journals and conferences. He was the recipient of the 2018 First Class Prize in Science and Technology of the Chinese Institute of Command and Control, Best Paper Awards (\emph{EAIS} 2012 and \emph{NVMTS} 2015), and the 2018 Excellent Young Talent Award of Beijing Natural Science Foundation.
\end{IEEEbiography}
\vspace{-40pt}
\begin{IEEEbiography}[{\includegraphics[width=1in,height=1.25in,clip,keepaspectratio]{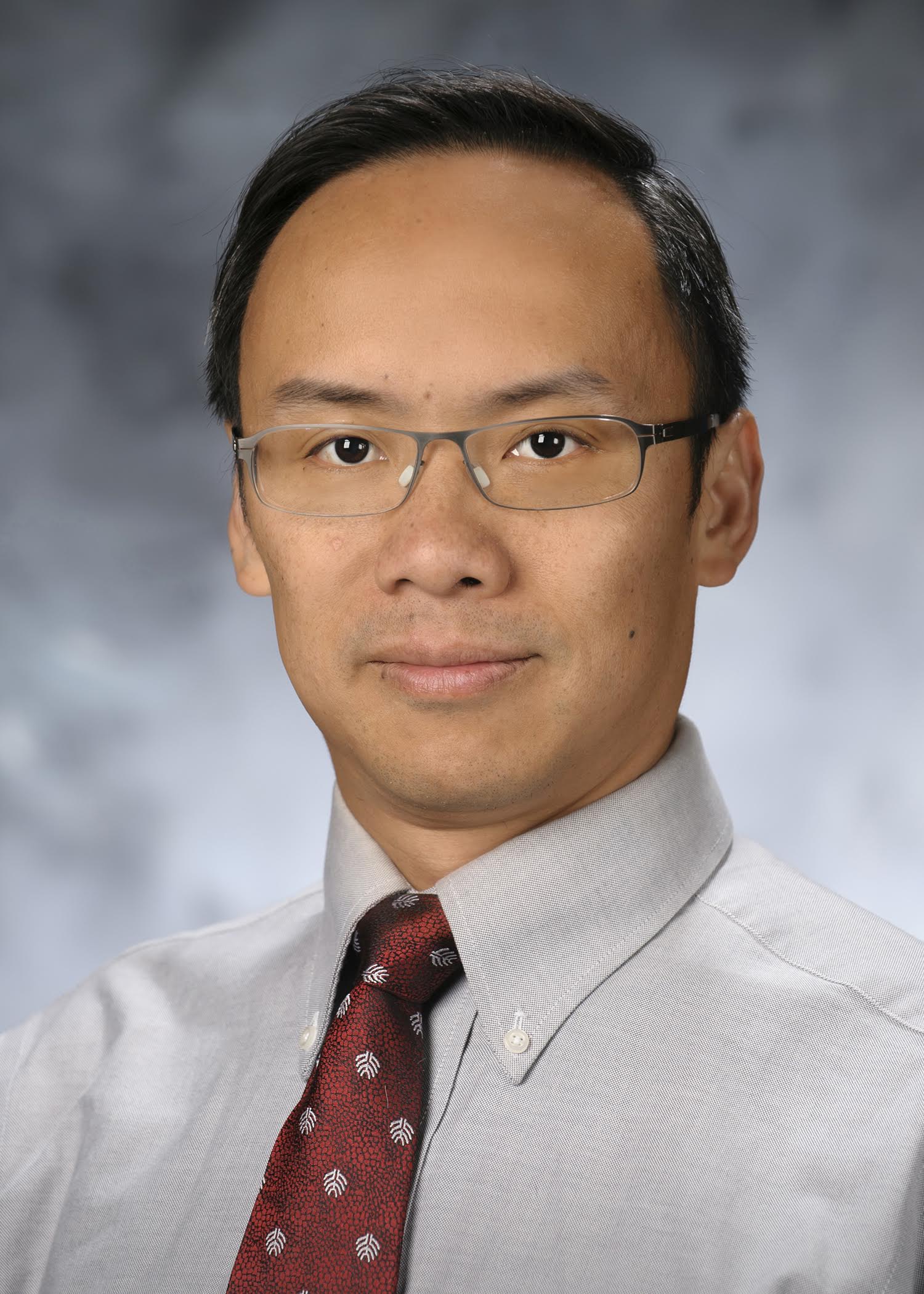}}]{Yuan Xie} received the B.S. degree in Electronic Engineering from Tsinghua University, Beijing, China in 1997, and M.S. and Ph.D. degrees in Electrical Engineering from Princeton University, NJ, USA in 1999 and 2002, respectively. He was an Advisory Engineer with IBM Microelectronic Division, VT, USA from 2002 to 2003. He was a Full Professor with Pennsylvania State University, PA, USA from 2003 to 2014. He was a Visiting Researcher with Interuniversity Microelectronics Centre (IMEC), Leuven, Belgium from 2005 to 2007 and in 2010. He was a Senior Manager and Principal Researcher with AMD Research China Lab, Beijing, China from 2012 to 2013. He is currently a Professor with the Department of Electrical and Computer Engineering, University of California at Santa Barbara, CA, USA. His interests include VLSI design, Electronics Design Automation (EDA), computer architecture, and embedded systems. 

Dr. Xie is an expert in computer architecture who has been inducted to \emph{ISCA}/\emph{MICRO}/\emph{HPCA} Hall of Fame and IEEE/AAAS/ACM Fellow. He was a recipient of Best Paper Awards (\emph{HPCA} 2015, \emph{ICCAD} 2014, \emph{GLSVLSI} 2014, \emph{ISVLSI} 2012, \emph{ISLPED} 2011, \emph{ASPDAC} 2008, \emph{ASICON} 2001) and Best Paper Nominations (\emph{ASPDAC} 2014, \emph{MICRO} 2013, \emph{DATE} 2013, \emph{ASPDAC} 2010-2009, \emph{ICCAD} 2006), the 2016 IEEE Micro Top Picks Award, the 2008 IBM Faculty Award, and the 2006 NSF CAREER Award. He served as the TPC Chair for \emph{ICCAD} 2019, \emph{HPCA} 2018, \emph{ASPDAC} 2013, \emph{ISLPED} 2013, and \emph{MPSOC} 2011, a committee member in IEEE Design Automation Technical Committee (DATC), the Editor-in-Chief for \emph{ACM Journal on Emerging Technologies in Computing Systems}, and an Associate Editor for \emph{ACM Transactions on Design Automations for Electronics Systems}, \emph{IEEE Transactions on Computers}, \emph{IEEE Transactions on Computer-Aided Design of Integrated Circuits and Systems}, \emph{IEEE Transactions on VLSI, IEEE Design and Test of Computers}, and \emph{IET Computers and Design Techniques}. Through extensive collaboration with industry partners (e.g. AMD, HP, Honda, IBM, Intel, Google, Samsung, IMEC, Qualcomm, Alibaba, Seagate, Toyota, etc.), he has helped the transition of research ideas to industry. 
\end{IEEEbiography}

%% file: chapters/appendix_proof.tex
\section{Proofs}



\subsection{A review of free probability: useful definitions and theorems}
\input{chapters/proofs/proof_review_free_probability.tex}
\subsection{Proof of Theorem \ref{theorem:multiplication}}\label{proof:multiplication}

\input{chapters/proofs/proof_mul.tex}

\subsection{Proof of Theorem \ref{theorem:Addition}}\label{proof:addition}
\input{chapters/proofs/proof_add.tex}

\subsection{Proof of Proposition \ref{prop:1st_moment_unitary_invariant_exp_diag}}\label{proof:1st_moment_unitary_invariant_exp_diag}
\input{chapters/prerequisites/mul.tex}
\subsection{Proof of Proposition \ref{prop:properties_of_eom_cm}}\label{proof:properties_of_eom_cm}
\input{chapters/proofs/proof_properties_of_eom_cm.tex}
\subsection{Proof of Table \ref{tab:parts_library}}\label{proof:part_library}
\input{chapters/proofs/proof_parts_library.tex}

\subsection{Proof of Equ. \ref{equ:eigs_l2norm}}\label{proof:eigs_l2norm}
\input{chapters/proofs/proof_second_order_moment_norm.tex}

\subsection{Proof of Proposition \ref{prop:gaussian_activ_2nd_invariant}}\label{proof:gaussian_activ_2nd_invariant}
\input{chapters/proofs/proof_gaussian_active.tex}

\subsection{Proof of Proposition \ref{prop:orth_activ_2nd_invariant}}\label{proof:orth_activ_2nd_invariant}
\input{chapters/proofs/proof_diag_orth.tex}

\subsection{Proof of Proposition \ref{prop:general_linear_transforms}}\label{proof:general_linear_transforms}
\input{chapters/proofs/proof_general_linear_op.tex}

\subsection{Proof of Proposition \ref{prop:evolution_2nd_norm_net }}\label{proof:evolution_2nd_norm_net}
\input{chapters/proofs/proof_second_moment.tex}

\subsection{Proof of Proposition \ref{prop:selu_expect_JJ}}\label{proof:selu_expect_JJ}

\input{chapters/proofs/selu_ej.tex}

\subsection{Proof of Proposition \ref{prop:prerequisite_series_parallel_hybrid_networks}}\label{proof:prerequisite_series_parallel_hybrid_networks}
\input{chapters/proofs/proof_pre_s_p_h_network.tex}
\subsection{Proof of Proposition \ref{prop:plus_1_trick}}\label{proof:plus_1_trick}
\input{chapters/proofs/proof_plus_1_trick.tex}


%% file: chapters/proofs/proof_review_free_probability.tex
Before starting, we review several transforms of free probability theory that will be used in later proofs following Mingo \& Speicher (2017) \cite{mingo2017free}, Ling \& Qiu (2018) \cite{ling2018spectrum}, and Pennington et al. (2017) \cite{pennington2017resurrecting}.

\begin{definition}
\textbf{(Stieltjes transform)} Let $\nu$ be a probability measure on $\mathbb{R}$ whose probability density function is $\rho_X$ and for $z \notin \mathbb{R}$, the stieltjes transform of $\rho_X$ is defined as
\begin{equation}
    G_X(z) = \int_{\mathbb{R}}\frac{\rho_X(t)}{z-t}dt.
\label{equ:stieltjes transform}
\end{equation}
$G_{\mathbf{X}}(z)$ can be expanded into a power series with coefficients as $G_{\mathbf{X}}(z)=\sum_{k=0}^{\infty}\frac{\alpha_k(\lambda_{\mathbf{X}})}{z^{k+1}}$ where $\alpha_k(\lambda_{\mathbf{X}})=\int  \rho_{\mathbf{X}}(\lambda)\lambda^k d\lambda$, which is the $k^{th}$ moment of $\lambda_{\mathbf{X}}$. \cite{ling2018spectrum, pennington2017resurrecting}
\label{def:stieltjes_transform}
\end{definition}

\begin{definition}
\textbf{(Moment generating function)} Given a stieltjes transform of $\rho_X$, the moment generating function (M-transform) is defined as
\begin{equation}
    M_X(z) = zG_X(z) - 1.
\label{equ:M-transform}
\end{equation}
\label{def:M-transform}
\end{definition}

\begin{definition}
\textbf{(S-transform)} Given the inverse function of $M^{-1}_X(M_X(z))=z$, which can be obtained through the Lagrange inversion theorem: $M^{-1}_X(z) = \frac{\alpha_1}{z} + \frac{\alpha_2}{\alpha_1}+...$, the S-transform is defined as
\begin{equation}
    S_X(z)=\frac{1+z}{zM_X^{-1}(z)}.
\label{equ:S-transform}
\end{equation}
\label{def:S-transform}
\end{definition}

\begin{definition}
\textbf{(R-transform)} Given the inverse function of $G^{-1}_X(G_X(z))=z$, the R-transform is defined as
\begin{equation}
    R_X(z) = G^{-1}_X(z) - \frac{1}{z}.
\label{equ:R-transform}
\end{equation}
\label{def:R-transform}
\end{definition}

One of the most important properties of S-transform and R-transform is that for any two freely independent non-commutative random variables $\mathbf{X,Y}$, the S-transform and R-transform have the following definite (convolution) properties \cite{ling2018spectrum}
\begin{equation}
    S_{XY}(z) = S_X(z)S_Y(z),~~R_{XY}(z) = R_X(z)R_Y(z).
\label{equ:S_R_transformation_property}
\end{equation}

While free probability handles the spectrum density of random matrix $\mathbf{X}$, we are actually interested in the property of $\phi(\mathbf{X}^k)$. As a result, we build a bridge between them with the following lemma:

\begin{lemma}
Let $\mathbf{X}$ be a wide enough real symmetric random matrix, then we have $\alpha_k(\lambda_{\mathbf{X}})=\phi(\mathbf{X}^k)$.
\label{lemma:bridge_eig_phi}
\end{lemma}
\begin{proof}
Since $\mathbf{X}$ is a real symmetric random matrix, with eigendecomposition, it can be decomposed as $\mathbf{X}=\mathbf{Q\Lambda Q}^T$, where $\mathbf{\Lambda}=diag(\lambda_1, \lambda_2, ...)$ is a diagonal matrix whose elements are the eigenvalues of $\mathbf{X}$. As a result, we have
\begin{equation}
    \alpha_k(\lambda_{\mathbf{X}}) = \lim_{N\rightarrow \infty}\frac{1}{N}\sum_{i=1}^N\lambda_i^k = \phi(\mathbf{X}^k).
\end{equation}
\end{proof}

%% file: chapters/proofs/proof_mul.tex
\textbf{Theorem \ref{theorem:multiplication}. }\textit{\textbf{(Multiplication).} Given $\mathbf{J} := \Pi_{i=L}^1\mathbf{J_i}$, where $\{\mathbf{J_i}\in\mathbb{R}^{m_{i}\times m_{i-1}}\}$ is a series of independent random matrices.}

\textit{If $(\Pi_{i=L}^1\mathbf{J_i})(\Pi_{i=L}^1\mathbf{J_i})^T$ is at least $1^{st}$ moment unitarily invariant, we have}
\begin{equation}
    \phi\left((\Pi_{i=L}^1\mathbf{J_i})(\Pi_{i=L}^1\mathbf{J_i})^T\right) = \Pi_{i}\phi(\mathbf{J_iJ_i}^T).
\end{equation}
\textit{If $(\Pi_{i=L}^1\mathbf{J_i})(\Pi_{i=L}^1\mathbf{J_i})^T$ is at least $2^{nd}$ moment unitarily invariant, we have}
\begin{equation}
\begin{split}
    &\varphi((\Pi_{i=L}^1\mathbf{J_i})(\Pi_{i=L}^1\mathbf{J_i})^T)=\\
    &~~~~~\phi^2\left((\Pi_{i=L}^1\mathbf{J_i})(\Pi_{i=L}^1\mathbf{J_i})^T\right) \sum_{i}\frac{m_{L}}{m_{i}}\frac{\varphi(\mathbf{J_iJ_i}^T)}{\phi^2(\mathbf{J_iJ_i}^T)}.
\end{split}
\end{equation}
\rule[0pt]{0.49\textwidth}{0.05em}

To begin with, we prove the following lemmas:

\begin{lemma}
Let $[f(z)]_{@k}$ denote the truncated series of the power series expanded from $f(z)$: i.e. given $f(z) = f_0 + f_1z+...$, we have $[f(z)]_{@2} = f_0 + f_1z$. 

Let $\mathbf{A}$, $\mathbf{B}$ be two hermitian matrices, if for $i\in\{1,..,k\}$, $k\in\{1,2\}$, we have $\phi(\mathbf{A}^k) = \phi(\mathbf{B}^k)$, then we have
\begin{equation}
    \left[S_{\mathbf{A}}(z)\right]_{@k} =  \left[S_{\mathbf{B}}(z)\right]_{@k}.
\end{equation}
\label{lemma:equivalent_S_transform}
\end{lemma}
\begin{proof}
Let's consider a hermitian matrix $\mathbf{X}$ whose spectrum density is $\rho_{X}(\lambda)$, following Definition \ref{def:S-transform}, we have
\begin{equation}
    M_{\mathbf{X}}^{-1} = \frac{\alpha_1(\lambda_{\mathbf{X}})}{z} + \frac{\alpha_2(\lambda_{\mathbf{X}})}{\alpha_1(\lambda_{\mathbf{X}})} + ...
\end{equation}
Therefore, we have
\begin{equation}
    S_{\mathbf{X}}(z) = \frac{1}{\alpha_1(\lambda_{\mathbf{X}})} + \left(\frac{1}{\alpha_1(\lambda_{\mathbf{X}})} - \frac{\alpha_2(\lambda_{\mathbf{X}})}{\left(\alpha_1(\lambda_{\mathbf{X}})\right)^3}\right)z + ...
\label{equ:expand_S}
\end{equation}
As $[S_{\mathbf{X}}]_{@k}$ is only determined by $\alpha_1(\lambda_{\mathbf{X}}),...,\alpha_k(\lambda_{\mathbf{X}})$, and with Lemma \ref{lemma:bridge_eig_phi}, $\alpha_i(\lambda_{\mathbf{X}})=\phi(\mathbf{X}^k)$, we have
\begin{equation}
    \left[S_{\mathbf{A}}(z)\right]_{@k} =  \left[S_{\mathbf{B}}(z)\right]_{@k}.
\end{equation}
\end{proof}

\begin{lemma}
Let $\mathbf{A}\in \mathbb{R}^{m\times n},\mathbf{B} \in \mathbb{R}^{n\times q}$ be two random matrices, $\mathbf{A}^T\mathbf{A}$ and $\mathbf{BB}^T$ are freely independent, then we have

\begin{equation}
\begin{split}
    & \phi(\mathbf{ABB}^T\mathbf{A}^T) = \phi(\mathbf{AA}^T)\phi(\mathbf{BB}^T),\\
    & \varphi(\mathbf{ABB}^T\mathbf{A}^T) = \frac{m}{n}\phi^2(\mathbf{AA}^T)\varphi(\mathbf{BB}^T)\! +\! \phi^2(\mathbf{BB}^T)\varphi(\mathbf{AA}^T)
\end{split}
\end{equation}
\label{lemma:cyclic_variant}
\end{lemma}
\begin{proof}
Firstly, with the cyclic-invariant property of trace operator, we have $Tr\left((\mathbf{ABB}^T\mathbf{A}^T)^k\right)=Tr\left((\mathbf{A}^T\mathbf{ABB}^T)^k\right)$. As $\phi:=E[tr]$, we have
\begin{equation}
    m\phi\left((\mathbf{ABB}^T\mathbf{A}^T)^k\right) = n\phi\left((\mathbf{A}^T\mathbf{ABB}^T)^k\right).
\end{equation}
As a result, we have
\begin{equation}
\begin{split}
    & \left[S_{\mathbf{ABB}^T\mathbf{A}^T}\right(z)]_{@1}=\frac{m}{n}\left[S_{\mathbf{A}^T\mathbf{ABB}^T}\right]_{@1},\\
    & \left[S_{\mathbf{ABB}^T\mathbf{A}^T}\right(z)]_{@2} = \frac{m}{n}\left[ S_{\mathbf{A}^T\mathbf{ABB}^T}\right]_{@2}\\
    & ~~~~~~~~~~~~~~~~~~~~~~~~~~~~~~~~~+ \left(\frac{n}{m}-1\right)\frac{\phi((\mathbf{\mathbf{ABB}^T\mathbf{A}^T})^2)}{\phi^3(\mathbf{\mathbf{ABB}^T\mathbf{A}^T})}z.
\end{split}
\label{equ:expand_S_1_2}
\end{equation}
As $\mathbf{A}^T\mathbf{A}$ and $\mathbf{BB}^T$ are freely independent, Equation \eqref{equ:S_R_transformation_property} yields
\begin{equation}
    S_{\mathbf{A}^T\mathbf{ABB}^T} = S_{\mathbf{A}^T\mathbf{A}}S_{\mathbf{BB}^T}.
\label{equ:expand_S_3}
\end{equation}

Similarly, as $\forall 0<k\in \mathbb{N}\le 2,~~n\phi\left((\mathbf{A}^T\mathbf{A})^k\right)=m\phi\left((\mathbf{AA}^T)^k\right)$, we have
\begin{equation}
\begin{split}
    & \left[S_{\mathbf{A}^T\mathbf{A}}\right(z)]_{@1}=\frac{n}{m}\frac{1}{\phi(\mathbf{AA}^T)},\\
    & \left[S_{\mathbf{A}^T\mathbf{A}}\right(z)]_{@2} = \\
    & \frac{n}{m}\left[\frac{1}{\phi(\mathbf{AA}^T)} + \left(\frac{1}{\phi(\mathbf{AA}^T)} - \frac{n}{m} \frac{\phi\left((\mathbf{AA}^T)^2\right)}{\phi^3(\mathbf{AA}^T)}\right)z\right].
\end{split}
\label{equ:expand_S_4}
\end{equation}
For the sake of simplicity, we denote $\widetilde{\mathcal{X}}:=\phi(\mathbf{XX}^T)$, $\widetilde{\mathcal{X}}_2 = \phi\left((\mathbf{XX}^T)^2\right)$. With Equation \eqref{equ:expand_S}, Equation \eqref{equ:expand_S_1_2}, Equation \eqref{equ:expand_S_3} and Equation \eqref{equ:expand_S_4}, we can get the following equations:
\begin{equation}
    \begin{split}
        &\frac{1}{\widetilde{\mathcal{AB}}} = \frac{1}{\widetilde{\mathcal{A}}\widetilde{\mathcal{B}}},\\
        &\frac{1}{\widetilde{\mathcal{AB}}} - \frac{n\widetilde{\mathcal{AB}}_2}{m\widetilde{\mathcal{AB}}^3} = \frac{1}{\widetilde{\mathcal{A}}}\left(\frac{1}{\widetilde{\mathcal{B}}} - \frac{\widetilde{\mathcal{B}}_2}{\widetilde{\mathcal{B}}^3}\right) + \frac{1}{\widetilde{\mathcal{B}}}\left(\frac{1}{\widetilde{\mathcal{A}}} - \frac{n\widetilde{\mathcal{A}}_2}{m\widetilde{\mathcal{A}}^3}\right).
    \end{split}
\end{equation}
whose solution is
\begin{equation}
\begin{split}
    & \phi(\mathbf{ABB}^T\mathbf{A}^T) = \phi(\mathbf{AA}^T)\phi(\mathbf{BB}^T),\\
    & \varphi(\mathbf{ABB}^T\mathbf{A}^T) = \frac{m}{n}\phi^2(\mathbf{AA}^T)\varphi(\mathbf{BB}^T)\! +\! \phi^2(\mathbf{BB}^T)\varphi(\mathbf{AA}^T).
\end{split}
\end{equation}
\end{proof}

With the above lemmas, $\phi\left(\left(\Pi_{i}\mathbf{U_iJ_i})(\Pi_{i}\mathbf{U_iJ_i})^T\right)^k\right)$ can be derived with Mathematical Induction:

\textcircled{1} For $\mathbf{U_LJ_L}\in\mathbb{R}^{m_L\times m_{L-1}}, \mathbf{U_{L-1}J_{L-1}} \in \mathbb{R}^{m_{L-1}\times m_{L-2}}$ where $\mathbf{U_i}$ is haar unitary matrix independent from $\mathbf{J_{L}}$, with Lemma \ref{lemma:unitary_invariant-free}, $\mathbf{J_L}^T\mathbf{J_L}$ is freely independent with $\mathbf{\mathbf{U_{L-1}J_{L-1}J_{L-1}}^T\mathbf{U_{L-1}}^T}$. As a result, with Lemma \ref{lemma:cyclic_variant}, we have
\begin{equation}
\begin{split}
    & \phi\left((\mathbf{U_LJ_LU_{L-1}J_{L-1}})(\mathbf{U_LJ_LU_{L-1}J_{L-1}})^T\right) =\\ 
    & ~~~~~~~~~~~~~~~~~~~~~~~~~~~~~~~~~~~~~~~~~~~~~~~~~\phi(\mathbf{J_LJ_L}^T)\phi(\mathbf{J_{L-1}J_{L-1}}^T),\\
    & \frac{\varphi((\mathbf{U_LJ_LU_{L-1}J_{L-1}})(\mathbf{U_LJ_LU_{L-1}J_{L-1}})^T)}{\phi^2\left((\mathbf{U_LJ_LU_{L-1}J_{L-1}})(\mathbf{U_LJ_LU_{L-1}J_{L-1}})^T\right)} =\\
    & ~~~~~~~~~~~~~~~~~~~~\frac{m_L}{m_{L-1}}\frac{\varphi(\mathbf{J_{L-1}J_{L-1}}^T)}{\phi^2(\mathbf{J_{L-1}J_{L-1}}^T)} + \frac{m_L}{m_L}\frac{\varphi(\mathbf{J_LJ_L}^T)}{\phi^2(\mathbf{J_LJ_L}^T)}.
\end{split}
\end{equation}

\textcircled{2} Assuming that for $\mathbf{J_i}\in \mathbb{R}^{m_{i}\times m_{i-1}}$, we have
\begin{equation}
\begin{split}
    & \phi\left((\Pi_{i=L}^n\mathbf{U_iJ_i})(\Pi_{i=L}^n\mathbf{U_iJ_i})^T\right) = \Pi_{i=L}^n\phi(\mathbf{J_iJ_i}^T),\\
    & \frac{\varphi((\Pi_{i=L}^n\mathbf{U_iJ_i})(\Pi_{i=L}^n\mathbf{U_iJ_i})^T)}{\phi^2\left((\Pi_{i=L}^n\mathbf{U_iJ_i})(\Pi_{i=L}^n\mathbf{U_iJ_i})^T\right)} = \sum_{i=1}^n\frac{m_{L}}{m_{i}}\frac{\varphi(\mathbf{J_iJ_i}^T)}{\phi^2(\mathbf{J_iJ_i}^T)}.
\end{split}
\end{equation}
For $\Pi_{i=L}^n\mathbf{U_iJ_i} \in \mathbb{R}^{m_L\times m_{n-1}}$ and $\mathbf{U_{n-1}J_{n-1}} \in \mathbb{R}^{m_{n-1}\times m_{n-2}}$, as $\mathbf{U_{n-1}}$ is a haar unitary matrix independent with $(\Pi_{i=L}^n\mathbf{U_iJ_i})^T(\Pi_{i=L}^n\mathbf{U_iJ_i})$, with Lemma \ref{lemma:cyclic_variant}, we have
\begin{equation}
    \begin{split}
        & \phi\left((\Pi_{i=L}^{n-1}\mathbf{U_iJ_i})(\Pi_{i=L}^{n-1}\mathbf{U_iJ_i})^T\right) = \Pi_{i=L}^{n-1}\phi(\mathbf{J_iJ_i}^T),\\
        & \frac{\varphi((\Pi_{i=L}^{n-1}\mathbf{U_iJ_i})(\Pi_{i=L}^{n-1}\mathbf{U_iJ_i})^T)}{\phi^2\left((\Pi_{i=L}^{n-1}\mathbf{U_iJ_i})(\Pi_{i=L}^{n-1}\mathbf{U_iJ_i})^T\right)} = \sum_{i=L}^{n-1}\frac{m_{L}}{m_{i}}\frac{\varphi(\mathbf{J_iJ_i}^T)}{\phi^2(\mathbf{J_iJ_i}^T)}.
    \end{split}
\end{equation}

At last, \textcircled{1} and \textcircled{2} yield
\begin{equation}
\begin{split}
    & \phi\left((\Pi_{i=L}^1\mathbf{U_iJ_i})(\Pi_{i=L}^1\mathbf{U_iJ_i})^T\right) = \Pi_{i}\phi(\mathbf{J_iJ_i}^T),\\
    & \varphi((\Pi_{i=L}^1\mathbf{U_iJ_i})(\Pi_{i=L}^1\mathbf{U_iJ_i})^T)=\\
    & \phi^2\left((\Pi_{i=L}^1\mathbf{U_iJ_i})(\Pi_{i=L}^1\mathbf{U_iJ_i})^T\right) \sum_{i}\frac{m_{L}}{m_{i}}\frac{\varphi(\mathbf{J_iJ_i}^T)}{\phi^2(\mathbf{J_iJ_i}^T)}.
\end{split}
\end{equation}

According to the prerequisite that $(\Pi_{i=L}^1\mathbf{J_i})(\Pi_{i=L}^1\mathbf{J_i})^T$ is at least $k^{th}$ moment unitarily invariant where $k\in\{1,2\}$, we have $k\in\{1, 2\}$, $\phi\left(\left(\Pi_{i=L}^1\mathbf{J_i})(\Pi_{i=L}^1\mathbf{J_i})^T\right)^k\right)=\phi\left(\left(\Pi_{i=L}^1\mathbf{U_iJ_i})(\Pi_{i=L}^1\mathbf{U_iJ_i})^T\right)^k\right)$, and the theory is proved.

%% file: chapters/proofs/proof_add.tex
\textbf{Theorem \ref{theorem:Addition}. }\textit{\textbf{(Addition)} Given $\mathbf{J} := \sum_{i}\mathbf{J_i}$, where $\{\mathbf{J_i}\}$ is a series of independent random matrices.}

\textit{If at most one matrix in $\{\mathbf{J_i}\}$ is not a central matrix (Definition \ref{def:central_matrix}), we have}
\begin{equation}
    \phi\left(\mathbf{JJ}^T\right)=\sum_i \phi\left(\mathbf{J_iJ_i}^T\right).
\label{equ:prof:addtion_exp}
\end{equation}

\textit{If $(\sum_{i}\mathbf{J_i})(\sum_{i}\mathbf{J_i})^T$ is at least $2^{nd}$ moment unitarily invariant (Definition \ref{def:moment_unitary_invariant}), and $\mathbf{U_iJ_i}$ is R-diagonal (Definition \ref{def:R-diagonal}), we have}
\begin{equation}
    \varphi\left(\mathbf{JJ}^T\right) = \phi^2\left(\mathbf{JJ}^T\right) + \sum_i \varphi\left(\mathbf{J_iJ_i}^T\right) - \phi^2\left(\mathbf{J_iJ_i}^T\right).
\label{equ:prof:addtion_var}
\end{equation}
\rule[0pt]{0.48\textwidth}{0.05em}

For Equation \eqref{equ:prof:addtion_exp}, we assume that $\mathbf{J_i}\in\mathbb{R}^{m_i\times n_i}$. Since we have
\begin{equation}
    \phi((\sum_{i=1}^L \mathbf{J_i})(\sum_{i=1}^L \mathbf{J_i})^T) = \sum_i\sum_j\phi(\mathbf{J_i}\mathbf{J_j}^T),
\end{equation}
and when $i\neq j$, $\mathbf{J_i}$ is independent of $\mathbf{J_j}$, we have
\begin{equation}
    \phi(\mathbf{J_iJ_j}^T) = \frac{1}{m_i}\sum_p\sum_qE\left[\left[\mathbf{J_i}\right]_{p,q}\right]E\left[\left[\mathbf{J_j}\right]_{p,q}\right].
\end{equation}
As at most one matrix in $\{\mathbf{J_i}\}$ is not a central matrix, $\forall i,j\neq i, E\left[\left[\mathbf{J_i}\right]_{p,q}\right]E\left[\left[\mathbf{J_j}\right]_{p,q}\right]=0$, and we have
\begin{equation}
    \phi((\sum_{i} \mathbf{J_i})(\sum_{i} \mathbf{J_i})^T)=\sum_i\phi(\mathbf{J_i}\mathbf{J_i}^T).
\end{equation}

For Equation \eqref{equ:prof:addtion_var}, we firstly propose the following lemma:
\begin{lemma}
Let $[f(z)]_{@k}$ denote the truncated series of the power series expanded from $f(z)$: i.e. given $f(z) = f_0 + f_1z+...$, we have $[f(z)]_{@2} = f_0 + f_1z$. 

Let $\mathbf{A}$, $\mathbf{B}$ be two hermitian matrices, if $\forall i\in\{1,..,k\}$, $\phi(\mathbf{A}^{i}) = \phi(\mathbf{B}^{i})$, then we have
\begin{equation}
    \left[G_{\mathbf{A}}(z)\right]_{@k+1} =  \left[G_{\mathbf{B}}(z)\right]_{@k+1}.
\end{equation}
\label{lemma:equivalent_G_transform}
\end{lemma}

\begin{proof}
Consider a hermitian matrix $\mathbf{X}$ whose spectrum density is defined as $\rho_{X}(\lambda)$, following Definition \ref{def:stieltjes_transform}, we have
\begin{equation}
    G_{\mathbf{X}}(z) = \sum_{k=0}^{\infty}\frac{\alpha_k(\lambda_{\mathbf{X}})}{z^{k+1}}.
\end{equation}
As $[G_{\mathbf{X}}]_{@k+1}$ is only determined by $\alpha_1(\lambda_{\mathbf{X}}),...,\alpha_k(\lambda_{\mathbf{X}})$ and with Lemma \ref{lemma:bridge_eig_phi}, we have $\alpha_i(\lambda_{\mathbf{X}})=\phi(\mathbf{X}^i)$, therefore we get
\begin{equation}
    \left[G_{\mathbf{A}}(z)\right]_{@k+1} =  \left[G_{\mathbf{B}}(z)\right]_{@k+1}.
\end{equation}
\end{proof}

As argued in Ling \& Qiu (2018) \cite{ling2018spectrum}, the elements of the expansion of $\left(\left(\sum_i\mathbf{J_i}\right)\left(\sum_i\mathbf{J_i}\right)^T\right)^k$ are not freely independent and when $i\ne j$, the resulting $\mathbf{J_iJ_j}^T$ is not hermitian, which can be overcome by the non-hermitian random matrix theory\cite{cakmak2012non,ling2018spectrum}. 

For a hermitian matrix $\widetilde{\mathbf{X}}$ whose empirical eigenvalue distribution is 
\begin{equation}
    \rho_{\widetilde{\mathbf{X}}}(\lambda) = \frac{\rho_{\sqrt{\mathbf{XX}^T}}(\lambda) + \rho_{\sqrt{\mathbf{XX}^T}}(-\lambda)}{2}.
\end{equation}
we have the lemmas below:

\begin{lemma}
(Lemma 9 in Cakmak (2012) \cite{cakmak2012non}) Let $\mathbf{X}$ be a rectangular non-Hermitian random matrix in general. Then we have
\begin{equation}
    G_{\widetilde{\mathbf{X}}}(z)=zG_{\mathbf{XX}^T}(z^2).
\label{equ:lemmaxxxt}
\end{equation}
where $G$ denotes the Stieltjes transform (Definition \ref{def:stieltjes_transform}).
\label{lemma:lemmaxxxt}
\end{lemma}

\begin{lemma}
(Theorem 25 in Cakmak (2012) \cite{cakmak2012non}) Let the asymptotically free random matrices $\mathbf{A}$ and $\mathbf{B}$ be R-diagonal, Define $\mathbf{C}=\mathbf{A}+\mathbf{B}$, then we have $R_{\widetilde{\mathbf{C}}}(z) = R_{\widetilde{\mathbf{A}}}(z) + R_{\widetilde{\mathbf{B}}}(z)$
\label{lemma:r-diagonal-R-transform}
\end{lemma}

Let $\mathbf{J_i}_u := \mathbf{U_i}\mathbf{J_i}$, where $\mathbf{U_i}$ is a Haar unitary matrix free of $\mathbf{J_i}$. As the prerequisite suggests that $\mathbf{J_i}_u$ is R-diagonal, with Lemma \ref{lemma:r-diagonal-R-transform}, we have:
\begin{equation}
    R_{\widetilde{\sum_{i}\mathbf{J_i}_u}}(z) = \sum_{i}R_{\widetilde{\mathbf{J_i}_u}}(z).
\label{equ:non-hermitian_R-transform}
\end{equation}
With Equation \eqref{equ:R-transform}, we have $R_{\mathbf{X}}(G_{\mathbf{X}}(z)) = z-\frac{1}{G_{\mathbf{X}}(z)}$. By substituting this into Equation \eqref{equ:non-hermitian_R-transform}, we can get
\begin{equation}
\begin{split}
    &\sum_{i}R_{\widetilde{\mathbf{J_i}_u}}\left[G_{\widetilde{\sum_{i}\mathbf{J_i}_u}}(z)\right] = R_{\widetilde{\sum_{i}\mathbf{J_i}_u}}\left[G_{\widetilde{\sum_{i}\mathbf{J_i}_u}}(z)\right]\\
    & = z - \frac{1}{G_{\widetilde{\sum_{i}\mathbf{J_i}_u}}(z)}.
\end{split}
\label{equ:equ10_pre}
\end{equation}
Finally, by substituting Equation \eqref{equ:lemmaxxxt} into Equation \eqref{equ:equ10_pre}, we have
\begin{equation}
\begin{split}
    & \frac{1}{\sqrt{z}G_{\sum_{i}\mathbf{J_i}_u(\sum_{i}\mathbf{J_i}_u)^T}(z)} + \sum_{i}R_{\widetilde{\mathbf{J_i}_u}}\left[\sqrt{z}G_{\sum_{i}\mathbf{J_i}_u(\sum_{i}\mathbf{J_i}_u)^T}(z)\right]\\
    & = \sqrt{z}.
\end{split}
\label{equ:equ10}
\end{equation}

With the prerequisite that $(\sum_{i}\mathbf{J_i})(\sum_{i}\mathbf{J_i})^T$ is at least $k^{th}$ moment unitarily invariant, using Lemma \ref{lemma:equivalent_G_transform}, Equation \eqref{equ:equ10} can be reformulated to:
\begin{equation}
\begin{split}
    &\left[\!1\! +\! \sqrt{z}G_{\sum_{i}\!\mathbf{J_i}(\!\sum_{i}\!\mathbf{J_i}\!)^T}\!(z)\!\sum_{i}\!R_{\widetilde{\mathbf{J_i}_u}}\!\left[\!\sqrt{z}G_{\sum_{i}\mathbf{J_i}(\sum_{i}\mathbf{J_i})^T}(z)\!\right]\!\right]_{@k+1}\\
    & = \left[zG_{\sum_{i}\mathbf{J_i}(\sum_{i}\mathbf{J_i})^T}(z)\right]_{@k+1}.
\end{split}
\label{equ:equ10@k}
\end{equation}

As the Stieltjes transform can be expanded into a power series with coefficients as $G_X(z)=\sum_{q=0}^{\infty}\frac{\alpha_q(\lambda_{\mathbf{X}})}{z^{q+1}}$ where $\alpha_q(\lambda_{\mathbf{X}})=\phi(\mathbf{X}^q)$, we can solve Equation \eqref{equ:equ10} by expanding both side of the equals sign into a polynomial of z in which the coefficients of all the orders of z are equal. For the sake of simplicity, we denote $m^{(i)}_q := \alpha_q(\lambda_{\mathbf{J_iJ_i}^T})$, $m_q := \alpha_q(\lambda_{\mathbf{JJ}^T})$, $m_{qu}:= \alpha_q(\lambda_{\mathbf{J_i}_u\mathbf{J_i}_u^T})$, and our first step is deriving the expansion of $R_{\widetilde{\mathbf{J_i}}}(z)$.

Firstly, with Lemma \ref{lemma:equivalent_G_transform}, we can obtain the first 2k+1 terms of $G_{\mathbf{J_i}_u\mathbf{J_i}_u^T}(z^2)$ (That's all we need) from $G_{\mathbf{J_i}\mathbf{J_i}^T}(z^2)$. Secondly, we will derives $G_{\widetilde{\mathbf{J_i}_u}}(z)$ from $G_{\mathbf{J_i}_u\mathbf{J_i}_u^T}(z^2)$ with Lemma \ref{lemma:lemmaxxxt}. Then, $G^{-1}_{\widetilde{\mathbf{J_i}_u}}(z)$ can be derived from $G_{\widetilde{\mathbf{J_i}_u}}(z)$ with the Lagrange inversion theorem. Last, as we have $R_{\mathbf{X}}(z) + \frac{1}{z} = G_{\mathbf{X}}^{-1}(z)$, $R_{\widetilde{\mathbf{J_i}_u}}(z)$ can be easily derived from $G_{\widetilde{\mathbf{J_i}_u}}^{-1}(z)$.

According to Lemma \ref{lemma:lemmaxxxt} and \ref{lemma:equivalent_G_transform}, we have:
\begin{equation}
\begin{split}
    & G_{\widetilde{\mathbf{J_i}_u}}(z)=zG_{\mathbf{J_i}_u\mathbf{J_i}_u^T}(z^2)=\sum_{q=0}^\infty \frac{m^{(i)}_q}{z^{2q+1}} =\\
    & ~~~~~~~~~~~~~~~~~~~~~~~~~\frac{1}{z} + \frac{0}{z^2} + \frac{m^{(i)}_1}{z^3}+\frac{0}{z^4} + \frac{m^{(i)}_2}{z^5}+...
\end{split}
\label{equ:GJ(z)}
\end{equation}

We can view Equation \eqref{equ:GJ(z)} as a formal power series: $f(\frac{1}{z}):=\sum_{k=1}^{\infty} \frac{f_k}{k!}(\frac{1}{z})^k$, where we have $f_0=0,f_1=1,f_2=0,f_3=m^{(i)}_13!,f_4=0,f_5=m^{(i)}_25!, f_6=0...$. Because of $f_0=0,f_1\ne 0$, $G^{-1}_{\widetilde{\mathbf{J_i}_u}}(z)$ can be obtained with the Lagrange inversion theorem: assuming that g is the inverse function of f: $g(f(z))=z$, since we have  $G_{\widetilde{\mathbf{J_i}_u}}(z)=f(\frac{1}{z})$, $z = G^{-1}_{\widetilde{\mathbf{J_i}_u}}(f(\frac{1}{z}))=\frac{1}{g(f(\frac{1}{z}))}$, we have $G^{-1}_{\widetilde{\mathbf{J_i}_u}}(z) = \frac{1}{g(z)}$. With the Lagrange inversion theorem, the expansion of $g(z)$ is
\begin{equation}
    \begin{split}
        & g(z) = \sum_{k=0}^\infty g^{(i)}_q\frac{z^q}{q!},~~g^{(i)}_0=0,~~ g^{(i)}_1=\frac{1}{f_1}=1,\\
        & g^{(i)}_{n\ge2}=\frac{1}{f_1^n}\sum_{q=1}^{n-1}(-1)^qn^{(q)}\mathbf{B}_{n-1,q}(\hat{f}_1,\hat{f}_2,...,\hat{f}_{n-q}).
    \end{split}
\label{equ:expansion:g(z)}
\end{equation}
where $\mathbf{B}$ denotes the Bell polynomials, $\hat{f}_q = \frac{f_{q+1}}{(q+1)f_1}$, $n^{(q)}=n(n+1)...(n+q-1)$. With Equation \eqref{equ:expansion:g(z)}, we have
\begin{equation}
\begin{split}
    & G^{-1}_{\widetilde{\mathbf{J_i}_u}}(z) = \frac{1}{z + \sum_{q=2}^{\infty}g^{(i)}_q\frac{z^q}{q!}} = \frac{1}{z}(1 - \frac{g^{(i)}_2}{2}z + ...)\\
    & := \frac{1}{z} + \sum_{q=0}^{\infty}{h^{(i)}_q}\frac{z^q}{q!}.
\end{split}
\end{equation}
Thus, we have
\begin{equation}
\begin{split}
    & R_{\widetilde{\mathbf{J_i}_u}}(z) = G^{-1}_{\widetilde{\mathbf{J_i}_u}}(z)-\frac{1}{z} = \sum_{k=0}^{\infty}{h^{(i)}_k}\frac{z^k}{k!}.
\end{split}
\label{equ:RJ(z):expansion}
\end{equation}
Then we substitute Equation \eqref{equ:RJ(z):expansion} into Equation \eqref{equ:equ10@k}, which yields
\begin{equation}
\begin{split}
    & \left[z\sum_{q=0}^{\infty}\frac{m_q}{z^{q+1}}\right]_{@k+1} =\\
    & \left[\sum_i\left[z^{\frac{1}{2}}\sum_{q=0}^{\infty}\frac{m_q}{z^{q+1}}\sum_{p=0}^{\infty}\frac{h^{(i)}_p}{p!}\left(z^{\frac{1}{2}}\sum_{q=0}^{\infty}\frac{m_q}{z^{q+1}}\right)^p\right] + 1\right]_{@k+1}.
\end{split}
\end{equation}

Let's consider the coefficient of $\frac{1}{z}$ first:
\begin{equation}
    m_1 = \sum_i h_1^{(i)}.
\end{equation}
Since $h_1^{(i)} = \frac{1}{2}\frac{d^2}{dz^2}zG^{-1}_{\widetilde{\mathbf{J_i}_u}}(z)|_{z=0} = \frac{g^{(i)}_2}{4}-\frac{g^{(i)}_3}{6}$, with Equation \eqref{equ:expansion:g(z)}, we have
\begin{equation}
g_2^{(i)} = 0,~~g_3^{(i)}=-6m^{(i)}_1.
\end{equation}
Thus we have $m_1=\sum_i m_1^{(i)}$.

Then, we consider the coefficient of $\frac{1}{z^2}$:
\begin{equation}
    m_2 = \sum_i (2m_1h^{(i)}_1+\frac{h^{(i)}_3}{6}).
\end{equation}
Since $h_3^{(i)} = \frac{3!}{4!}\frac{d^4}{dz^4}zG^{-1}_{\widetilde{\mathbf{J_i}_u}}(z)|_{z=0}$. This time, $g_2^{(i)}$ to $g_5^{(i)}$ are used, and we derive $g_4^{(i)},g_5^{(i)}$ first.
\begin{equation}
    g_4^{(i)}=0,~~g_5^{(i)}=360(m^{(i)}_1)^2-120m^{(i)}_2.
\end{equation}
Then we calculate $h_3^{(i)}$:
\begin{equation}
    h_3^{(i)} = \frac{3!}{4!}\frac{d^4}{dz^4}zG^{-1}_{\widetilde{\mathbf{J_i}_u}}(z)|_{z=0} = \frac{(g_3^{(i)})^2}{6}-\frac{g_5^{(i)}}{20}.
\end{equation}
Finally we have
\begin{equation}
    \begin{split}
        & \phi(\mathbf{JJ}^T)=m_1=\sum_i \phi(\mathbf{J_iJ_i}^T),\\
        & \varphi(\mathbf{JJ}^T) = m_2 - m_1^2 = m_1^2 + \sum_i \varphi(\mathbf{J_iJ_i}^T) - \phi^2(\mathbf{J_iJ_i}^T).
    \end{split}
\label{equ:the_forge_add}
\end{equation}

%% file: chapters/prerequisites/mul.tex
\textbf{Proposition \ref{prop:1st_moment_unitary_invariant_exp_diag}. }\textit{
$(\Pi_{i=L}^i \mathbf{J_i})(\Pi_{i=L}^1 \mathbf{J_i})^T$ is at least $1^{st}$ moment unitary invariant if: \textcircled{1} $\forall i,j\neq i$, $\mathbf{J_i}$ is independent with $\mathbf{J_j}$; \textcircled{2} $\forall i \in [2, L]$, $\mathbf{J_i}$ is an expectant orthogonal matrix.
}
\rule[0pt]{0.48\textwidth}{0.05em}

Firstly, we prove the following lemma:

\begin{lemma}
Let $\mathbf{D}\in \mathbb{R}^{k\times n}$ and $\mathbf{W}\in\mathbb{R}^{n\times m}$ be two independent random matrices. We further let $\widetilde{\mathbf{W}}\in \mathbb{R}^{n\times m}$ be a random matrix that satisfies: $\phi(\widetilde{\mathbf{W}}\widetilde{\mathbf{W}}^T) = \phi(\mathbf{WW}^T)$ and $\mathbf{D}$ be an expectant orthogonal matrix. Then we have
\begin{equation}
    \phi\left(\left(\mathbf{DW}\right)\left(\mathbf{DW}\right)^T\right) = \phi\left(\left(\mathbf{UD}\widetilde{\mathbf{W}}\right)\left(\mathbf{UD}\widetilde{\mathbf{W}}\right)^T\right).
\label{equ:exp_diag_lemma_equ}
\end{equation}
\label{lemma:exp_diag_lemma}
\end{lemma}

\begin{proof}
Firstly, we have
\begin{equation}
\begin{split}
    & \phi\left(\left(\mathbf{UD}\widetilde{\mathbf{W}}\right)\left(\mathbf{UD}\widetilde{\mathbf{W}}\right)^T\right) = \phi\left(\mathbf{D}\widetilde{\mathbf{W}}\widetilde{\mathbf{W}}^T\mathbf{D}^T\right)\\
    &  = \frac{n}{k}\phi\left(\mathbf{D}^T\mathbf{D}\widetilde{\mathbf{W}}\widetilde{\mathbf{W}}^T\right).
\end{split}
\end{equation}
For the sake of simplicity, we denote $[\mathbf{WW}^T]_{i,j} := w_{i,j}$, $[\widetilde{\mathbf{W}}\widetilde{\mathbf{W}}^T]_{i,j} := \widetilde{w_{i,j}}$, $[\mathbf{D}^T\mathbf{D}]_{i,j} := d_{i,j}$. 
\begin{equation}
\begin{split}
    & \phi\left(\mathbf{D}^T\mathbf{D}\widetilde{\mathbf{W}}\widetilde{\mathbf{W}}^T\right) = E[\frac{1}{n}\sum_i\sum_pd_{i,p}\widetilde{w_{p,i}}] =\\
    & \frac{1}{n}\sum_i\sum_pE[d_{i,p}]E[\widetilde{w_{p,i}}].
\end{split}
\end{equation}
Since $\mathbf{D}$ is an expectant orthogonal matrix, we have
\begin{equation}
\begin{split}
    &  \frac{1}{n}\sum_i\sum_pE[d_{i,p}]E[\widetilde{w_{p,i}}] = E[d_{diag}]\frac{1}{n}\sum_iE[\widetilde{w_{i,i}}]\\
    & \forall i,~~E[d_{i,i}]=E[d_{diag}].
\end{split}
\label{equ:equ94}
\end{equation}
We get $\phi\left(\left(\mathbf{UD}\widetilde{\mathbf{W}}\right)\left(\mathbf{UD}\widetilde{\mathbf{W}}\right)^T\right) = E[d_{diag}]\frac{1}{k}\sum_iE[\widetilde{w_{i,i}}]$. With the same process, it is easy to prove that $\phi\left(\left(\mathbf{D}\mathbf{W}\right)\left(\mathbf{D}\mathbf{W}\right)^T\right) = E[d_{diag}]\frac{1}{k}\sum_iE[w_{i,i}]$. Since we have $\phi(\widetilde{\mathbf{W}}\widetilde{\mathbf{W}}^T) = \phi(\mathbf{WW}^T)$, $\sum_iE[\widetilde{w_{i,i}}]=\sum_iE[w_{i,i}]$, and we have $\phi\left(\left(\mathbf{DW}\right)\left(\mathbf{DW}\right)^T\right) = \phi\left(\left(\mathbf{UD}\widetilde{\mathbf{W}}\right)\left(\mathbf{UD}\widetilde{\mathbf{W}}\right)^T\right)$.
\end{proof}

With Lemma \ref{lemma:exp_diag_lemma}, we use Mathematical Induction to prove Proposition \ref{prop:1st_moment_unitary_invariant_exp_diag}.

\textcircled{1} As $\mathbf{J_{2}}$ is an expectant orthogonal matrix and independent with $\mathbf{J_1}$, let $\mathbf{U_1}$ be a haar unitary matrix independent with $\forall i, \mathbf{J_i}$. As $\phi(\mathbf{U_1J_1J_1}^T\mathbf{U_1}^T) = \phi(\mathbf{J_1J_1}^T)$, we have
\begin{equation}
    \phi\left(\left(\mathbf{J_{2}J_{1}}\right)\left(\mathbf{J_{2}J_1}\right)^T\right) = \phi\left(\left(\mathbf{U_{2}J_{2}U_1J_{1}}\right)\left(\mathbf{U_{2}J_{2}U_1J_{1}}\right)^T\right).
\end{equation}

\textcircled{2} Assuming that we have
\begin{equation}
    \phi\left(\left(\Pi_{i=l}^{1}\mathbf{J_{i}}\right)\left(\Pi_{i=l}^{1}\mathbf{J_{i}}\right)^T\right) = \phi\left(\left(\Pi_{i=l}^{1}\mathbf{U_iJ_{i}}\right)\left(\Pi_{i=l}^1\mathbf{U_iJ_{i}}\right)^T\right).
\end{equation}

As $\mathbf{J_{l+1}}$ is an expectant orthogonal matrix, and independent with $\Pi_{i=l}^1\mathbf{J_{i}}$ and $\Pi_{i=l}^1\mathbf{U_iJ_{i}}$, we have
\begin{equation}
\begin{split}
    & \phi\left(\left(\Pi_{i=l+1}^1\mathbf{J_{i}}\right)\left(\Pi_{i=l+1}^1\mathbf{J_{i}}\right)^T\right) =\\
    & \phi\left(\left(\Pi_{i=l+1}^1\mathbf{U_iJ_{i}}\right)\left(\Pi_{i=l+1}^1\mathbf{U_iJ_{i}}\right)^T\right).
\end{split}
\end{equation}

With \textcircled{1} and \textcircled{2}, we have
\begin{equation}
    \phi\left(\left(\Pi_{i=L}^1\mathbf{J_{i}}\right)\left(\Pi_{i=L}^1\mathbf{J_{i}}\right)^T\right)\! =\! \phi\left(\left(\Pi_{i=L}^1\mathbf{U_iJ_{i}}\right)\left(\Pi_{i=L}^1\mathbf{U_iJ_{i}}\right)^T\right).
\end{equation}
and Proposition \ref{prop:1st_moment_unitary_invariant_exp_diag} is proved.

%% file: chapters/proofs/proof_properties_of_eom_cm.tex
\textbf{Proposition \ref{prop:properties_of_eom_cm}. }\textit{\textbf{(Properties of expectant orthogonal matrices and central matrices)}
\begin{itemize}
\item Let $\{\mathbf{J_i}\}$ be a series independent expectant orthogonal matrices, then $\Pi_i\mathbf{J_i}$ is also an expectant orthogonal matrices.
    \item Let $\mathbf{J_i}$ be a central matrix, for any random matrices $\mathbf{A}$ independent with $\mathbf{J_i}$ with proper size, $\mathbf{J_iA},\mathbf{AJ_i}$ are also central matrices.
\end{itemize}}
\rule[0pt]{0.46\textwidth}{0.05em}

\textcircled{1} Let $\{\mathbf{J_i}\}$ be a series independent expectant orthogonal matrices, then $\Pi_i\mathbf{J_i}$ is also an expectant orthogonal matrix.

\begin{proof}
Firstly, we consider the multiplication between two independent expectant orthogonal matrices $\mathbf{W}\mathbf{D}$. We denote $[\mathbf{W}^T\mathbf{W}]_{i,j}=\omega_{i,j}$, $[\mathbf{D}^T\mathbf{D}]_{i,j}=\delta_{i,j}$, $[\mathbf{D}]_{i,j}=d_{i,j}$.

\begin{equation}
    E\left[\left[\mathbf{D}^T\mathbf{W}^T\mathbf{WD}\right]_{i,j}\right]=E\left[\sum_q\sum_pd_{p,i}\omega_{p,q}d_{q,j}\right].
\end{equation}
Since $\mathbf{W}$ is an expectant orthogonal matrix, we have $E[\omega_{p,q\neq p}=0]$, $E[\omega_{p,p}]:=\omega$, and we have
\begin{equation}
    E\left[\left[\mathbf{D}^T\mathbf{W}^T\mathbf{WD}\right]_{i,j}\right]=\omega E\left[\sum_qd_{q,i}d_{q,j}\right] = \omega E[\delta_{i,j}].
\end{equation}
As $\mathbf{D}$ is an expectant orthogonal matrix, $E[\delta_{p,q\neq p}=0]$, $E[\delta_{p,p}]:=\delta$, finally we have
\begin{equation}
    E\left[\left[\mathbf{DWW}^T\mathbf{D}^T\right]_{i,j\neq i}\right]=0,E\left[\left[\mathbf{DWW}^T\mathbf{D}^T\right]_{i,i}\right]=\omega\delta.
\end{equation}
Thus $\mathbf{WD}$ is also an expectant orthogonal matrix. In another word, the result of multiplication between expectant orthogonal matrices is also an expectant orthogonal matrices, and this conclusion can be easily extended to the multiplication between several expectant orthogonal matrices.
\end{proof}

\textcircled{2} Let $\mathbf{J_i}$ be a central matrix, for any random matrices $\mathbf{A}$ independent with $\mathbf{J_i}$ with proper size, $\mathbf{J_iA},\mathbf{AJ_i}$ are also central matrices.
\begin{proof}
As $\mathbf{J_i}$ is independent with $\mathbf{A}$ and $\forall i,j,~~E\left[\left[\mathbf{J_i}\right]_{i,j}\right]=0$, it's obvious that 
\begin{equation}
    \forall i,j,~~E\left[\left[\mathbf{AJ_i}\right]_{i,j}\right]= E\left[\left[\mathbf{J_iA}\right]_{i,j}\right]=0.
\end{equation}
\end{proof}

%% file: chapters/proofs/proof_parts_library.tex
\subsubsection{Activation Functions}
For clarity, we denote the activation functions as $\mathbf{f}(\mathbf{x})$ and their Jacobian matrices as $\mathbf{f_x}$. Since the Jacobian matrix $\mathbf{J}$ of any element-wise activation functions is diagonal, and its diagonal entries share the same non-zero expectation, $\mathbf{J}$ is expectant orthogonal matrices (satisfies Definition \ref{def:expectattion_diagonal_matrix}) but defies Definition \ref{def:central_matrix}.

\textbf{ReLU.} ReLU is defined as: $f(x)=x$ if $x\ge0$ else $0$, thus $\mathbf{f_x}$ is a diagonal matrix whose elements are either 0 or 1. As a result, the spectral density of $\mathbf{f_x}$ is: $\rho_{\mathbf{f_x}}(z) = (1-p)\delta(z)+p\delta(z-1)$, where $p$ denotes the probability that the input data is greater than 0, and we have $\rho_{\mathbf{f_x}}(z) = \rho_{\mathbf{f_xf_x}^T}(z)$. Thus, we have
\begin{equation}
    \begin{split}
        &\phi(\mathbf{f_xf_x}^T) = \int_{\mathbb{R}}z\left((1-p)\delta(z)+p\delta(z-1)\right)dz = p,\\
        &\varphi(\mathbf{f_xf_x}^T) = \int_{\mathbb{R}}z^2\left((1-p)\delta(z)+p\delta(z-1)\right)dz - \phi^2(\mathbf{f_xf_x}^T)\\
        &~~~~~~~~~~~~~~ = p - p^2.
    \end{split}
\label{equ:eig_relu}
\end{equation}

\textbf{Leaky ReLU. }Leaky ReLU is defined as: $f(x)=x$ if $x\ge0$ else $\gamma x$, $\gamma$ is called negative slope coefficient. Similar to ReLU, its Jacobian matrix is diagonal and its diagonal entries are either $1$ or $\gamma$. Therefore, the spectral density is $(1-p)\delta(z-\gamma^2) + p\delta(z-1)$, and we have
\begin{equation}
    \begin{split}
        &\phi(\mathbf{f_xf_x}^T) = \int_{\mathbb{R}}z\left((1-p)\delta(z-\gamma^2)+p\delta(z-1)\right)dz\\
        &~~~~~~~~~~~~~~ = p + \gamma^2(1-p),\\
        &\varphi(\mathbf{f_xf_x}^T)\! =\! \int_{\mathbb{R}}z^2\!\left((1\!-\!p)\delta(z\!-\!\gamma^2)\!+\!p\delta(z\!-\!1)\right)dz\! -\! \phi^2(\mathbf{f_xf_x}^T)\\
        & = \gamma^4(1-p)+p - (p+\gamma^2(1-p))^2.
    \end{split}
\label{equ:eig_prelu}
\end{equation}

\textbf{Tanh. }Tanh is defined as: $f(x)=\frac{2}{1+e^{-2x}}-1$, and its derivative is $f'(x)=1 - f^2(x)$. Analysis of tanh is more challenging due to its complex nonliearity. However, as illustrated in Fig. \ref{fig:tanh}, $\forall x\in\mathbb{R},|tanh(x)|/|x|<1$, 
\begin{figure}[ht]
\centering
\includegraphics[width=0.48\textwidth]{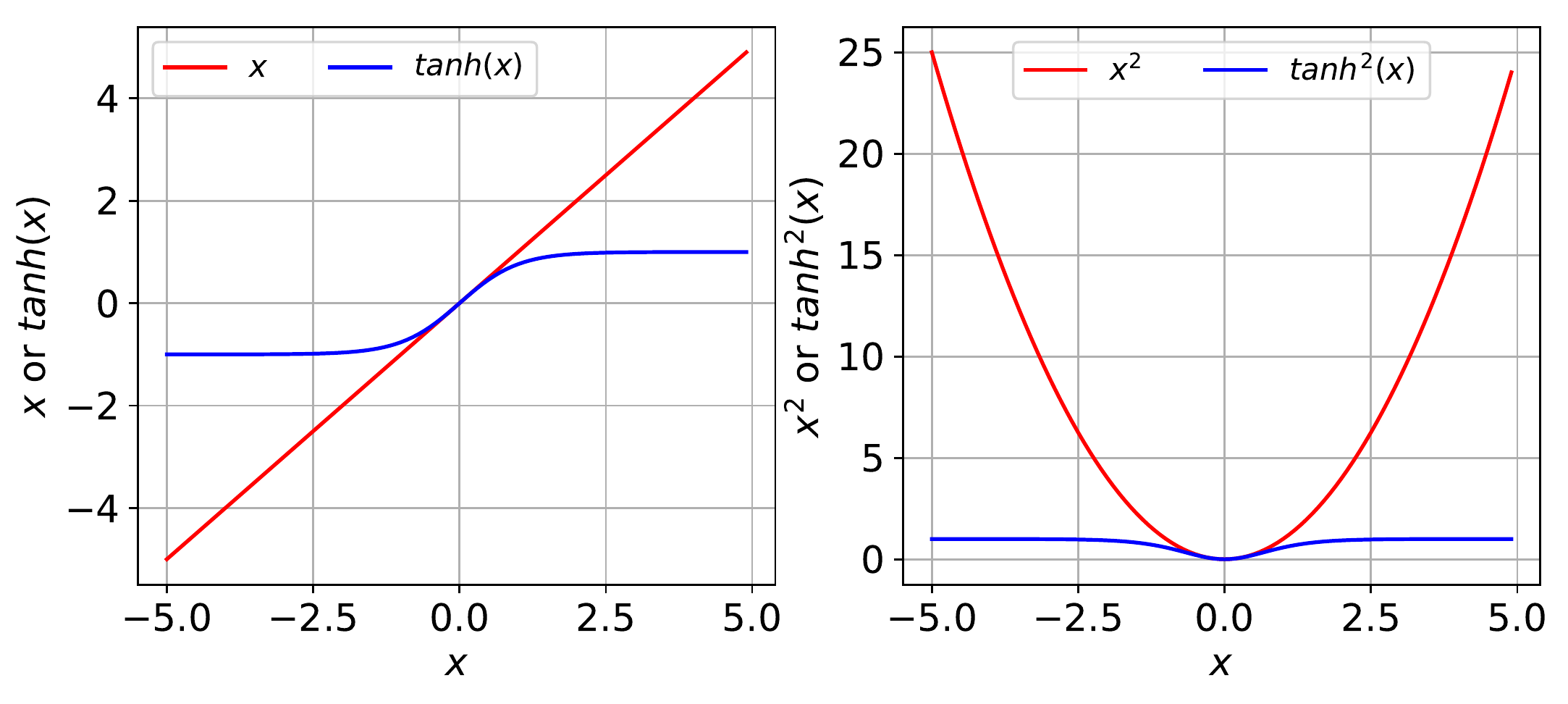}
\caption{Tanh decays the $2^{nd}$ moment of forward signal to around 0.}
\label{fig:tanh}
\end{figure}
therefore, given enough layers, most of activations will concentrate around 0, and we can simplify the tanh with Taylor series around $0$: $f(x)\approx f(0)+f'(0)x=x$, therefore $\mathbf{f_x}$ is approximately an identity matrix, for whom we have $\phi(\mathbf{f_xf_x}^T)=1$, $\varphi(\mathbf{f_xf_x}^T)=0$.

\subsubsection{Linear Transforms}\label{proof:linear_transform}

\textbf{Gaussian Fully-connected Layer. ($\mathbf{u}:=\mathbf{K}_t\mathbf{y}_t$)} $\mathbf{u}(t,\mathbf{z})$ is a linear model, thus $\mathbf{u_y} = \mathbf{K}_t \in \mathbb{R}^{m\times n}$. Assuming that $[\mathbf{K}_t]_{i,j}\sim N(\mu, \sigma^2)$ and i.i.d, we define $\hat{\mathbf{K}}_t := (\mathbf{K}_t - \mu)/\sigma$. Then, we have
\begin{equation}
\begin{split}
    &tr(\mathbf{K}_t\mathbf{K}_t^T) = \sigma^2tr(\hat{\mathbf{K}}_t\hat{\mathbf{K}}_t^T) + \mu^2,\\
    &tr\left((\mathbf{K}_t\mathbf{K}_t^T)^2\right) = \sigma^4tr\left((\mathbf{K}_t\mathbf{K}_t^T)^2\right) + 6n\mu^2\sigma^2tr\left(\mathbf{K}_t\mathbf{K}_t^T\right)\\
    & +mn^2\mu^4.
\end{split}
\label{equ:target_fc}
\end{equation}
We further define that $\mathbf{A}_t=\frac{1}{m}\hat{\mathbf{K}}_t\hat{\mathbf{K}}_t^T$, then $\mathbf{A}_t$ is a \textit{Wishart random matrix}. According to Mingo \& Speicher (2017) \cite{mingo2017free}, we have
\begin{equation}
    \lim_{m,n\rightarrow \infty,~n/m\rightarrow c}tr(\mathbf{A}_t^k) = \sum_{\pi\in NC(k)}c^{\#(\pi)},
\end{equation}
where $NC(k)$ is a set of non-crossing partition of $[k]=\{1, 2, 3,...,k\}$ and $\#(\pi)$ denotes the number of blocks of $\pi$. As we are interested in the $1^{st}$ and $2^{nd}$ moment, when $k=1$, in the limit of  $m,n\rightarrow \infty,~n/m\rightarrow c$, $NC(1) = \{\{1\}\} $ thus $\#(\pi)=1$; when $k=2$, we have $NC(2) = \{\{1, 2\}, \{\{1\},\{2\}\}\}$, and we have
\begin{equation}
\begin{split}
    & \phi(\mathbf{K}_t\mathbf{K}_t^T) = \sigma^2 n+n\mu^2,\\
    & \varphi(\mathbf{K}_t\mathbf{K}_t^T) = \left(m^2\sigma^4(c+c^2)+6n^2\mu^2\sigma^2+mn^2\mu^4\right)\\
    & -(\sigma^2 n+n\mu^2)^2,~~c = n/m.
\end{split}
\label{equ:eig_fc}   
\end{equation}
If $\mu=0$, we have $E[[\mathbf{K_t}^T\mathbf{K_t}]_{i,j\neq i}]= 0$, thus fully-connected layers with i.i.d. zero-mean weight satisfy Definition \ref{def:expectattion_diagonal_matrix}  and \ref{def:central_matrix}. And we have
\begin{equation}
\begin{split}
    \phi(\mathbf{K}_t\mathbf{K}_t^T) = \sigma^2 n, \varphi(\mathbf{K}_t\mathbf{K}_t^T) = \sigma^4 mn.
\end{split}
\end{equation}

\textbf{2D Convolution with Gaussian Kernel.} We assume that the input data $\mathbf{Y}\in \mathbb{R}^{c_{in} \times h_{in} \times w_{in}}$, where $c_{in}, h_{in}, w_{in}$ denote the number of input channels, the height and width of the images of each channel, respectively. We further assume that the elements of the convolving kernel $\mathbf{K}\in \mathbb{R}^{c_{out}\times c_{in}\times k_h \times k_w}$ follow i.i.d $N(0, \sigma_K^2)$, $c_{out}$ denotes the number of output channels while $k_h, k_w$ represent the shape of filters. The stride of convolution is defined by $s_h, s_w$ for the two directions. We use $p_h, p_w$ to denote the zero-padding added to both side of the input respect to each direction.

The convolution can be expanded into the multiplication between $\widetilde{\mathbf{K}}\in \mathbb{R}^{c_{out}h_{out}w_{out}\times c_{in}h_{in}w_{in}}$ and the vectorized $\mathbf{Y}$: $\widetilde{\mathbf{y}}\in \mathbb{R}^{c_{in}h_{in}w_{in}}$. $\widetilde{\mathbf{K}}$ is a expanded version of $\mathbf{K}$ and $h_{out}, w_{out}$ can be calculated with:
\begin{equation}
\begin{split}
    & h_{out} = \left\lfloor \frac{1}{s_h}(h_{in}+2p_{h} - k_h) + 1\right\rfloor,\\
    & w_{out} = \left\lfloor\frac{1}{s_w}(w_{in}+2p_{w} - k_w) + 1\right\rfloor.
\end{split}
\end{equation}

Since the Jacobian matrix of a linear transform is the transform itself, for the sake of simplicity, we write $\widetilde{\mathbf{K}}$ as a block matrix where each block is denoted as $\widetilde{\mathbf{K}}_{i,j} \in \mathbb{R}^{h_{out}w_{out}\times h_{in}w_{in}}$, thus we have $\left[\widetilde{\mathbf{K}}\widetilde{\mathbf{K}}^T\right]_{i,j} = \sum_{p=1}^{c_{in}}\widetilde{\mathbf{K}}_{i,p}\widetilde{\mathbf{K}}_{j,p}^T$, and $\phi\left(\widetilde{\mathbf{K}}\widetilde{\mathbf{K}}^T\right)$ is given by
\begin{equation}
    \phi\left(\widetilde{\mathbf{K}}\widetilde{\mathbf{K}}^T\right) =  E\left[\frac{1}{c_{out}}\sum_itr\left(\left[\widetilde{\mathbf{K}}\widetilde{\mathbf{K}}^T\right]_{i,i}\right)\right].
\end{equation}

If $p_h, p_w$ are zero, the diagonal elements of $\left[\widetilde{\mathbf{K}}\widetilde{\mathbf{K}}^T\right]_{i,i}$ will follow a i.i.d. scaled chi-square distribution with degrees of freedom $k_hk_w$, and we have $\phi\left(\widetilde{\mathbf{K}}\widetilde{\mathbf{K}}^T\right) = \sigma_K^2c_{in}k_hk_w$.

\begin{figure}[ht]
\centering
\includegraphics[width=0.48\textwidth]{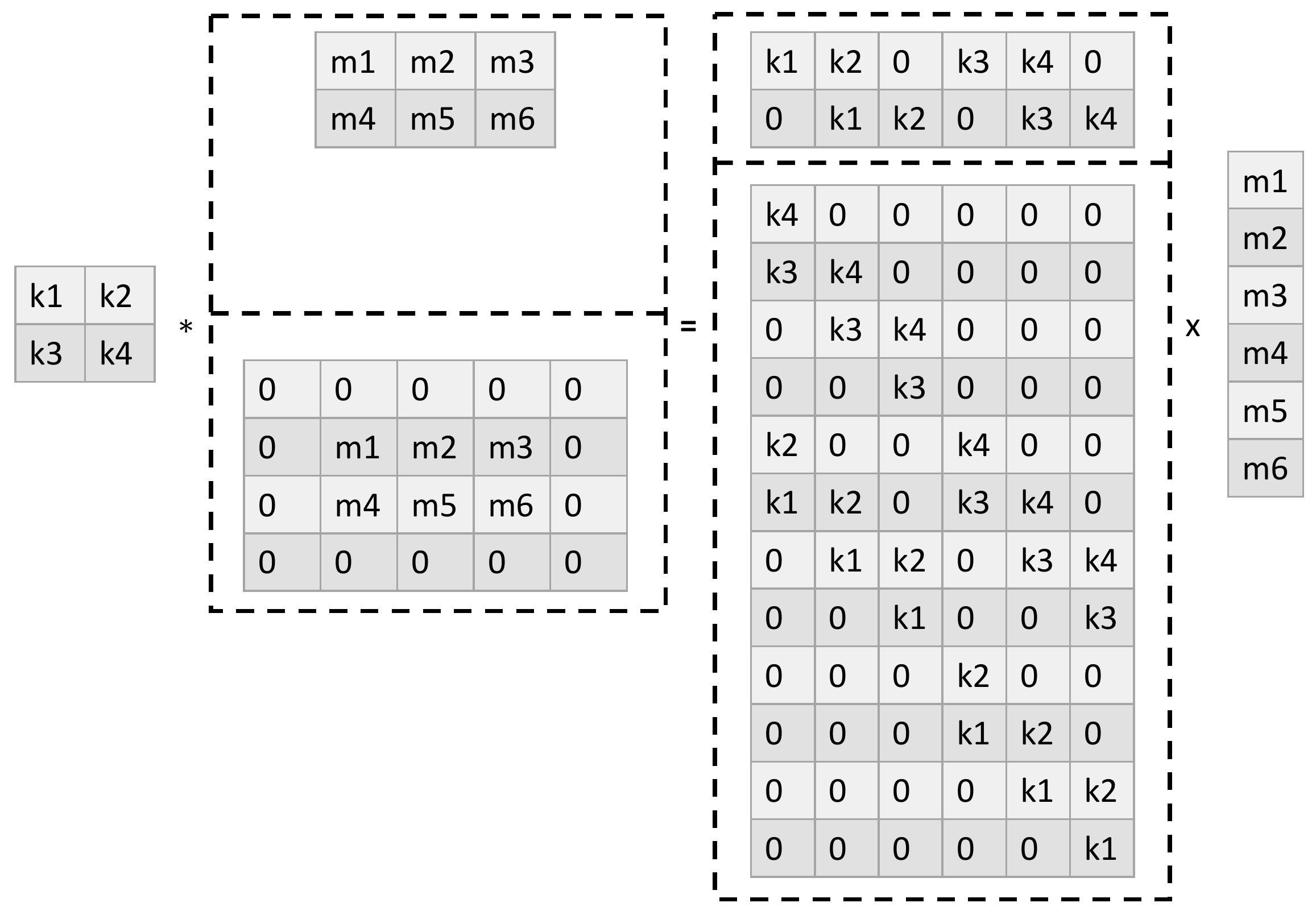}
\caption{A simple illustration for the cutting-off effect of padding under $stride=1$.}
\label{fig:padding_cutoff}
\end{figure}

If $p_h, p_w$ are not zero, it will cause the "cutting-off" effect shown in Fig. \ref{fig:padding_cutoff}: without padding, each row of the linear transform matrix contains one full set of the entries of the convolving kernel. Whereas when padding is involved, in some of the convolution operations where a part of the filter falls on the padded zeros, therefore some of the entries do not appear in the final transform matrix.

We can view the cutting-off effect of padding as a reduction to the effective size of the kernel. As a result, to precisely reflect this, we can simply replace $k_hk_w$ with $\widetilde{k_hk_w}$, where $\widetilde{k_hk_w}$ denotes the average amount of elements of the kernel appeared in each row of the transform matrix. For instance, in Fig. \ref{fig:padding_cutoff}, $k_hk_w=4$ while $\widetilde{k_hk_w}=2$. For 2D convolutions with any kernel size, strides, padding size and feature map size, Alg. \ref{Alg:conv_adjust} provide a general way to calculate the effective kernel size.

\input{algorithms/conv_adjust.tex}

Of course, this cutting-off effect is neglectable when the network is not too deep or input image size is large enough compared with the padding size, thus is neglected in previous studies, whereas when the image size is small, we find it is quiet influential to the estimation of $\phi\left(\widetilde{\mathbf{K}}\widetilde{\mathbf{K}}^T\right)$. To illustrate this point, we uniformly sample convolutional layers from the joint state $[h_{in}, w_{in}, c_{in}, c_{out}, stride, \sigma, padding, $ $kernel]$, whose elements represent height/width of input images (from $7$ to $32$), number of input/output channels (from $8$ to $32$), the convolving stride ($1$ to $3$ for each direction), the variance of the Gaussian kernel (for $0.1$ to $5$), the padding size ($1$ or $2$ for each direction), and the kernel size ($1$ to $5$ for each direction), respectively. We run the experiment for 100 times and evaluate how the theoretical $\phi(\mathbf{JJ}^T)_{t}$ approximate to the real value with the metric $\phi(\mathbf{JJ}^T)/\phi(\mathbf{JJ}^T)_{t}$. The result is shown in Fig. \ref{fig:verify_effective_kernel_size}.
\begin{figure}[htb!]
\centering
    \includegraphics[width=0.40\textwidth]{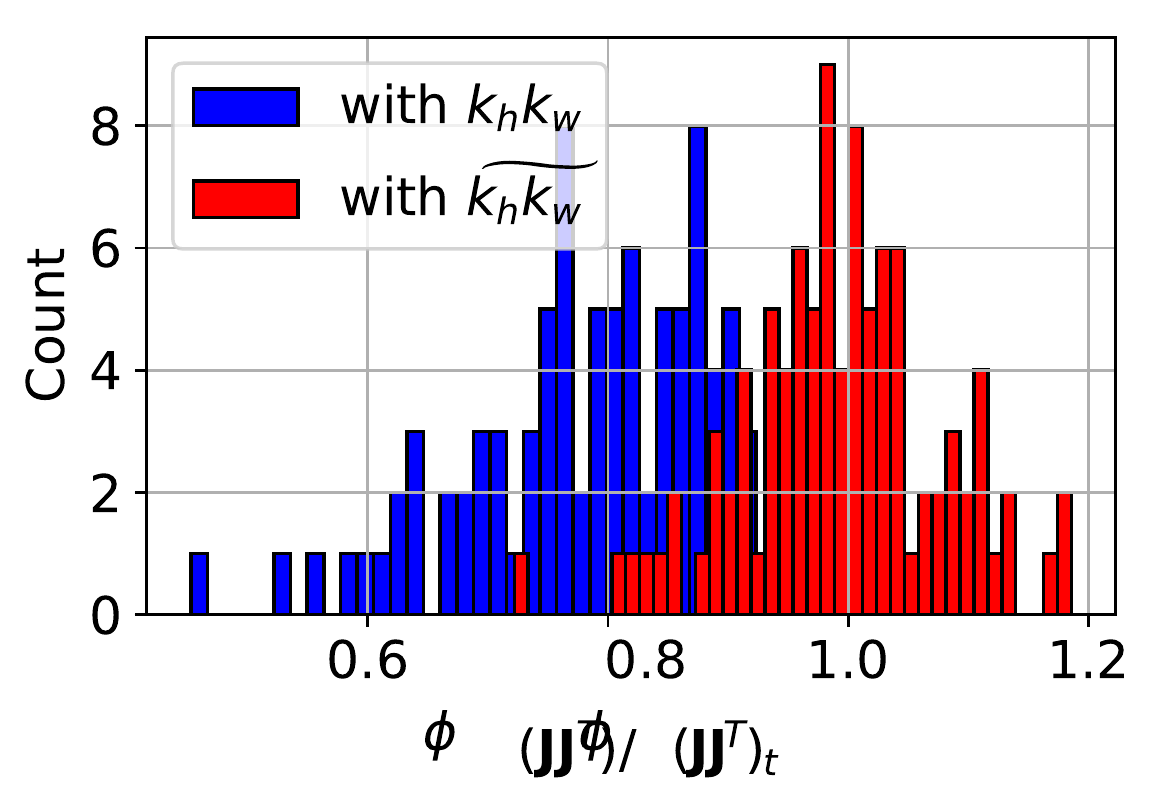}
  \caption{The effectiveness of the correction with Algorithm \ref{Alg:conv_adjust}.}
  \label{fig:verify_effective_kernel_size}
\end{figure}
We can see that the results with the effective kernel size $\widetilde{k_hk_w}$ perfectly centered around 1, while those with $k_hk_w$ has an obvious bias over 10\%.

Then we prove that convolution 2D with the Gaussian kernel satisfies Definition \ref{def:expectattion_diagonal_matrix} and \ref{def:central_matrix}.

We can write the Jacobian matrix of the 2D convolution as $\mathbf{J}=[\mathbf{j_1}, \mathbf{j_2}, .., \mathbf{j_n}]$, where $\mathbf{j_i}$ are column vectors. Thus 
\begin{equation}
    [\mathbf{J}^T\mathbf{J}]_{i,j} = \mathbf{j_i}^T\cdot \mathbf{j_j}.
\end{equation}

Since the kernel is initialized with i.i.d. $N(0, \sigma^2)$, when $i \neq j$, the corresponding entries of $\mathbf{j_i}$ and $\mathbf{j_j}$ are independent. As a result, $\forall i,j\neq i, E[\mathbf{j_i}^T\cdot \mathbf{j_j}] = 0$.

When $i=j$, $E\left[\mathbf{J}^T\mathbf{J}\right]_{i,i} = E\left[||\mathbf{j_i}||_2^2\right] = ||\mathbf{j_i}||_0\sigma^2$, where $||\mathbf{j_i}||_0$ denotes the $L_0$-norm of vector $\mathbf{j_i}$ (the amount of non-zero entries in $\mathbf{j_i}$).As almost all the column vectors contain the same amount of non-zero entries, we say that the diagonal elements of $\widetilde{\mathbf{K}}^T\widetilde{\mathbf{K}}$ almost share the same expectation
The proof of Definition \ref{def:central_matrix} is much simpler, as the entries of the Jacobian matrix is either zero-mean random variable or 0, it is a central matrix.

All in all, the 2D convolution satisfies Definition \ref{def:expectattion_diagonal_matrix} and \ref{def:central_matrix}.

\textbf{Orthogonal Transform} \modify{Because our target is to have $\phi(\mathbf{J_iJ_i}^T)=1$ and $\varphi(\mathbf{J_iJ_i}^T)\approx 0$, one intuitive idea is to initialize $\mathbf{J_i}$ to have orthogonal rows, such that we have}
\begin{equation}
    \phi(\mathbf{JJ}^T)=\beta^2,~~\varphi(\mathbf{JJ}^T)=0.
\label{equ:phi_orth}
\end{equation}
\input{algorithms/orth_init.tex}
Saxe et al. (2013) \cite{saxe2013exact} initialize the kernels in CONV layers with Algorithm \ref{Alg:orth_init}.

As $\mathbf{J}$ is orthogonal, it certainly satisfies Definition \ref{def:expectattion_diagonal_matrix}. Let $\mathbf{A}:= [\mathbf{\alpha_1}, \mathbf{\alpha_2}, ..., \mathbf{\alpha_n}]$ is an i.i.d. Gaussian matrix. We can get a group of orthogonal basis $[\hat{\mathbf{\beta_1}}, \hat{\mathbf{\beta_2}}, ..., \hat{\mathbf{\beta_n}}]$with Gram-Schmidt process:
\begin{equation}
    \mathbf{\beta_j} = \mathbf{\alpha_j} - \sum_{k=1}^{j-1}\frac{\mathbf{\alpha_j}^T\cdot\mathbf{\beta_k}}{\mathbf{\beta_k}^T\cdot\mathbf{\beta_k}}\mathbf{\beta_k},~~~\hat{\mathbf{\beta_j}} = \frac{1}{||\mathbf{\beta_j}||_2}\mathbf{\beta_j}.
\label{equ:gram-schmidt-process}
\end{equation}

As long as $\forall i,j,~~E\left[\left[\mathbf{A}\right]_{i,j}\right]=0$, with Equation \eqref{equ:gram-schmidt-process}, we have $\forall j,k, E\left[\left[\hat{\mathbf{\beta_j}}\right]_{k}\right]=0$. As a result, orthogonal layers satisfy Definition \ref{def:central_matrix}.

\subsubsection{Normalization and Regularization}

\textbf{Data Normalization.($\mathbf{g}:=norm(\mathbf{u})$)} We first formulate $\mathbf{g}(\mathbf{u})$ as:

\begin{equation}
    \mathbf{g}_i = \frac{\mathbf{u}_i-\mu_B}{\sigma_B}, \mu_B = \frac{1}{m}\sum_{k=1}^m \mathbf{u}_k, \sigma_B = \sqrt{\frac{1}{m}\sum_{k=1}^m (\mathbf{u}_k - \mu_B)^2}
\end{equation}
Following this formulation, $\partial \mathbf{g}_i/\partial \mathbf{u_j}$ can be calculated as below:

\begin{equation}
    \begin{split}
        & \frac{\partial \mathbf{g}_i}{\partial \mathbf{u}_j} = \frac{1}{\sigma_B^2}\left[\left( \frac{\partial \mathbf{u}_i}{\partial \mathbf{u}_j} - \frac{\partial \mu_B}{\partial \mathbf{u}_j}\right) \sigma_B - \left((\mathbf{u}_i - \mu)\frac{\partial \sigma_B}{\partial \mathbf{u}_j}\right) \right],\\
        & \frac{\partial \mu_B}{\partial \mathbf{u}_j} = \frac{1}{m},~~\frac{\partial\sigma_B}{\partial\mathbf{u}_j} = \frac{1}{m\sigma_B}(\mathbf{u}_j-\mu_B).
    \end{split}
\label{equ:Jacobin_g_u}
\end{equation}
From Equation \eqref{equ:Jacobin_g_u}, we can get:

\begin{equation}
\begin{split}
    & \frac{\partial \mathbf{g}_i}{\partial \mathbf{u}_j} = \frac{1}{\sigma_B^2}\left[\left( \Delta - \frac{1}{m}\right) \sigma_B - \frac{1}{m\sigma_B}\left(\mathbf{u}_i - \mu_B)(\mathbf{u}_j - \mu_B)\right) \right]\\
    & \Delta = 1 ~~ if~~i=j~~else~~0.
\end{split}
\label{equ:Jacobin_g_u_result}
\end{equation}
We denote $\hat{\mathbf{u}}_i := (\mathbf{u}_i - \mu_B)/\sigma_B$, then Equation \eqref{equ:Jacobin_g_u_result} can be simplified as:

\begin{equation}
    \frac{\partial \mathbf{g}_i}{\partial \mathbf{u}_j} = \frac{1}{\sigma_B}\left[ \Delta - \frac{1}{m}(1 + \hat{\mathbf{u}}_i\hat{\mathbf{u}}_j) \right],\Delta = 1 ~~ if~~i=j~~else~~0.
\label{equ:Jacobin_g_u_result_simple}
\end{equation}
A matrix $\mathbf{U}$ defined as $U_{i,j} := 1 + \hat{\mathbf{u}}_i\hat{\mathbf{u}}_j$ is a real symmetric matrix, which can be broken down with eigendecomposition to $\mathbf{U} = \mathbf{P}^T\mathbf{D}\mathbf{P}$. And we have

\begin{equation}
    \mathbf{g_u} = \frac{1}{\sigma_B}\mathbf{P}^T(\mathbf{I - \frac{1}{m}\mathbf{D}})\mathbf{P}.
\label{equ:bn_eig_expansion}
\end{equation}
With Equation \eqref{equ:bn_eig_expansion}, as $\mathbf{P}^T\mathbf{P}=\mathbf{I}$, if $m$ is sufficiently big, we will easily get

\begin{equation}
\begin{split}
    &\phi(\mathbf{g_ug_u}^T) = \lim_{m\rightarrow\infty} \frac{1}{\sigma_B^2}\frac{1}{m}\sum_{i=1}^m(1-\frac{\lambda_i}{m})^2,\\
    &\varphi(\mathbf{g_ug_u}^T) = \lim_{m\rightarrow\infty} \frac{1}{\sigma_B^4}\frac{1}{m}\sum_{i=1}^m(1-\frac{\lambda_i}{m})^4 - \phi^2(\mathbf{g_ug_u}^T).
\end{split}
\label{equ:eig_bn_pre}
\end{equation}
where $\lambda_i$ are the eigenvalues of $\mathbf{U}:= \mathbf{1}+\hat{\mathbf{u}}\hat{\mathbf{u}}^T$, $\hat{\mathbf{u}}\hat{\mathbf{u}}^T = \frac{1}{\sigma_B^2}\left((\mathbf{u} - \mu_B)(\mathbf{u} - \mu_B)^T\right)$ and $\mathbf{1}$ is an $m\times m$ matrix filled with $1$.
\begin{lemma}
 $\mathbf{U}:= \mathbf{1}+\hat{\mathbf{u}}\hat{\mathbf{u}}^T$ has two non-zero eigenvalues, and in the limit of $m$, $\lim_{m\rightarrow\infty}\frac{\lambda_1}{m}=\lim_{m\rightarrow\infty}\frac{\lambda_2}{m}=1$.
 \label{lemma:eig_bn}
\end{lemma}
\begin{proof}
We denote the $i^{th}$ element of $\hat{\mathbf{u}}$ as $u_i$ and solve the eigenvalues with $det(\mathbf{1}+\hat{\mathbf{u}}\hat{\mathbf{u}}^T-\lambda \mathbf{I})=0$.

\begin{equation}
\begin{split}
&{\left|\begin{array}{cccc} 
    u_1u_1+1-\lambda &    u_1u_2+1    & ... &u_1u_m+1 \\ 
    u_2u_1+1 &    u_2u_2+1 - \lambda    & ... &u_2u_m+1 \\
    ... & ... & ...&...\\
    u_mu_1+1 &    u_mu_2+1    & ... &u_mu_m+1-\lambda
\end{array}\right|}\\
&=(-\lambda)^m\left(-\sum_{i=1}^m\frac{u_i^2-1}{\lambda}\right)\\
& \times \left(\left(1-\frac{m}{\lambda}\right)\!+\!\left(1-\sum_{i=1}^m\frac{u_i+1}{\lambda}\right)\frac{\lambda+\sum_{i=1}^mu_i-1}{-\sum_{i=1}^mu_i^2-1}\right)=0.
\end{split}
\label{equ:1+uut}
\end{equation}
For the sake of simplicity, we denote $A:=\sum_{i=1}^mu_i,B:=\sum_{i=1}^mu_i^2$, and Equation \eqref{equ:1+uut} can be simplified as $\lambda^2 - (B+m)\lambda + (mB-A^2)=0$, whose solution is:

\begin{equation}
    \lambda = \frac{1}{2}(B+m)\pm \sqrt{(B-m)^2+A^2}.
\end{equation}
When $m$ is large enough, $\frac{A}{m}=\frac{1}{m}\sum_{i=1}^mu_i=0,\frac{B}{m}=\frac{1}{m}\sum_{i=1}^mu_i^2=1$ and we have:
\begin{equation}
    \lim_{m\rightarrow \infty}\frac{\lambda}{m} = \frac{1}{2}\left(\frac{B}{m}+1\right)\pm \sqrt{\left(\frac{B}{m}-1\right)^2+\left(\frac{A}{m}\right)^2}=1.
\end{equation}
\end{proof}

By substituting Lemma \ref{lemma:eig_bn} into Equation \eqref{equ:eig_bn_pre}, we get
\begin{equation}
    \phi(\mathbf{g_ug_u}^T) \approx \frac{1}{\sigma_B^2},~~\varphi(\mathbf{g_ug_u}^T) \approx \frac{2}{m\sigma_B^4}.
\label{equ:eig_bn}
\end{equation}

As $m$ is greater than tens of thousands, $\mathbf{g_u}\approx\frac{1}{\sigma_B}\mathbf{I}$, which is diagonal. As a result, it satisfies Definition \ref{def:expectattion_diagonal_matrix} and defies Definition \ref{def:central_matrix}.

%% file: algorithms/conv_adjust.tex
\begin{algorithm}[ht]
\DontPrintSemicolon
 \KwData{input channels: $c_{in}$; output channels: $c_{out}$; kernel size: $[k_h, k_w]$; stride: $[s_h, s_w]$; padding: $[p_h, p_w]$; input images' size: $[h_{in}, w_{in}]$.}
 \KwResult{effective kernel size $\widetilde{k_hk_w}$}
 $UpperSide(i) = min(k_h, p_h - i\times s_h)$, $LowerSize(i)=min(k_h, s_h(i+h_{out}^{hp}) + k_h - h_{in}-p_h$\\
 $LeftSide(i) = min(k_w, p_w - i\times s_w)$, $RightSide(i)=min(k_w, s_h(i+w_{out}^{hp}) + k_w - w_{in}-p_w$
 \Begin{
 $h_{out} = \left\lfloor\frac{1}{s_h}(h_{in} + 2p_h - k_h) \right\rfloor + 1$, $w_{out} = \left\lfloor\frac{1}{s_w}(w_{in} + 2p_w - k_w) \right\rfloor + 1$\\
 $h_{out}^{hp} = \left\lfloor\frac{1}{s_h}(h_{in} + p_h - k_h) \right\rfloor + 1$, $w_{out}^{hp} = \left\lfloor\frac{1}{s_w}(w_{in} + p_w - k_w) \right\rfloor + 1$\\
 $it_{h}^{upper}=\left\lfloor\frac{p_h}{s_h}\right\rfloor + 1$,  $it_{h}^{lower}=h_{out} - h_{out}^{hp}$, 
 $it_{w}^{left}=\left\lfloor\frac{p_w}{s_w}\right\rfloor + 1$,  $it_{w}^{right}=w_{out} - w_{out}^{hp}$\\
 
 $T=k_hk_wh_{out}w_{out},~~P=0,~~R=0$\\
 \For{$i \in range(it_{h}^{upper})$}{
 $P+=UpperSide(i)w_{out}k_w$\\
 \For{$j \in range(it_{w}^{left})$}{
 $R+=UpperSide(i)LeftSide(j)$\\
 }
 \For{$j \in range(it_{w}^{right})$}{
 $R+=UpperSide(i)RightSide(j)$
 }
 }
 \For{$i \in range(it_{h}^{lower})$}{
 $P+=LowerSide(i)w_{out}k_w$\\
 \For{$j \in range(it_{w}^{left})$}{
 $R+=LowerSide(i)LeftSide(j)$\\
 }
 \For{$j \in range(it_{w}^{right})$}{
 $R+=LowerSide(i)RightSide(j)$
 }
 }
 \For{$i \in range(it_{w}^{left})$}{
 $P+=LeftSide(i)h_{out}k_h$
 }
 \For{$i \in range(it_{w}^{right})$}{
 $P+=RightSide(i)h_{out}k_h$
 }

 }
 \Return{$\widetilde{k_hk_w}=\frac{T-P+R}{T}k_hk_w$}
 \caption{Effective Kernel Size of Conv.2D}
 \label{Alg:conv_adjust}
\end{algorithm}

%% file: algorithms/orth_init.tex
\begin{algorithm}[ht]
\DontPrintSemicolon
 \KwData{Kernel size $[c_{out}, c_{in}, k_h, k_w]$; gain: $\beta$.}
 \KwResult{Kernel $\mathbf{K}$.}
 \Begin{
 Initialize $\mathbf{A}\in \mathbb{R}^{c_{in}k_hk_w\times c_{out}}$ with i.i.d. $N(0, 1)$.\\
 $\mathbf{Q}, \mathbf{R} = reduced~~QR~~decomposition(\mathbf{A})$.\\
   //$\mathbf{Q}\in \mathbb{R}^{c_{in}k_hk_w\times c_{out}}$ with orthogonal columns\\
 $\mathbf{S}=diag(sign([\mathbf{R}]_{11}), sign([\mathbf{R}]_{22}),...)$\\
 $\mathbf{K} = \beta \times reshape((\mathbf{QS})^T, [c_{out}, c_{in}, k_h, k_w])$
}
 \Return{$\mathbf{K}$}
 \caption{Haar Orthogonal Initialization}
 \label{Alg:orth_init}
\end{algorithm}

%% file: chapters/proofs/proof_second_order_moment_norm.tex
We denote $\sqrt{E\left[[\mathbf{x}]_i^2\right]}$ as $\alpha_2$, therefore we have
\begin{equation}
    \begin{split}
        & \frac{\partial [\hat{\mathbf{x}}]_i}{\partial [\mathbf{x}]_j} = \frac{1}{\alpha_2}\left[\Delta - \frac{1}{m}[\hat{\mathbf{x}}]_i[\hat{\mathbf{x}}]_j\right],~~\Delta=1~~if~~i=j~~else~~0.
    \end{split}
\end{equation}
A matrix $\mathbf{U}$ defined as $[\mathbf{U}]_{i,j}:=[\hat{\mathbf{x}}]_i[\hat{\mathbf{x}}]_j$ is a real symmetric matrix, which can be broken down with eigendecomposition to $\mathbf{U}=\mathbf{P}^T\mathbf{DP}$. And we have
\begin{equation}
    \hat{\mathbf{x}}_{\mathbf{x}}=\frac{1}{\alpha_2}\mathbf{P}^T(\mathbf{I}-\frac{1}{m}\mathbf{D})\mathbf{P}.
    \label{equ:l2n_eig_expansion}
\end{equation}
With Equation \eqref{equ:l2n_eig_expansion}, as $\mathbf{P}^T\mathbf{P}=\mathbf{I}$ and assuming that $m$ is sufficiently big, we can easily get
\begin{equation}
\begin{split}
    &\phi(\hat{\mathbf{x}}_{\mathbf{x}}\hat{\mathbf{x}}_{\mathbf{x}}^T) = \lim_{m\rightarrow\infty} \frac{1}{\alpha_2^2}\frac{1}{m}\sum_{i=1}^m(1-\frac{\lambda_i}{m})^2,\\
    &\varphi(\hat{\mathbf{x}}_{\mathbf{x}}\hat{\mathbf{x}}_{\mathbf{x}}^T) = \lim_{m\rightarrow\infty} \frac{1}{\alpha_2^4}\frac{1}{m}\sum_{i=1}^m(1-\frac{\lambda_i}{m})^4 - \phi^2(\hat{\mathbf{x}}_{\mathbf{x}}\hat{\mathbf{x}}_{\mathbf{x}}^T).
\end{split}
\label{equ:eig_l2n_pre}
\end{equation}
where $\lambda_i$ is the eigenvalue of $\mathbf{U}:= \hat{\mathbf{x}}\hat{\mathbf{x}}^T$, $\hat{\mathbf{x}}\hat{\mathbf{x}}^T = \frac{1}{\alpha_2^2}\mathbf{x}\mathbf{x}^T$.
\begin{lemma}
 $\mathbf{U}:= \hat{\mathbf{u}}\hat{\mathbf{u}}^T$ has one non-zero eigenvalues, and in the limit of $m$, $\lim_{m\rightarrow\infty}\frac{\lambda}{m}=0$.
 \label{lemma:eig_l2n}
\end{lemma}
\begin{proof}
We denote the $i^{th}$ element of $\hat{\mathbf{u}}$ as $u_i$ and solve the eigenvalues with $det(\hat{\mathbf{u}}\hat{\mathbf{u}}^T-\lambda \mathbf{I})=0$.

\begin{equation}
\begin{split}
&{\left|\begin{array}{cccc} 
    u_1u_1-\lambda &    u_1u_2    & ... &u_1u_m \\ 
    u_2u_1 &    u_2u_2 - \lambda    & ... &u_2u_m \\
    ... & ... & ...&...\\
    u_mu_1 &    u_mu_2    & ... &u_mu_m-\lambda
\end{array}\right|}\\
&~~~~~~~~~~~~~~~~~~~~~~~~~~~~~=\left(-\lambda\right)^m\left(1-\frac{1}{\lambda}\sum_{i=1}^mu_i^2\right)=0.
\end{split}
\label{equ:uut}
\end{equation}
Therefore the only non-zero eigenvalue is $\lambda = \sum_{i=1}^mu_i$.
When $m$ is large enough, we have
\begin{equation}
    \lim_{m\rightarrow \infty}\frac{\lambda}{m} = \frac{1}{m}\sum_{i=1}^mu_i = 0.
\end{equation}
\end{proof}

By substituting Lemma \ref{lemma:eig_bn} into Equation \eqref{equ:eig_bn_pre}, we get
\begin{equation}
    \phi(\hat{\mathbf{x}}_{\mathbf{x}}\hat{\mathbf{x}}_{\mathbf{x}}^T) \approx \frac{1}{\alpha_2^2},~~\varphi(\hat{\mathbf{x}}_{\mathbf{x}}\hat{\mathbf{x}}_{\mathbf{x}}^T) \approx 0.
\label{equ:eig_l2n}
\end{equation}

As $m$ is greater than tens of thousands, $\hat{\mathbf{x}}_{\mathbf{x}}\approx\frac{1}{\alpha_2}\mathbf{I}$, which is diagonal. As a result, it satisfies Definition \ref{def:expectattion_diagonal_matrix} and defies Definition \ref{def:central_matrix}.

%% file: chapters/proofs/proof_gaussian_active.tex
\textbf{Proposition \ref{prop:gaussian_activ_2nd_invariant}. } \textit{A neural network composed of cyclic central Gaussian transform with i.i.d. entries and any network components is $\infty^{th}$ order moment unitary invariant.}
\rule[0pt]{0.48\textwidth}{0.05em}

\begin{definition}
\textbf{(bi-unitarily invariant random matrix)}\cite{cakmak2012non}
Let $\mathbf{X}$ be a $R\times T$ random matrix. If for any unitary matrices $\mathbf{U}$, $\mathbf{V}$, the joint distribution of the entries of $\mathbf{X}$ equals to the joint distribution of the entries of $\mathbf{Y}=\mathbf{UHV}^{H}$, then $\mathbf{X}$ is called bi-unitarily invariant random matrix.
\end{definition}

\begin{lemma}
Standard Gaussian matrices are bi-unitarily invariant random matrices.(Example 7 of chapter 3.1 in Cakmak (2012) \cite{cakmak2012non})
\label{lemma:bi_unitary_gaussain}
\end{lemma}

For a neural network composed of cyclic central Gaussian transform with i.i.d. entries and any network components with $L$ layers, it's input-output Jacobian can be written as $\Pi_{i=L}^1\mathbf{D_i}\mathbf{W_i}$, where $\mathbf{W_i}=\beta_i \mathbf{W}$, $\mathbf{W}$ is a standard Gaussian Matrix. $\forall p < \infty$, we have
\begin{equation}
\begin{split}
    & \phi\left(\left((\Pi_{i=L}^1 \mathbf{U_{D,i}D_iU_{W,i}W_i})(\Pi_{i=L}^1 \mathbf{U_{D,i}D_iU_{W,i}W_i})^T\right)^p\right)\\
    &=\phi\left(\mathbf{U_{D, L}}((\Pi_{i=L}^1 \mathbf{D_iU_{W,i}\beta_iWU_{D,i-1}})\right.\\
    &~~~~~~~~~~~~~~~~~~~~~~~~~~~\left.(\Pi_{i=L}^1 \mathbf{D_iU_{W,i}\beta_iWU_{D,i-1}})^T)^p\mathbf{U_{D, L}}^T\right).
\end{split}
\end{equation}
With Lemma \ref{lemma:bi_unitary_gaussain}, as $\mathbf{W}$ is a bi-unitarily invariant matrix, we have $\forall i, \mathbf{D_iU_{W,i}\beta_iWU_{D,i-1}} = \mathbf{D_i\beta_iW}$. As a result, we have
\begin{equation}
\begin{split}
    & \phi\left(\left((\Pi_{i=L}^1 \mathbf{U_{D,i}D_iU_{W,i}W_i})(\Pi_{i=L}^1 \mathbf{U_{D,i}D_iU_{W,i}W_i})^T\right)^p\right)\\
    &  =\phi\left(\left((\Pi_{i} \mathbf{D_iW_i})(\Pi_{i} \mathbf{D_iW_i})^T\right)^p\right).
\end{split}
\end{equation}
and the proposition is proved.

%% file: chapters/proofs/proof_diag_orth.tex
\textbf{Proposition \ref{prop:orth_activ_2nd_invariant}. } \textit{A neural network block composed of an orthogonal transform layer and an activation function layer is at least $2^{nd}$ moment unitary invariant.}

\rule[0pt]{0.46\textwidth}{0.05em}

Let $\mathbf{O}\in \mathbb{R}^{m\times n}$ be an orthogonal random matrix and $\mathbf{D}\in \mathbb{R}^{k\times m}$ be an diagonal random matrix. Let $\mathbf{U_o}$, $\mathbf{U_d}$ be two haar unitary matrices with proper size and independent with $\mathbf{O}$ and $\mathbf{D}$, then we have
\begin{equation}
\begin{split}
    &\phi\left(\left(\mathbf{U_dDU_oO}\right)\left(\mathbf{U_dDU_oO}\right)^T\right) = \phi(\mathbf{DD}^T),\\
    &\phi\left(\left(\left(\mathbf{U_dDU_oO}\right)\left(\mathbf{U_dDU_oO}\right)^T\right)^2\right) = \phi(\mathbf{DD}^T\mathbf{DD}^T).
\end{split}
\end{equation}
Similarly, we can prove that
\begin{equation}
\begin{split}
    &\phi\left(\left(\mathbf{DO}\right)\left(\mathbf{DO}\right)^T\right) = \phi(\mathbf{DD}^T),\\
    &\phi\left(\left(\left(\mathbf{DO}\right)\left(\mathbf{DO}\right)^T\right)^2\right) = \phi(\mathbf{DD}^T\mathbf{DD}^T).
\end{split}
\end{equation}
Thus $\mathbf{DO}$ is at least $2^{nd}$ moment unitary invariant.

%% file: chapters/proofs/proof_general_linear_op.tex
\textbf{Proposition \ref{prop:general_linear_transforms}. } \textit{
    Data normalization with 0-mean input, linear transforms and rectifier activation functions are general linear transforms.
}

\rule[0pt]{0.46\textwidth}{0.05em}

\textbf{Data Normalization. } As shown in Table \ref{tab:parts_library}, for data normalization, $\phi(\mathbf{f_x}\mathbf{f_x}^T)=\frac{1}{\sigma_B^2}$, where $\sigma_B$ is the standard deviation of input vector. As the input activation has zero-mean, $\sigma_B^2 = E\left[\frac{||\mathbf{x}||_2^2}{len(\mathbf{x})}\right]$ and the output of data normalization follows $N(0, 1)$, we have $E\left[\frac{||\mathbf{f}(\mathbf{x})||_2^2}{len(\mathbf{f}(\mathbf{x}))}\right]=1$.

\begin{lemma}
Let $\mathbf{f}(\mathbf{x})$ be a transform whose Jacobian matrix is $\mathbf{J}$. If $\mathbf{f}(\mathbf{x})=\mathbf{Jx}$, then we have
\begin{equation}
    E\left[\frac{||\mathbf{f}(\mathbf{x})||_2^2}{len(\mathbf{f}(\mathbf{x}))}\right] = \phi(\mathbf{JJ}^T)E\left[\frac{||\mathbf{x}||_2^2}{len(\mathbf{x})}\right].
\end{equation}
\label{lemma:g_linear_trans}
\end{lemma}
\begin{proof}
As $||\mathbf{f}(\mathbf{x})||_2^2 = ||\mathbf{J} \mathbf{x}||_2^2 = \mathbf{x}^T\mathbf{J}^T\mathbf{J}\mathbf{x}$, $\mathbf{J}^T\mathbf{J}$ is a real symmetric matrix which can be broken down as $\mathbf{J}^T\mathbf{J}=\mathbf{Q\Sigma Q}^T$ with eigendecomposition. Thus, we have
\begin{equation}
    E\left[\frac{||\mathbf{f}(\mathbf{x})||_2^2}{len(\mathbf{f}(\mathbf{x}))}\right] = \phi(\mathbf{JJ}^T)E\left[\frac{||\mathbf{x}||_2^2}{len(\mathbf{x})}\right].
\end{equation}
\end{proof}

\textbf{Linear transforms.} The Jacobian matrix of linear transform is itself, with Lemma \ref{lemma:g_linear_trans}, linear transform belongs to general linear transforms.

\textbf{Rectifier Activation Functions.} We generally denote the rectifier activation functions as $f(x) = \alpha x~~if~x\ge0~~else~~\beta x$. As the linear rectifiers perform element-wise linear transform, we can easily derives that $\forall i, [\mathbf{f}(\mathbf{x})]_i = [\mathbf{f_x}]_{i, i}[\mathbf{x}]_{i}$. Similarly, with Lemma \ref{lemma:g_linear_trans}, rectifier activation functions belong to general linear transforms.

%% file: chapters/proofs/proof_second_moment.tex
\textbf{Proposition \ref{prop:evolution_2nd_norm_net }. }\textit{
    For the series neural network $\mathbf{f}(\mathbf{x})$ composed of general linear transforms, whose input-output Jacobian matrix is $\mathbf{J}$, we have
}
\begin{equation}
    E\left[\frac{||\mathbf{f}(\mathbf{x})||_2^2}{len(\mathbf{f}(\mathbf{x}))}\right] = \phi(\mathbf{JJ}^T)E\left[\frac{||\mathbf{x}||_2^2}{len(\mathbf{x})}\right].
\end{equation}
\rule[0pt]{0.48\textwidth}{0.05em}

Let $\mathbf{f}(\mathbf{x})=\mathbf{f^{(1)}}\circ\mathbf{f^{(2)}}\circ...\mathbf{f^{(n)}}(\mathbf{x})$, since its components are general linear transforms, as a result, we have
\begin{equation}
    E\left[\frac{||\mathbf{f}(\mathbf{x})||_2^2}{len(\mathbf{f}(\mathbf{x}))}\right] = \Pi_i\phi(\mathbf{J_iJ_i}^T)E\left[\frac{||\mathbf{x}||_2^2}{len(\mathbf{x})}\right].
\end{equation}
And with Equation \eqref{equ:mul_expectation}, we have $\Pi_i\phi(\mathbf{J_iJ_i}^T) = \phi(\mathbf{JJ}^T)$, and the proposition is proved.

%% file: chapters/proofs/selu_ej.tex
SeLU is formulated as
\begin{equation}
    selu(x) = \lambda\left\{
             \begin{array}{lr}
             x~~~~~~~~~~~~~~if~~x>0  \\
             \alpha e^x-\alpha~~if~~x \le0
             \end{array}
\right. .
\end{equation}
Let $\mathbf{J}:=\frac{\partial SeLU(\mathbf{x})}{\partial \mathbf{x}}$ and the entries of $\mathbf{x}$ follow $N(0, \sigma^2)$,
\begin{equation}
\begin{split}
    & \phi(\mathbf{JJ}^T) = \lambda^2\alpha^2\int_{-\infty}^0e^{2x}\frac{1}{\sqrt{2\pi\sigma^2}}e^{-\frac{1}{2\sigma^2}x^2}dx + \frac{\lambda^2}{2}\\
    &=\lambda^2\alpha^2e^{2\sigma^2}\int_{-\infty}^0\frac{1}{\sqrt{2\pi\sigma^2}}e^{-\frac{1}{2\sigma^2}(x^2-4\sigma^2x+4\sigma^4)}dx + \frac{\lambda^2}{2}\\
    &=\lambda^2\alpha^2e^{2\sigma^2}cdf(-2\sigma^2, N(0, \sigma^2))+ \frac{\lambda^2}{2}.
\end{split}
\end{equation}
Similarly, we have
\begin{equation}
\begin{split}
    & E[SeLU(\mathbf{x})]\!=\! \lambda\alpha\int_{-\infty}^0e^{x}\frac{1}{\sqrt{2\pi\sigma^2}}e^{-\frac{1}{2\sigma^2}x^2}dx - \frac{\lambda\alpha}{2}\!+\!\sqrt{\frac{\sigma^2}{2\pi}}\lambda\\
    &=\lambda\alpha e^{\frac{\sigma^2}{2}}\int_{-\infty}^0\frac{1}{\sqrt{2\pi\sigma^2}}e^{-\frac{1}{2\sigma^2}(x^2-2\sigma^2x+\sigma^4)}dx - \frac{\lambda\alpha}{2}\!+\!\sqrt{\frac{\sigma^2}{2\pi}}\lambda\\
    &=\lambda\alpha e^{\frac{\sigma^2}{2}}cdf(-\sigma^2, N(0, \sigma^2))- \frac{\lambda\alpha}{2} + \sqrt{\frac{\sigma^2}{2\pi}}\lambda.
\end{split}
\end{equation}
We can further have
\begin{equation}
    \begin{split}
        & E[SeLU^2(\mathbf{x})] = \int_{-\infty}^0\left(\lambda\alpha e^{x}-\lambda\alpha\right)^2\frac{1}{\sqrt{2\pi\sigma^2}}e^{-\frac{1}{2\sigma^2}x^2}dx\\
        & + \lambda^2\int_{0}^{+\infty}x^2 \frac{1}{\sqrt{2\pi\sigma^2}}e^{-\frac{1}{2\sigma^2}x^2}dx=\\
        &\lambda^2\alpha^2\left(e^{2\sigma^2}cdf(-2\sigma^2, N(0, \sigma^2))\!-\! 2e^{\frac{\sigma^2}{2}}cdf(-\sigma^2, N(0, \sigma^2))\right)\\
        & + \frac{1}{2}\lambda^2\alpha^2 + \frac{1}{2}\lambda^2\sigma^2.
    \end{split}
\end{equation}

%% file: chapters/proofs/proof_pre_s_p_h_network.tex
\textbf{Proposition \ref{prop:prerequisite_series_parallel_hybrid_networks}. }\textit{
Let $\{\mathbf{J_i}\}$ denote a group of independent input-output Jacobian matrices of the parallel branches of a block. $\sum_i\mathbf{J_i}$ is an expectant orthogonal matrix, if it satisfies:
\begin{itemize}
    \item $\forall i$, $\mathbf{J_i}$ is an expectant orthogonal matrix
    \item At most one matrix in $\{\mathbf{J_i}\}$ is not central matrix
\end{itemize}}
\rule[0pt]{0.46\textwidth}{0.05em}

According to Definition \ref{def:expectattion_diagonal_matrix}, we have to prove that the non-diagonal entries of $(\sum_i\mathbf{J_i})(\sum_i\mathbf{J_i})^T$ has zero expectation while the diagonal entries share identical expectation. We first expand $(\sum_i\mathbf{J_i})^T(\sum_i\mathbf{J_i})$ as
\begin{equation}
    (\sum_i\mathbf{J_i})(\sum_i\mathbf{J_i})^T=\sum_{i,j}\mathbf{J_i}^T\mathbf{J_j} = \sum_{i,i}\mathbf{J_i}\mathbf{J_i}^T + \sum_{i,j\neq i}\mathbf{J_i}\mathbf{J_j}^T.
\end{equation}
As $\forall p,q,i,j,~~E\left[\left[\sum_{i,j}\mathbf{J_i}\mathbf{J_j}^T\right]_{p,q}\right] = \sum_{i,j}E\left[\right[\mathbf{J_i}\mathbf{J_j}^T\left]_{p,q}\right]$, $\forall i$, $\mathbf{J_i}$ is an expectant orthogonal matrix and at most one matrix in $\{\mathbf{J_i}\}$ is not central matrix, we have
\begin{itemize}
    \item $\forall i, j\neq i, p,q,~~E\left[\right[\mathbf{J_i}\mathbf{J_j}^T\left]_{p,q}\right]=0$.
    \item $\forall i,j, p, q\neq p,~~E\left[\right[\mathbf{J_i}\mathbf{J_j}^T\left]_{p,q}\right]=0$.
    \item $\forall i,p, ~~E\left[\right[\mathbf{J_i}\mathbf{J_i}^T\left]_{p,p}\right]$ share identical value.
\end{itemize}
Thus $(\sum_i\mathbf{J_i})$ is an expectant orthogonal matrix, and Proposition \ref{prop:prerequisite_series_parallel_hybrid_networks} is proved.

%% file: chapters/proofs/proof_plus_1_trick.tex
\textbf{Proposition \ref{prop:plus_1_trick}. }\textit{
\textbf{(The "Plus One" Trick). }Assuming that for each of the sequential block of a series-parallel hybrid neural network consists of $L$ blocks whose input-output Jacobian matrix is denoted as $\mathbf{J_i}$, we have $\phi(\mathbf{J_iJ_i}^T)=1+\phi(\widetilde{\mathbf{J_i}}\widetilde{\mathbf{J_i}}^T)$, then the neural network has gradient norm equality as long as}
\begin{equation}
    \phi(\widetilde{\mathbf{J_i}}\widetilde{\mathbf{J_i}}^T) = O(\frac{1}{L^p}), p>1.
\end{equation}
\rule[0pt]{0.48\textwidth}{0.05em}

$\widetilde{\mathbf{J_i}}\widetilde{\mathbf{J_i}}^T$ is a positive semi-definite matrix, $\phi(\widetilde{\mathbf{J_i}}\widetilde{\mathbf{J_i}}^T)\ge 0$, and $\phi(\mathbf{J_iJ_i}^T)\ge 1$, thus gradient vanishing will never occur. Moreover, since $\forall a<b$, we have $\Pi_{i=1}^a\phi(\mathbf{J_iJ_i}^T)\le\Pi_{i=1}^b\phi(\mathbf{J_iJ_i}^T)$, in order to avoid gradient explosion, we only have to make sure $\Pi_{i=1}^L\phi(\mathbf{J_iJ_i}^T)$ is $O(1)$. As $\phi(\widetilde{\mathbf{J_i}}\widetilde{\mathbf{J_i}}^T) = O(\frac{1}{L})$, we have
\begin{equation}
\begin{split}
    & \Pi_{i=1}^L\phi(\mathbf{J_iJ_i}^T) = 1 + \sum_{i=1}^L\phi(\widetilde{\mathbf{J_i}}\widetilde{\mathbf{J_i}}^T) + o(\frac{1}{L^{p-1}})\\
    & = 1 + O(\frac{1}{L^{p-1}})+ o(\frac{1}{L^{p-1}}).
\end{split}
\end{equation}
Thus gradient explosion can be avoid.

%% file: chapters/l2n_overhead.tex
\subsection{Computational Complexity of Second Moment Normalization}

\modify{

Since batch normalization is highly optimized in cuDNN library, it's difficult to directly measure the speedup provided by Second Moment Normalization. So inspired by Chen et al. (2019) \cite{chen2019effective}, we theoretically estimate the speedup.

Because both BN\cite{ioffe2015batch} and our second moment normalization are applied channel-wise, we can compare then with one channel. Let the pre-activation be $vec(\mathbf{x})\in\mathbb{R}^{m}$, the entries in the weight kernel be $vec(\mathbf{K})\in \mathbb{R}^{q}$. For identity transforms, we have scalar scale and bias coefficient $\gamma$ and $\beta$.

According to Chen et al. (2019) \cite{chen2019effective}, the forward and backward pass of BN can formulated as below
\begin{small}
\begin{equation}
    \label{equ:eff_BN}
    \begin{split}
        & \mu=\frac{1}{m}\sum_{j=1}^mx_j ,~\sigma^2 =  \frac{1}{m}\sum_{j=1}^m(x_j -\mu )^2,\\
        & y_i = \frac{\gamma }{\sigma }x_i - \left(\frac{\mu \gamma }{\sigma } - \beta \right),\\
        & \frac{\partial l}{\partial \beta } = \sum_{i=1}^m \frac{\partial l}{\partial y_i },~\frac{\partial l}{\partial \gamma }  =  \frac{1}{\sigma } \left( \sum_{i=1}^m \frac{\partial l}{\partial y_i } \cdot x_i  -  \mu \frac{\partial l}{\partial \beta } \right),\\
        & \frac{\partial l}{\partial x_i } = \frac{\gamma }{\sigma } \left[ \frac{\partial l}{\partial y_i } - \frac{\partial l}{\partial \gamma } \frac{x_i }{m\sigma } - \left(\frac{\partial l}{\partial \beta } \frac{1}{m} - \frac{\partial l}{\partial \gamma } \frac{\mu }{m\sigma } \right) \right].
    \end{split}
\end{equation}
\end{small}
Similarly, in the second moment normalization, we have
\begin{equation}
    \begin{split}
        & \mu = \frac{1}{q}\sum_{j=1}^q K_j, \sigma^2 = \frac{1}{m}\sum_{j=1}^m x_j^2, \hat{K_i} = K_i - \mu,\\
        & y_i = \frac{\gamma}{\sigma}x_i + \beta, \frac{\partial l}{\partial \gamma} = \frac{1}{\sigma}\sum_{j=1}^m \frac{\partial l}{\partial y_j}x_j,\\
        & \frac{\partial l}{\partial K_i}=\frac{\partial l}{\partial \hat{K_i}} - \frac{1}{q}\sum_{j=1}^q\frac{\partial l}{\partial \hat{K_j}}, \frac{\partial l}{\partial \beta} = \sum_{j=1}^m\frac{\partial l}{\partial y_i},\\
        & \frac{\partial l}{\partial x_i} = (\frac{\partial l}{\partial \beta}\frac{\gamma}{m\sigma^2})x_i + \frac{\gamma}{\sigma}\frac{\partial l}{\partial y_i}.
    \end{split}
\end{equation}

The forward and backward passes involve several element-wise operations and reduction operations. Following Chen et al. (2019) \cite{chen2019effective}, we denote an element-wise operation involving m elements as $E[m]$, a reduction operations involving m elements as $R[m]$, and we count the number of operations in BN and SMN. The results are shown in the table below:
\begin{table}[htbp]
\caption{Number of ops in BN and SMN.}
\centering
\resizebox{0.42\textwidth}{!}{
\begin{tabular}{c|c|c|c|c}
\hline
  &{\bf $R^{[m]}$} &{\bf $R^{[q]}$} &{\bf $E^{[m]}$} &{\bf $E^{[q]}$}
\\ \hline \hline
BN: Forward Pass & 2 & & 4 & \\
\hline
BN: Backward Pass &2 && 5&   \\
\hline
BN: Total & 4& & 9 &  \\
\hline \hline
SMN: Forward Pass & 1& 1 & 3 &1 \\
\hline
SMN: Backward Pass &2 & 1 & 4& 1 \\
\hline
SMN: Total & 3 & 2 & 7 &2\\
\hline
\end{tabular}}
\label{tab:ops_in_bn_sbn}
\end{table}
Since $q$ is usually much smaller than $m$, the cost of $R[q]$ and $E[q]$ can be ignored, and SMN reduces the number of operations from 13 to 10, which is roughly 30\% speedup.
}

%% file: chapters/appendix_discussions.tex
\section{Discussion}

\subsection{Freely Independent}
\input{chapters/discussions/freely_independent.tex}

%% file: chapters/discussions/freely_independent.tex
In recent years, several studies exploit the free probability theory in the analysis of the spectrum density of neural networks\cite{pennington2017resurrecting, ling2018spectrum, tarnowski2018dynamical, pennington2018emergence}. We find that most of them just make arbitrary assumption on the freeness between Jacobian matrices, which is not commonly held \cite{chen2012partial}. One simple counter example is that two independent i.i.d.random Gaussian matrices with non-zero mean are not freely independent. 

Moreover, it's also challenging to verify whether two matrices are free. The definition of freeness is given as below.
\begin{definition}
\textbf{(Freeness of Random Matrices)}(Definition 1 in Chen et al. (2012) \cite{chen2012partial}) Two random matrices $\mathbf{A}, \mathbf{B}$ are free respect to $\phi$ if for all $k\in \mathbb{N}$,
\begin{equation}
    \phi(p_1(\mathbf{A})q_1(\mathbf{B})p_2(\mathbf{A})q_2(\mathbf{B})...p_k(\mathbf{A})q_k(\mathbf{B})) = 0.
\end{equation}
for all polynomials $p_1,q_1,p_2,q_2,...,p_k,q_k$ such that we have $\phi(p_1(\mathbf{A}))=\phi(q_1(\mathbf{B}))=...=0$.

One can check whether two matrices $\mathbf{A},\mathbf{B}$ were free by checking that if 
\begin{equation}
\begin{split}
    & \phi\left(\left(\mathbf{A}^{n_1}-\phi(\mathbf{A}^{n_1})\right)\left(\mathbf{B}^{m_1}-\phi(\mathbf{B}^{m_1})\right)...\right.\\
    & \left....\left(\mathbf{A}^{n_k}-\phi(\mathbf{A}^{n_k})\right)\left(\mathbf{B}^{m_k}-\phi(\mathbf{B}^{m_k})\right)\right).
\end{split}
\end{equation}
vanish for all positive exponents $n_1,m_1,...,n_k,m_k$.
\label{def:freeness}
\end{definition}

Definition \ref{def:freeness} shows that it's numerically impractical to do the verification with Monte Carlo. Theoretically, unitarily invariant hermitian matrices(Definition \ref{def:unitarily_invariant}) are free with other hermitian matrices\cite{mingo2017free}, as formally proposed in Lemma \ref{lemma:unitary_invariant-free}.

\begin{definition}
\textbf{(Unitarily Invariant)}(Definition 8 in Cakmak (2012) \cite{cakmak2012non}) If a hermitian random matrix $\mathbf{H}$ has same distributed with $\mathbf{UHU}^*$ for any unitary matrix $\mathbf{U}$ independent of $\mathbf{H}$, then the $\mathbf{H}$ is called unitarily invariant.
\label{def:unitarily_invariant}
\end{definition}

\begin{lemma}
(Theorem 29 in Chapter 5.4 of Mingo \& Speicher (2017)\cite{mingo2017free}) Let $\mathbf{A},\mathbf{B}$ be hermitian matrices whose asymptotic averaged emprical eigenvalue distributions are compactly supported and $\mathbf{U}$ a Haar Matrix independent of $\mathbf{A},\mathbf{B}$, then $(\mathbf{UAU}^*,\mathbf{B})$ are almost surely freely independent.
\label{lemma:unitary_invariant-free}
\end{lemma}

However, as unitarily invariant requires the invariant of distribution, to the best of our knowledge, only some special matrices, i.e. haar matrices, Gaussian Wigner matrices, central Wishart matrices and certainly identity matrices can be easily proved to satisfy this requirement, and it's also challenging to verify either theoretically or numerically, not to mention the extension to general situations.
